\documentclass{article}
\usepackage{Packages}
\iclrfinalcopy
\title{Last-Iterate Guarantees for Online Reinforcement Learning in Structured Constrained MDPs}

\author{%
  Nam Phuong Tran \\
  Inria centre at the University of Lille\\
  Lille, France \\
   \And
   Trinh Ha Mai Huynh \\
   Vin Motion\\
   Vietnam \\
   \And
   Tuyen Pham Le \\
   Vin Motion\\
   Hanoi, Vietnam \\
   \And
   Van-Truong Nguyen \\
   Vin Motion\\
   Hanoi, Vietnam \\
   \And
   Quan Nguyen \\
   Vin Motion\\
   Hanoi, Vietnam \\
   \And
   Long Tran-Thanh\\
   Department of Computer Science\\
   University of Warwick\\
   Coventry, United Kingdom\\
}

\begin{document}

\addtocontents{toc}{\protect\setcounter{tocdepth}{0}}

\maketitle
\lhead{Preprint}

\begin{abstract}
In safety-critical applications, deployment uses a single policy, 
whose performance and constraint satisfaction should hold directly rather than 
only for an average or mixture of training policies. 
This motivates last-iterate guarantees in constrained reinforcement learning. 
Recent progress has established such guarantees in exact-gradient or tabular online settings, 
yet scalable results for structured large-state problems remain open.
We develop a general, statistically efficient framework for last-iterate convergence in structured Constrained MDPs (CMDPs). 
Our analysis separates contraction of the regularised primal-dual dynamics from actor approximation 
and statistical errors in policy evaluation under online exploration. 
This enables model-free on- and off-policy learning with structured function approximation: 
optimistic policy evaluation avoids explicit transition-model construction, 
while a compact parametric actor avoids maintaining mixtures or histories of past policies. 
We instantiate the framework for linear CMDPs and general function approximation, 
obtaining representation-dependent complexity and improved target-accuracy dependence over prior optimistic regularised primal-dual analyses.
We further validate the stabilising effect predicted by our theory on a synthetic linear CMDP:
the regularised method exhibits stable last-iterate behaviour, 
whereas its unregularised counterpart shows larger oscillations.
\end{abstract}

\section{Introduction}

Constrained Reinforcement Learning (RL) studies sequential decision making in which an agent maximizes
reward while satisfying explicit safety, resource, or performance constraints. Constrained Markov
decision processes (CMDPs) provide a standard framework for this problem
\citep{altman1999constrained,efroni2020explorationexploitationconstrainedmdps}. Safe robotics offers
a concrete motivation: reinforcement learning is increasingly used for locomotion
\citep{kumar2024rapid}, manipulation \citep{kroemer2021review}, and navigation
\citep{xu2025navrl}, where optimising reward alone is insufficient. A robot must also avoid unsafe
behaviour, and deployment ultimately relies on a single learned policy rather than an average of the
policies encountered during training.
Other applications in which the deployed policy must satisfy
constraints are discussed in Appendix~\ref{sec:Comprehensive related-work}.
Moreover, in these applications, state spaces are often large or effectively
continuous, and practical methods must remain computationally manageable. These applications therefore
make three general requirements especially clear: reliability of the final iterate, scalability
beyond tabular models, and efficient computation.

Safety requirements are often encoded through handcrafted rules, such as barrier-style constraints
\citep{ames2017control,wang2017safety,cheng2019end}, or through fixed reward penalties
\citep{lee2024learning,fu2022coupling,xu2025navrl}. Handcrafted
rules can be difficult to design for complex long-horizon tasks, while a fixed scalarized reward need
not recover an optimal constrained policy \citep{ding2024lastiterateconvergentpolicygradient}.
Lagrangian-based primal-dual methods instead adapt the reward-constraint tradeoff during learning.
Yet they also face a fundamental challenge: their iterates can oscillate, so average reward and
constraint satisfaction do not guarantee that the final policy is near-optimal and feasible
\citep{ding2024lastiterateconvergentpolicygradient}. For constrained learning, this distinction is
essential: one ultimately deploys a single policy, not an average of all policies produced during
training. Last-iterate convergence is therefore the natural theoretical objective.

Recent work has made important progress in this direction. In the exact setting, where the CMDP is
known, \citet{ding2024lastiterateconvergentpolicygradient} showed that regularised policy-gradient
primal-dual methods can stabilise the learning dynamics and achieve last-iterate convergence. More
recently, \citet{montenegro2024lastiterate,lu2026augmented} established model-free last-iterate convergence for
general parameterised policies using stochastic policy gradients. This moves the theory beyond
tabular policies, but relies on conditions ensuring that on-policy gradient estimates make global
progress and does not explicitly address exploration. In a complementary direction,
\citet{müller2024trulynoregretlearningconstrained,zuo2026flexdome} used optimism under uncertainty to address
exploration, but their online framework is tabular and model-based. Thus, a gap remains between
scalable model-free policy optimisation and explicit online exploration. The tabular dependence is
inadequate for large state spaces, while constructing an explicit optimistic transition
model can be computationally demanding. Moreover, the implicit optimistic policy update
accumulates past critic estimates; without a compact parametric representation, the policy
representation and its statistical complexity can grow with the number of iterations
\citep{lin2025optimisticworkshop}.

A natural way to move beyond the tabular regime is to exploit structured
function approximation, including linear CMDPs
\citep{jin2020provably,ghosh24aToward} and more general value-function
classes. Under suitable structural assumptions, complexity can depend on
the feature dimension or function-class complexity rather than the number
of states. Moreover, when the state space is large, model-free methods can be computationally more efficient 
since they avoid explicitly constructing a transition model. 
Extending online last-iterate guarantees to these settings requires preserving contraction 
across optimisation steps 
while controlling both statistical policy-evaluation error and 
approximation error from a compact parametric actor.
These challenges call for a unified, estimator-agnostic analysis.

Our central idea is a general framework that separates optimisation from statistical learning. The
master analysis controls the contraction of the regularised primal-dual dynamics, while modular
actor and critic conditions quantify the approximation and estimation errors arising from actor
fitting and policy evaluation under online exploration. Consequently, the same convergence argument
accommodates on- and off-policy data collection and different structured CMDP classes whenever these
errors can be controlled. The model-free critic exploits problem structure without constructing a
transition model, while the actor compresses the history-dependent policy update into a
fixed-dimensional parameterization rather than storing all past critics.

\textbf{Contributions.}
Motivated by these limitations, we ask whether \textit{a model-free
primal-dual method can achieve last-iterate convergence in structured CMDPs while explicitly
addressing exploration from online interaction}. We answer this question affirmatively. Our main
contributions are threefold:

\textit{First}, we develop a general algorithmic and analytical framework for last-iterate convergence of
online regularised policy-gradient primal-dual methods. 
It separates the contraction analysis from the approximation and statistical errors introduced by actor fitting, 
policy evaluation, and exploration, thereby accommodating both on- and off-policy data collection 
and different structured CMDP classes whenever these errors can be controlled. 
Its model-free optimistic critic and compact parametric actor avoid 
explicit transition-model construction and storage of the full critic history.

\textit{Second}, we instantiate the framework for linear CMDPs and general function
approximation, with statistical complexity governed by the feature
dimension or function-class complexity rather than state-space size.
Moreover, the linear-CMDP instantiation yields a computationally
efficient algorithm.

\textit{Third}, we sharpen the dependence on the target accuracy $\epsilon$ in the regularised primal-dual analysis of
    \citet{müller2024trulynoregretlearningconstrained}. In the comparable online tabular setting without actor
    approximation error, this improves the last-iterate episode complexity from
    \(\widetilde O(\epsilon^{-14})\) to \(\widetilde O(\epsilon^{-6})\).

There is a large literature on constrained reinforcement learning and primal-dual methods for CMDPs; 
we defer a detailed discussion to Appendix~\ref{sec:Comprehensive related-work}.
Table~\ref{tab:related-work-cmdp-comparison} lists the sample complexities
of CMDP methods, highlighting that our work achieves the tightest rate
among existing methods for online CMDPs.
 Moreover, to the best of our knowledge, we establish the first non-asymptotic \textit{last-iterate guarantees}
for online RL in \textit{structured CMDPs} while \textit{explicitly addressing the exploration problem}.


\section{Preliminaries}
\label{sec:prelim}
For \(K\in\mathbb N_+\), let \([K]\triangleq\{1,\ldots,K\}\). For a finite set
\(\mathcal X\), let \(\Delta(\mathcal X)\) denote the probability simplex over \(\mathcal X\), and
for \(p,q\in\Delta(\mathcal X)\), let
$\KL(p\|q)\triangleq\sum_{x\in\mathcal X}p(x)\log\frac{p(x)}{q(x)}$.
The notations $\widetilde O$ and $\widetilde\Omega$ suppress polylogarithmic factors.
For any $x\in \mathbb R$, let $[x]_+ \triangleq \max\{0,x\}$.
For a matrix $A$, let $A^\dagger$ denote its Moore-Penrose pseudoinverse.

\paragraph{Episodic CMDPs.}
We consider an episodic constrained Markov decision process (CMDP)
$M=(\mathcal S,\mathcal A,H,P,r,u,s_1)$, where $\mathcal S$ is a finite state space, $\mathcal A$ is a finite action space , $P=\{P_h\}_{h=1}^H$ is the transition
kernel, $r=\{r_h\}_{h=1}^H$ is the reward function, and
$u=\{u_{i,h}\}_{i\in[I],h\in[H]}$ are the utilities associated with the safety constraints, 
where $I$ is the number of constraints. The
learner observes stochastic rewards and utilities supported on $[0,1]$, with conditional means
$r_h(s,a)$ and $u_{i,h}(s,a)$, respectively.

A policy \(\pi = \{\pi_h\}_{h=1}^H\) is a sequence of stochastic decision rules
\(\pi_h(\cdot \mid s) \in \Delta(\mathcal{A})\). Let $\Pi$ denote the class of these policies. For any policy \(\pi\), let
$
d_h^\pi(s,a) \triangleq \mathbb{P}_\pi(s_h=s, a_h=a \mid s_1),
\;
d_h^\pi(s) \triangleq \sum_{a \in \mathcal{A}} d_h^\pi(s,a)
$
denote the occupancy measure and state marginal induced by \(\pi\).
For each $(h,s,a)$,  and every function
\(g:\mathcal S\to\mathbb R\), define $ (P_h g)(s,a)\triangleq\mathbb E_{s'\sim P_h(\cdot\mid s,a)}[g(s')].
$
For any \(r'=\{r_h'\}_{h=1}^H\) defined on state-action pairs, define
$
Q_{r',h}^\pi(s,a)
\triangleq
\mathbb{E}_\pi\left[\sum_{t=h}^H r_t'(s_t,a_t)\,\middle|\, s_h=s, a_h=a\right].
$
For any function \(f:\mathcal S\times\mathcal A\to\mathbb R\) and
\(h\in[H]\), define the \(\pi\)-Bellman operator by
\[
(\mathbb B_{r',h}^{\pi}f)(s,a)
\triangleq
\begin{cases}
r'_h(s,a)
+\mathbb E_{s'\sim P_h(\cdot\mid s,a),\,a'\sim\pi_{h+1}(\cdot\mid s')}
[f(s',a')], & h<H,\\
r'_H(s,a), & h=H.
\end{cases}
\]
Then \(Q_{r',h}^\pi\) satisfies the Bellman recursion
\[
Q_{r',h}^\pi(s,a) = (\Bell_{r',h}^{\pi} Q_{r',h+1}^\pi)(s,a),
\qquad
V_{r',h}^\pi(s)
=
\mathbb{E}_{a \sim \pi_h(\cdot \mid s)}\left[Q_{r',h}^\pi(s,a)\right].
\]
Here $Q_{r',H+1}^\pi\equiv0$.
In particular, \(V_r^\pi\) denotes the expected cumulative reward, and
$
V_u^\pi \triangleq \bigl(V_{u_1}^\pi,\dots,V_{u_I}^\pi\bigr) \in \mathbb{R}^I
$
collects the expected cumulative utilities.
When \(h=1\) and the initial state is \(s_1\), we write
$
V_{r'}^\pi \triangleq V_{r',1}^\pi(s_1).
$
In this work, our goal is to solve the constrained control problem
\begin{equation}
\label{eq:cmdp-primal}
\max_{\pi \in \Pi} \; V_r^\pi
\qquad
\text{s.t.}
\qquad
V_{u_i}^\pi \ge c_i,
\quad
\forall i \in [I],
\end{equation}
where \(c\in[0,H]^I\) is the vector of safety thresholds. Let $\pi^\star$ denote an optimal feasible policy.
We also define the log-linear policy class $\Pi_{\rm lin}$: given an
actor feature map $\varphi:\cS\times\cA\to\mathbb R^{d_{\rm a}}$ satisfying
$\|\varphi(s,a)\|_2\leq1$ and parameters $\omega=\{\omega_h\}_{h=1}^H$, define
\begin{equation}
\pi_{h}(a\mid s,\omega)
=
\frac{\exp\bigl(\langle\varphi(s,a),\omega_h\rangle\bigr)}
{\sum_{a'\in\mathcal{A}}\exp\bigl(\langle\varphi(s,a'),\omega_h\rangle\bigr)}.
\label{eq:master-actor-policy}
\end{equation}
We denote the corresponding policy by $\pi(\omega)$. The real-valued scores supplied to a softmax are called logits; here,
for a log-linear policy, the logit for action $a$ at state $s$ and stage $h$ is $\langle\varphi(s,a),\omega_h\rangle$.

\paragraph{Regularised primal-dual formulation.} To obtain a bounded dual domain, we impose the standard Slater condition.
\begin{assumption}[Slater condition]
\label{ass:slater}
There exist a policy \(\bar{\pi}\in\Pi\) and a constant \(\xi>0\) such that
$V_{u_i}^{\bar\pi}\geq c_i+\xi$ for every $i\in[I]$.
\end{assumption}
It implies $\|\lambda^\star\|_1\leq H/\xi$ for an optimal dual solution $\lambda^\star$
\citep{ying2022dual}; hence, we work on $\Lambda\triangleq[0,\lambda_{\max}]^I$ with
$\lambda_{\max}\geq H/\xi$.
By strong duality, the constrained problem \eqref{eq:cmdp-primal} is equivalent to the Lagrangian
saddle-point problem \citep{altman1999constrained,paternain2019constrainedreinforcementlearningzero}
\begin{equation}
\label{eq:lagrangian}
\max_{\pi\in\Pi}\min_{\lambda\in\Lambda} L(\pi,\lambda),
\qquad
L(\pi,\lambda)
\triangleq
V_r^\pi + \lambda^\top (V_u^\pi - c).
\end{equation}
Although \eqref{eq:lagrangian} involves the scalarized reward \(r+\lambda^\top u\), it is important
not to confuse the CMDP with a single unconstrained MDP obtained by fixing \(\lambda\).
optimising a fixed scalarization is generally not equivalent to solving the constrained problem, since
the appropriate tradeoff must be adjusted through the dual dynamics
\citep{altman1999constrained,ding2024lastiterateconvergentpolicygradient}. Because vanilla
primal-dual dynamics can oscillate, we add entropy regularisation to the primal variable and
quadratic regularisation to the dual variable to stabilise the last iterate
\citep{ding2024lastiterateconvergentpolicygradient}.

For any policy \(\pi\), define the entropy cost
$
\psi_h^\pi(s,a) \triangleq - \log \pi_h(a\mid s)
$,  and the entropy-value function
$
V_{\psi^\pi}^\pi
\triangleq
\mathbb{E}_\pi\left[\sum_{h=1}^H \psi_h^\pi(s_h,a_h)\,\middle|\, s_1\right].
$
For a regularisation parameter \(\tau>0\), define the regularised Lagrangian
\begin{equation}
\label{eq:lagrangian-reg}
L_\tau(\pi,\lambda)
\triangleq
V_r^\pi
+
\lambda^\top (V_u^\pi-c)
+
\tau V_{\psi^\pi}^\pi
+
\frac{\tau}{2}\|\lambda\|_2^2,
\qquad
(\pi,\lambda)\in\Pi\times\Lambda.
\end{equation}
Let \((\pi_\tau^\star,\lambda_\tau^\star)\) denote a saddle point of \(L_\tau\) over
\(\Pi\times\Lambda\). Given the current iterate \((\pi^k,\lambda_k)\), define $
\psi_{k,h}(s,a) \triangleq -\log \pi_h^k(a\mid s),
\;z_k \triangleq r + \lambda_k^\top u + \tau \psi_k.
$
Following \citep{müller2024trulynoregretlearningconstrained}, the exact primal-dual update admit
closed-form expressions
\begin{equation}
\label{eq:reg-primal}
\pi_h^{k+1}(\cdot \mid s)
\propto
\pi_h^k(\cdot \mid s)\exp\bigl(\eta Q_{z_k,h}^{\pi^k}(s,\cdot)\bigr),
\quad
\lambda_{k+1}
=
\operatorname{proj}_{\Lambda}
\bigl((1-\eta\tau)\lambda_k - \eta (V_u^{\pi^k}-c)\bigr).
\end{equation}
Moreover, uniform initialization and exponential primal update ensure that
\(\pi_h^k(a\mid s)>0\) for every $(k,h,s,a)$, thus \(\psi_{k,h}(s,a)\) is finite.


\section{Master Algorithm and Analysis}
\label{section:master-rpgpd}
So far, we have presented the regularised primal-dual formulation in an idealised setting.
We now move to the online setting, where the CMDP is unknown, and policy gradients must be estimated from data collected through interaction with the environment. 
To separate the convergence analysis from model-specific estimation and optimisation details, we formulate an oracle-based algorithm with two abstract components: an optimistic policy-evaluation oracle (critic) and an actor-fitting oracle. 
We state sufficient conditions on the critic and actor oracles. 
Under these conditions, we derive a master contraction recursion that makes explicit how the two sources of error affect the last iterate. 
Concrete policy-evaluation constructions are developed in Sections \ref{section:linear-cmdp-explicit-actor}
 and \ref{section:general-function-approximation}.

\paragraph{Master algorithm.}
Let $\KL_{k,h}(s) \triangleq \KL\bigl(\pi_{\tau,h}^{\star}(\cdot \mid s) \,\|\, \pi_h^k(\cdot \mid s)\bigr) $, and define the potential
$
\Phi_k
\triangleq 
\sum_{h=1}^H
\mathbb E_{s \sim d_h^{\pi_{\tau}^{\star}}}
\Bigl[
\KL_{k,h}(s)
\Bigr]
+
\frac{1}{2}\|\lambda_k-\lambda_{\tau}^{\star}\|_2^2.
$
Let $\omega^k=(\omega_h^k)_{h=1}^H$ and
$Q_{z_k}^k=(Q_{z_k,h}^k)_{h=1}^H$ denote the collections of stagewise actor
parameters and critic estimates, respectively. At iteration $k$, the Online
Policy-Evaluation oracle ($\mathrm{OPE}$) returns the optimistic critic
$Q_{z_k}^k$, together with the utility estimates $V_u^k$. The actor-fitting
oracle $\mathrm{ActorFit}$ then returns $\omega^{k+1}$ by fitting the
$H$ stagewise logit targets
$\{\langle\varphi,\omega_h^k\rangle+\eta Q_{z_k,h}^k\}_{h=1}^H$.
The on- and off-policy variants differ only in the dataset supplied
to the $\mathrm{OPE}$: in the on-policy mode, $\cD^k$ consists of a fresh batch of
$N$ independent trajectories generated by $\pi^k$; in the off-policy mode,
$\cD^k$ augments $\cD^{k-1}$ with a single trajectory generated by $\pi^k$.
The resulting Regularsied Policy-Gradient Primal Dual (RPGPD) procedure is presented in
Algorithm~\ref{alg:master-rpgpd}.

\begin{algorithm}[t]
\caption{Online \texttt{RPGPD}}
\label{alg:master-rpgpd}
\begin{algorithmic}[1]
\STATE \textbf{Input}: $K, N,\eta,\tau,\Lambda$, a sampling scheme (on-policy or off-policy), $\mathrm{ActorFit}$, and $\mathrm{OPE}$.
\STATE \textbf{Initialize}: set $\omega_h^1=0$ for every $h$, $\lambda_1=0$, and
$\cD^0=\varnothing$; define $\pi^1=\pi(\omega^1)$.
\FOR{$k=1,\ldots,K-1$}
\STATE Collect data:
\(
\cD^k \leftarrow
\begin{cases}
\text{on-policy}: &\{N\text{ fresh traj. } \sim \pi^k\},
\\
\text{off-policy}: &\cD^{k-1}\cup
\{\text{1 fresh traj. } \sim \pi^k\}.
\end{cases}
\)
\STATE Update critic: $\{Q_{z_k,h}^k\}_{h=1}^H,V_u^k\leftarrow
\mathrm{OPE}(\pi^k,\lambda_k,\cD^k)$.
\STATE Update actor: $\omega^{k+1}\leftarrow
\mathrm{ActorFit}(\omega^k,Q_{z_k}^k,\eta)$ and define
$\pi^{k+1}=\pi(\omega^{k+1})$.
\STATE Update dual variable: $\lambda_{k+1}\leftarrow
\proj_\Lambda[(1-\eta\tau)\lambda_k-\eta(V_u^k-c)]$.
\ENDFOR
\STATE \textbf{Output}: $\pi^K$.
\end{algorithmic}
\end{algorithm}

\paragraph{Policy-evaluation oracle.}
We first state the critic requirements used by the master analysis.
\begin{assumption}[Policy-evaluation oracle]
\label{ass:master-ope}
With probability at least $1-\delta$, simultaneously for every $k\in[K-1]$, the oracle returns
$Q_{z_k}^k$ and $V_u^k$ such that, with $Q_{z_k,H+1}^k\equiv0$ and
$V_{z_k,h}^k(s)\triangleq
\mathbb E_{a\sim\pi_h^k(\cdot\mid s)}[Q_{z_k,h}^k(s,a)]$ and
$V_{z_k}^k\triangleq V_{z_k,1}^k(s_1)$:
\begin{enumerate}
\item[(i)] \textbf{Optimism:} $Q_{z_k,h}^k\geq\Bell_{z_k,h}^{\pi^k}Q_{z_k,h+1}^k$ pointwise for every $h\in[H]$;
\item[(ii)] \textbf{Consistency:} there is $\beta_k\geq0$ for which
$V_{z_k}^k-V_{z_k}^{\pi^k}\leq\beta_k$ and
$\max_{i\in[I]}|V_{u_i}^k-V_{u_i}^{\pi^k}|\leq\beta_k$;
\item [(iii)] \textbf{Boundedness:} $0\leq V_{u_i}^k\leq H$ and, for every $(h,s)$,
\begin{equation}
\sum_{a\in\cA}\pi_h^k(a\mid s)
\exp\bigl(\eta Q_{z_k,h}^k(s,a)\bigr)
\bigl(Q_{z_k,h}^k(s,a)\bigr)^2
\leq C_{\eta,\tau,\Lambda,k}.
\label{eq:master-local-norm}
\end{equation}
\end{enumerate}
\end{assumption}

The three conditions capture standard requirements for optimistic online
policy evaluation. Optimism makes $Q_{z_k}^k$ an upper bound for $Q_{z_k}^{\pi^k}$,
a common mechanism in online RL for handling uncertainty and encouraging
exploration. Consistency keeps the estimated regularised return and utilities
close to their true values, with error controlled by $\beta_k$, while
boundedness controls the dual update and the second-order term of the
exponential actor update. Importantly, these are modular oracle-level
conditions: they do not prescribe a particular estimator or sampling scheme
and can be verified by suitable constructions for both on- and off-policy data
collection and for structured CMDP classes, including linear CMDPs.

\paragraph{Actor oracle.}
Having specified the critic requirements, we next characterize how accurately
the actor oracle must approximate the ideal policy update. Given $Q_{z_k,h}^k$, the ideal entropy-regularised mirror-ascent step is
\[
\widetilde\pi_h^{k+1}(a\mid s)
\propto
\pi_h^k(a\mid s)
\exp\bigl(\eta Q_{z_k,h}^k(s,a)\bigr),
\]
whereas $\mathrm{ActorFit}$ returns a policy $\pi^{k+1}$ in the log-linear policy class.  The master
analysis requires only that, for some $\epsilon_{\rm act}\geq0$ and every
$(k,h,s)\in[K-1]\times[H]\times\cS$, the fitted update satisfies
\begin{equation}
\KL\bigl(\pi_{\tau,h}^\star(\cdot\mid s)\,\|\,\pi_h^{k+1}(\cdot\mid s)\bigr)
-\KL\bigl(\pi_{\tau,h}^\star(\cdot\mid s)\,\|\,\widetilde\pi_h^{k+1}(\cdot\mid s)\bigr)
\leq \eta\epsilon_{\rm act}.
\label{eq:master-projection-bound}
\end{equation}
This condition measures the KL progress lost by projecting the ideal update
onto the log-linear class $\Pi_{\rm lin}$ and contributes $\eta H\epsilon_{\rm act}$ to the
master recursion.
Although explicit, the ideal update accumulates past critics and need not
remain in the log-linear class. Following \citet{lin2025optimisticworkshop},
$\mathrm{ActorFit}$ approximates each target increment $\eta Q_{z_k,h}^k$
using a linear combination of the actor features $\varphi$. Moreover, we
can relax this approximation requirement by observing that, for a
state-dependent baseline $b:\cS\to\mathbb R$, subtracting $b(s)$
multiplies every unnormalized action weight at state $s$ by the same
factor $e^{-\eta b(s)}$, which cancels upon normalization:
\[
\frac{
\pi_h^k(a\mid s)\exp\bigl(\eta(Q_{z_k,h}^k(s,a)-b(s))\bigr)
}{
\sum_{a'\in\cA}
\pi_h^k(a'\mid s)\exp\bigl(\eta(Q_{z_k,h}^k(s,a')-b(s))\bigr)
}
=
\widetilde\pi_h^{k+1}(a\mid s).
\]
In words, a log-linear policy may represent the action differences of
$Q_{z_k,h}^k$ without representing a state-dependent term $b(s)$ common
to all actions. Thus, choosing $b$ to remove such a term can reduce the
actor approximation error compared with fitting $Q_{z_k,h}^k$ directly as in \cite{lin2025optimisticworkshop}.

Furthermore, fitting the actor over every triplet in $\cS\times\cA^2$ can be
infeasible when the state space is large. Adapting the coreset approach
of \citet{lin2025optimisticworkshop} to action differences, we instead
evaluate the actor loss on a fixed, finite set of triplets $\cD_{\rm exp}$.

To state the actor approximation condition, for any triplet $(s,a,a')$, define the action differences
$\Delta\varphi(s,a,a')\triangleq\varphi(s,a)-\varphi(s,a')$ and
$\Delta Q_{z_k,h}^k(s,a,a')\triangleq Q_{z_k,h}^k(s,a)-Q_{z_k,h}^k(s,a')$.

\begin{assumption}[Log-linear actor approximation]
\label{ass:master-actor-oracle}
There exists a fixed finite weighted coreset $(\cD_{\rm exp},\rho_{\rm exp})$, where
$\cD_{\rm exp}\subset\cS\times\cA\times\cA$ and
$\rho_{\rm exp}\in\Delta(\cD_{\rm exp})$, with
\begin{equation*}
G_{\rm exp}
\triangleq
\sum_{\mathclap{(s,a,a')\in\cD_{\rm exp}}}
\rho_{\rm exp}(s,a,a')\Delta\varphi(s,a,a')\Delta\varphi(s,a,a')^\top,
\quad
\kappa_G
\triangleq
\sup_{s,a,a'}\|\Delta\varphi(s,a,a')\|_{G_{\rm exp}^{\dagger}}
<\infty.
\end{equation*}
We require
$\Delta\varphi(s,a,a')\in\operatorname{range}(G_{\rm exp})$ for every $(s,a,a')$.
Moreover, there is an $\epsilon_{\rm bias}\geq0$ such that
\begin{equation}
\max_{\substack{k\in[K-1]\\h\in[H]}}
\inf_{\substack{\omega\in\mathbb R^{d_{\rm a}} \\ b:\cS\to\mathbb R}}\sup_{s,a}
\left|
\langle\varphi(s,a),\omega-\omega_h^k\rangle
-\eta\bigl(Q_{z_k,h}^k(s,a)-b(s)\bigr)
\right|
\leq \eta\epsilon_{\rm bias}.
\label{eq:master-actor-bias}
\end{equation}
\end{assumption}
A detailed comparison explaining
why the Assumption \ref{ass:master-actor-oracle} is weaker than
that of \citet{lin2025optimisticworkshop} is deferred to
Appendix~\ref{app:On actor assumption}. 
The coreset $\cD_{\rm exp}$ and its corresponding weight distribution
$\rho_{\rm exp}$ can be constructed using $G$-optimal design, yielding
$|\cD_{\rm exp}|=\widetilde O(d_{\rm a})$ and
$\kappa_G=O(\sqrt{d_{\rm a}})$
\citep{lattimore2020goodfeatures,lin2025optimisticworkshop};
see Appendix~\ref{app:On actor assumption} for details.
Since this construction is addressed in the literature, we assume for
simplicity that $(\cD_{\rm exp},\rho_{\rm exp})$ is given.

In implementation, we fit differences in $Q$-values between actions,
which eliminate $b$ without estimating it.
In particular, for class $\Pi_{\rm lin}$, $\mathrm{ActorFit}$ computes
$\omega^{k+1}=(\omega_h^{k+1})_{h=1}^H$ by solving
\begin{equation}
\begin{aligned}
\omega_h^{k+1}
\in
\argmin_{w_h \in \mathbb R^{d_{\rm a}}}\sum_{(s,a,a')\in\cD_{\rm exp}}\rho_{\rm exp}(s,a,a')
\left(
\langle\Delta\varphi(s,a,a'),w_h-\omega_h^k\rangle
-\eta \Delta Q_{z_k,h}^k(s,a,a')
\right)^2.
\end{aligned}
\label{eq:explicit-actor-loss}
\end{equation}
Thus, the loss fits the ideal policy update $\widetilde\pi^{k+1}$ within the log-linear class $\Pi_{\rm lin}$, without estimating
the baseline function $b$.
Here, $\epsilon_{\rm bias}$ measures uniform approximation error, while $\kappa_G$ controls its amplification under
coreset fitting.

\begin{proposition}[Log-linear actor certificate]
\label{prop:explicit-actor-certificate}
Under Assumption~\ref{ass:master-actor-oracle}, the $\mathrm{ActorFit}$ update in
\eqref{eq:explicit-actor-loss} satisfies \eqref{eq:master-projection-bound} with
$\epsilon_{\rm act}=2(\kappa_G+1)\epsilon_{\rm bias}$.
\end{proposition}


The actor error $\epsilon_{\rm act}$ measures how accurately the log-linear actor 
represents and fits the ideal update.  It can decrease with a richer actor class 
and vanishes under exact representability and fitting, 
as for tabular actors with independent state-action logits, where $\epsilon_{\rm act}=0$.

\paragraph{Last-iterate convergence of Algorithm~\ref{alg:master-rpgpd}.}
Combining the abstract critic condition with the actor projection bound yields
a geometric recursion for the master potential. Under the additional Slater
and fixed-domain conditions, the same result converts this potential bound
into last-iterate reward suboptimality and constraint-violation guarantees
for the original CMDP.

\begin{theorem}[General bound for the master algorithm]
\label{thm:master-recursion}
Fix $K\geq 1$. Let
$
G_\tau\triangleq\sqrt{I}(H+\tau\lambda_{\max}).
$
Under Assumption~\ref{ass:master-ope}, suppose that $\mathrm{ActorFit}$
satisfies \eqref{eq:master-projection-bound} and $0<\eta\tau\leq1$.
Then, with probability at least $1-\delta$, 
Algorithm~\ref{alg:master-rpgpd} satisfies
\begin{equation}
\Phi_K
\leq
e^{-\eta\tau(K-1)}\Phi_1
+
\sum_{k=1}^{K-1}
(1-\eta\tau)^{K-1-k}
\left[
\frac{\eta^2}{2}
\bigl(HC_{\eta,\tau,\Lambda,k}+G_\tau^2\bigr)
+
\eta(1+I\lambda_{\max})\beta_k
+
\eta H\epsilon_{\rm act}
\right].
\label{eq:master-recursion}
\end{equation}
Moreover, if Assumption~\ref{ass:slater} holds, $0<\tau\leq1$, and
$\lambda_{\max}=2H\log(e|\cA|)/\xi$, then, for every $i\in[I]$,
\begin{equation}
\begin{aligned}
V_r^{\pi^\star}-V_r^{\pi^K}
\leq
H^{3/2}\sqrt{2\Phi_K}
+\tau H\log|\cA|,
\quad
c_i-V_{u_i}^{\pi^K}
\leq
H^{3/2}\sqrt{2\Phi_K}
+\frac{\tau H(1+\tau\log|\cA|)}{\xi}.
\end{aligned}
\label{eq:master-regularisation-conversion}
\end{equation}
\end{theorem}

The key distinction from 
\citet[Lemma~4.2]{müller2024trulynoregretlearningconstrained}
is that our choice
$\lambda_{\max}=2H\log(e|\cA|)/\xi$, independent of the target
accuracy $\epsilon$, gives the strict inequality
$\lambda_{\tau,i}^\star<\lambda_{\max}$ for every $i\in[I]$.
Therefore, dual optimality yields
$[c_i-V_{u_i}^{\pi_\tau^\star}]_+
=\tau\lambda_{\tau,i}^\star=O(\tau)$.
Thus, our conversion keeps $\lambda_{\max}$ fixed and permits
$\tau=\Theta(\epsilon)$, which is the primary source of the sharper $\epsilon$
dependence later in Section \ref{section:linear-cmdp-explicit-actor} and \ref{section:general-function-approximation}.
By contrast, \citet{müller2024trulynoregretlearningconstrained}'s radius-dependent bound balances terms
of order $1/\lambda_{\max}$ and $\tau\lambda_{\max}$, requiring
$\lambda_{\max}=\Theta(\epsilon^{-1})$ and
$\tau=\Theta(\epsilon^2)$ to attain $O(\epsilon)$ error.
Remark~\ref{rem:fixed-domain-conversion} gives the
detailed comparison, and Appendix~\ref{app:accuracy-improvement} records its
sample-complexity consequences.

Environment structure and sampling enter the recursion only through the
oracle certificates $\beta_k$, $C_{\eta,\tau,\Lambda,k}$, and
$\epsilon_{\rm act}$. The conversion bounds in Theorem~\ref{thm:master-recursion} then translate
a bound on $\Phi_K$ into last-iterate reward and constraint guarantees
for the original CMDP. Unlike the tabular model-based framework of
\citet{müller2024trulynoregretlearningconstrained}, this analysis supports
model-free evaluation with linear or general function approximation and
either on- or off-policy sampling, provided the oracle assumptions hold.

\section{Last-Iterate Convergence for Linear CMDPs}
\label{section:linear-cmdp-explicit-actor}
This section instantiates the policy-evaluation interface of
Section~\ref{section:master-rpgpd} under linear-CMDP realizability. We first
give a linear policy-evaluation routine shared by both sampling modes, followed
by joint oracle and last-iterate guarantees for fresh on-policy batches and
cumulative off-policy data.
Throughout this section, we work under the following linear-CMDP realizability assumption.
\begin{assumption}[Linear CMDP]
\label{ass:linear-cmdp}
There exist a known feature map \(\phi : \mathcal{S}\times\mathcal{A}\to\mathbb{R}^{d_{\rm c}}\),
signed measures \(\{\mu_h\}_{h=1}^H\) with \(\mu_h(\cdot)\in\mathbb{R}^{d_{\rm c}}\),
reward parameters \(\{\upsilon_{r,h}\}_{h=1}^H\subset\mathbb{R}^{d_{\rm c}}\), and utility parameters
\(\{\upsilon_{u_i,h}\}_{i \in [I],\, h \in [H]} \subset \mathbb{R}^{d_{\rm c}}\) such that for all
\((s,a,h)\),
$
P_h(\cdot \mid s,a) = \langle \phi(s,a), \mu_h(\cdot) \rangle,
\;
r_h(s,a) = \langle \phi(s,a), \upsilon_{r,h} \rangle,
\;
u_{i,h}(s,a) = \langle \phi(s,a), \upsilon_{u_i,h} \rangle .
$
Moreover, for all \((s,a,h)\),
$
\|\phi(s,a)\|_2 \le 1,
\;
\|\upsilon_{r,h}\|_2 \le \sqrt{d_{\rm c}},
\;
\|\upsilon_{u_i,h}\|_2 \le \sqrt{d_{\rm c}},
\;
\left\|
\left(
|\mu_h^{(1)}|(\mathcal S),\ldots,|\mu_h^{(d_{\rm c})}|(\mathcal S)
\right)
\right\|_2
\leq\sqrt{d_{\rm c}}.
$
\end{assumption}

\paragraph{Linear policy-evaluation oracle.}
The policy-evaluation computation is shared by both sampling modes. Given a trajectory dataset
$\cD$, let $\cD_h$ denote its stage-$h$ transitions and let $\Sigma_h^k$ be their regularised
empirical covariance, as specified in Algorithm~\ref{alg:explicit-linear-ope}. Initialize
$V_{r,H+1}^k=V_{u_i,H+1}^k=V_{\psi,H+1}^k=0$. Index the samples by $\ell$ and set
$y_{r,h}^\ell=r_h^\ell$, $y_{u_i,h}^\ell=u_{i,h}^\ell$, and $y_{\psi,h}^\ell=0$. Then, for
$j\in\{r,u_1,\ldots,u_I,\psi\}$, the ridge-regression coefficients are

\begin{equation}
\widehat\theta_{j,h}^k
=(\Sigma_h^k)^{-1}
\sum_{\ell=1}^{|\cD_h|}
\phi(s_h^\ell,a_h^\ell)
\bigl(y_{j,h}^\ell+V_{j,h+1}^k(s_{h+1}^\ell)\bigr).
\label{eq:explicit-linear-ope-regressions}
\end{equation}

Algorithm~\ref{alg:explicit-linear-ope} summarizes the resulting backward recursion. At each stage,
it evaluates the coefficients in \eqref{eq:explicit-linear-ope-regressions}, adds elliptical bonuses,
truncates the component critics to their natural ranges, and propagates their policy averages. The
reward, utility, and entropy components are estimated separately before being combined using
$\lambda_k$ and $\tau$. In particular, the immediate entropy cost $\psi_{k,h}$ is known from
$\pi^k$, so only its value function $V_{\psi,h}^k$ is estimated.

\begin{algorithm}[t]
\caption{Linear $\mathrm{OPE}(\pi^k,\lambda_k,\cD)$}
\label{alg:explicit-linear-ope}
\begin{algorithmic}[1]
\STATE \textbf{Input}: $\pi^k,\lambda_k$, a trajectory dataset $\cD$, and confidence radii
$\alpha_r,\alpha_u,\alpha_\psi$.
\STATE Set $V_{r,H+1}^k=V_{u_i,H+1}^k=V_{\psi,H+1}^k=0$ for every $i\in[I]$.
\FOR{$h=H,H-1,\ldots,1$}
\STATE Set
$\Sigma_h^k\leftarrow I_{d_{\rm c}}+\sum_{(s_h,a_h,r_h,u_h,s_{h+1})\in\cD_h}
\phi(s_h,a_h)\phi(s_h,a_h)^\top$.
\STATE Compute the regression coefficients in \eqref{eq:explicit-linear-ope-regressions}.

\STATE Set $\alpha_{u_i}\equiv\alpha_u$ for $i\in[I]$ and
$b_{j,h}^k(s,a)=\alpha_j\|\phi(s,a)\|_{(\Sigma_h^k)^{-1}}$ for
$j\in\{r,u_1,\ldots,u_I,\psi\}$. For $j\in\{r,u_1,\ldots,u_I\}$, define
\begin{equation}
\begin{aligned}
Q_{j,h}^k
&=\mathrm{Truncate}_{[0,H+1-h]}
\bigl(\langle\phi,\widehat\theta_{j,h}^k\rangle+b_{j,h}^k\bigr),
\\
Q_{\psi,h}^k(s,a)
&=\psi_{k,h}(s,a)
+\mathrm{Truncate}_{[0,(H-h)\log|\cA|]}
\bigl(\langle\phi(s,a),\widehat\theta_{\psi,h}^k\rangle+b_{\psi,h}^k(s,a)\bigr).
\end{aligned}
\label{eq:explicit-linear-ope-components}
\end{equation}
\STATE Set $V_{j,h}^k(s)=\mathbb E_{a\sim\pi_h^k(\cdot\mid s)}[Q_{j,h}^k(s,a)]$ for
$j\in\{r,u_1,\ldots,u_I,\psi\}$.
\ENDFOR
\STATE Return $Q_{z_k,h}^k=Q_{r,h}^k+\lambda_k^\top Q_{u,h}^k+\tau Q_{\psi,h}^k$ and
$V_u^k=(V_{u_1,1}^k(s_1),\ldots,V_{u_I,1}^k(s_1))$.
\end{algorithmic}
\end{algorithm}

\paragraph{On- and off-policy guarantees.}
The same linear oracle supports both modes. At iteration $k$, on-policy
evaluation uses $N$ fresh trajectories, whereas off-policy evaluation appends
one trajectory to the cumulative dataset. The following lemma gives the
corresponding certificates.
\begin{lemma}[Linear policy-evaluation certificates]
\label{lem:explicit-linear-ope-certificates}
Under Assumptions~\ref{ass:master-actor-oracle} and
\ref{ass:linear-cmdp}, suppose $0<\tau\leq1$, $0<\eta\tau\leq1/4$, and
$
\eta H(1+I\lambda_{\max}+\tau\log|\cA|)\leq1/4.
$
For each sampling mode, there exist mode-specific confidence radii such that
Algorithm~\ref{alg:explicit-linear-ope} satisfies
Assumption~\ref{ass:master-ope} with probability at least $1-\delta$.
For the on-policy mode,
\begin{equation}
\sup_{k<K}\beta_k
=
\widetilde O\left(
\bigl(1+I\lambda_{\max}+\tau\log|\cA|\bigr)
\frac{d_{\rm c}^{3/2}H^2}{\sqrt N}
\right).
\label{eq:explicit-linear-on-policy-ope-rate}
\end{equation}
For the off-policy mode,
\begin{equation}
\sum_{k=1}^{K-1}\beta_k
=
\widetilde O\left(
\bigl(1+I\lambda_{\max}+\tau\log|\cA|\bigr)H^2
\sqrt{d_{\rm c}(d_{\rm c}^2+d_{\rm a})K}
\right).
\label{eq:explicit-linear-off-policy-ope-rate}
\end{equation}
Moreover, in both modes,
$
\sup_{k<K}C_{\eta,\tau,\Lambda,k}
=
\widetilde O\left(
H^2\bigl(1+I\lambda_{\max}+\tau\log|\cA|\bigr)^2
\right).
\label{eq:explicit-linear-common-second-moment-rate}
$
\end{lemma}
\begin{theorem}[On- and off-policy linear-CMDP guarantees]
\label{thm:explicit-linear-last-iterate}
\label{thm:explicit-on-policy-last-iterate}
\label{thm:explicit-final-last-iterate}
Fix $\epsilon,\delta\in(0,1)$. Under Assumptions~\ref{ass:slater},
\ref{ass:linear-cmdp}, and \ref{ass:master-actor-oracle}, run
Algorithm~\ref{alg:master-rpgpd} with Algorithm~\ref{alg:explicit-linear-ope} in either sampling
mode and use the corresponding parameter choices in the appendix. Then, with probability at least
$1-\delta$,
\begin{equation}
\max\left\{
\bigl[V_r^{\pi^\star}-V_r^{\pi^K}\bigr]_+,
\max_{i\in[I]}\bigl[c_i-V_{u_i}^{\pi^K}\bigr]_+
\right\}
\leq\widetilde O\left(
\epsilon+
\frac{H^{5/3}}{\xi^{1/3}}\epsilon_{\rm act}^{1/3}
\right).
\label{eq:explicit-linear-final-error}
\end{equation}
For the on-policy mode, use the parameter choices in
\eqref{eq:explicit-on-policy-final-parameters} of
Appendix~\ref{app:explicit-on-policy-final-rate}. The total number of on-policy episode rollouts
satisfies
\begin{equation}
NK
=\widetilde O\left(
\frac{d_{\rm c}^3I^6H^{22}}{\xi^6\epsilon^6}
\max\left\{1,\min\left\{
\frac{H^4}{\xi^4\epsilon^4},
\frac{1}{H^{8/3}\xi^{8/3}
\epsilon_{\rm act}^{4/3}}
\right\}\right\}
\right).
\label{eq:explicit-on-policy-sample-complexity}
\end{equation}
For the off-policy mode, use the parameter choices in
\eqref{eq:explicit-final-parameters} of Appendix~\ref{app:explicit-off-policy-final-rate}. The number
of off-policy episode rollouts satisfies
\begin{equation}
K
=\widetilde O\left(
\frac{d_{\rm c}(d_{\rm c}^2+d_{\rm a})I^4H^{14}}
{\xi^4\epsilon^4}
\max\left\{1,\min\left\{
\frac{H^2}{\xi^2\epsilon^2},
\frac{1}
{H^{4/3}\xi^{4/3}\epsilon_{\rm act}^{2/3}}
\right\}\right\}
\right).
\label{eq:explicit-final-sample-complexity}
\end{equation}
\end{theorem}

Beyond any dependence encoded by the representation dimensions $d_{\rm c}$ and $d_{\rm a}$, both
guarantees depend on $|\cA|$ only through logarithmic factors and contain no separate polynomial
factor in $|\cA|$.  This is a substantial improvement in action-space dependence over the online
tabular bound of \citet[Theorem~5.1]{müller2024trulynoregretlearningconstrained}, which is
polynomial in $|\cA|$. We refer the reader to Appendix~\ref{app:actor-action-dependence} for a
detailed explanation of this improvement. 

Moreover, the sharper dependence on $\epsilon$ follows from the tighter
conversion in Theorem~\ref{thm:master-recursion} and the contraction analysis. 
Unlike the radius-dependent conversion of
\citet[Lemma~4.2]{müller2024trulynoregretlearningconstrained}, our argument
avoids the inverse-radius term and preserves geometric contraction with less
conservative choices of $\tau$ and $\eta$.
In particular, under a direct tabular specialization, our off-policy bound scales as
$\widetilde O(\epsilon^{-6})$, improving upon the
$\widetilde O(\epsilon^{-14})$ accuracy dependence implied by the
model-based tabular analysis \citep{müller2024trulynoregretlearningconstrained}. 
Appendix~\ref{app:accuracy-improvement} provides the
full sample-complexity and tabular comparison.

Although our off-policy sample complexity does not attain the
$\Omega(\epsilon^{-2})$ lower bound \citep{pmlr-v132-domingues21a}, its accuracy dependence matches the
$O(\epsilon^{-6})$ iteration complexity of RPGPD
\citep{ding2024lastiterateconvergentpolicygradient}. Thus, sampling adds only
an additional multiplicative factor depending on the representation complexity. 

\section{General Function Approximation}
\label{section:general-function-approximation}
This section extends the policy-evaluation construction of
Section~\ref{section:linear-cmdp-explicit-actor} to general critic classes while keeping the
master algorithm and log-linear actor of Section~\ref{section:master-rpgpd}.  Here, we state
only the function-class condition, the oracle idea, and the final guarantees; precise
definitions, algorithms, oracle certificates, and proofs are deferred to
Appendix~\ref{app:general-function-approximation}.

\paragraph{Function classes and complexity.}
Let $\mathcal X\triangleq\cS\times\cA$ and
$\mathcal J=\{r,u_1,\ldots,u_I,\psi\}$.  For every $h\in[H]$, let
$\mathcal F_{r,h}$ and $\mathcal F_{u_i,h}$ take values in $[0,H+1-h]$, and let
$\mathcal F_{\psi,h}$ take values in $[0,(H-h)\log|\cA|]$.  The entropy class models only $P_hV_{\psi,h+1}^k$, since the known immediate
cost $\psi_{k,h}(s,a)$ is included explicitly in the
action-value estimate.

\begin{assumption}[Value closeness]
\label{ass:general-bellman-closeness}
For every $h\in[H]$, every function $V:\cS\to[0,H-h]$, and every
function $W:\cS\to[0,(H-h)\log|\cA|]$, the following inclusions hold:
\begin{equation*}
\begin{aligned}
r_h+P_h V
\in\mathcal F_{r,h},
\qquad
u_{i,h} + P_h V
\in\mathcal F_{u_i,h},
\, i\in[I],
\qquad
P_h W
\in\mathcal F_{\psi,h}.
\end{aligned}
\end{equation*}

\end{assumption}
This is a standard componentwise value-closeness condition, similar to that of
\citet{liu2023optimistic}. We summarize the
critic complexity by the scale-dependent eluder dimension $d_{\rm E}$, which controls sequential
confidence widths, and the domain metric entropy $\mathfrak h_{\mathcal X}$, which controls the
discretization of data-dependent bonuses.  Their precise definitions and evaluation scales are given
in Appendix~\ref{app:general-complexity}.
\paragraph{General optimistic policy evaluation.}
The general oracle fits the reward, utility, and entropy-value critics backward, adds a
nonnegative stable bonus to each component, adds $-\log\pi_h^k$ exactly to the entropy critic, and
scalarizes only after componentwise fitting.  
The complete construction is Algorithm~\ref{alg:general-ope} in
Appendix~\ref{app:general-ope}; its on- and off-policy certificates are proved in
Appendices~\ref{app:general-on-policy} and \ref{app:general-replay}.

\paragraph{Last-iterate guarantees.}
We obtain the sample complexity stated in the theorem below.

\begin{theorem}[General last-iterate guarantees]
\label{thm:general-last-iterate}
Fix $\epsilon,\delta\in(0,1)$.  Under Assumptions~\ref{ass:slater},
\ref{ass:master-actor-oracle}, and \ref{ass:general-bellman-closeness},
Algorithm~\ref{alg:master-rpgpd}, instantiated with either the on-policy or off-policy general
oracle and the corresponding parameter choices specified below, guarantees, with probability at
least $1-\delta$,
\begin{equation}
\max\left\{
\bigl[V_r^{\pi^\star}-V_r^{\pi^K}\bigr]_+,
\max_{i\in[I]}\bigl[c_i-V_{u_i}^{\pi^K}\bigr]_+
\right\}
 \leq
\widetilde O\left(
\epsilon+\frac{H^{5/3}}{\xi^{1/3}}\epsilon_{\rm act}^{1/3}
\right).
\label{eq:general-final-error}
\end{equation}
For the on-policy oracle, use the parameter choices in
\eqref{eq:general-on-policy-parameters} of Appendix~\ref{app:general-on-policy}.  The total number
of on-policy episode rollouts satisfies
\begin{equation}
NK
 =
\widetilde O\left(
\frac{d_{\rm E}^3\mathfrak h_{\mathcal X}I^6H^{22}}
{\xi^6\epsilon^6}
\max\left\{1,\min\left\{
\frac{H^4}{\xi^4\epsilon^4},
\frac{1}{H^{8/3}\xi^{8/3}
\epsilon_{\rm act}^{4/3}}
\right\}\right\}
\right).
\label{eq:general-on-policy-sample-complexity}
\end{equation}
For the off-policy oracle, use the parameter choices in
\eqref{eq:general-replay-parameters} of Appendix~\ref{app:general-replay}.  The number of
off-policy episodes satisfies
\begin{equation}
K
 =
\widetilde O\left(
\frac{d_{\rm E}(d_{\rm E}^2\mathfrak h_{\mathcal X}+d_{\rm a})I^4H^{14}}
{\xi^4\epsilon^4}
\max\left\{1,\min\left\{
\frac{H^2}{\xi^2\epsilon^2},
\frac{1}{H^{4/3}\xi^{4/3}
\epsilon_{\rm act}^{2/3}}
\right\}\right\}
\right).
\label{eq:general-replay-sample-complexity}
\end{equation}
\end{theorem}

The theorem extends last-iterate convergence beyond linear realizability:
its statistical complexity is governed by the critic-class quantities
$d_{\rm E}$ and $\mathfrak h_{\mathcal X}$, rather than the number of states.


\section{Experiment} \label{sec:experiment}
We conduct a synthetic experiment to examine the mechanism and
empirical effectiveness of regularisation. We consider an
episodic linear CMDP with unknown dynamics under fresh on-policy and cumulative
off-policy sampling. The learner uses sample-based optimistic policy evaluation,
while the hidden model is used only for post-hoc evaluation.
The log-linear actor is fitted using the action-difference objective
in~\eqref{eq:explicit-actor-loss}.
We present the setup and results here and defer all
details to Appendix \ref{app:experiment-details}.

\begin{figure}[h]
    \centering
    \includegraphics[width=\linewidth]{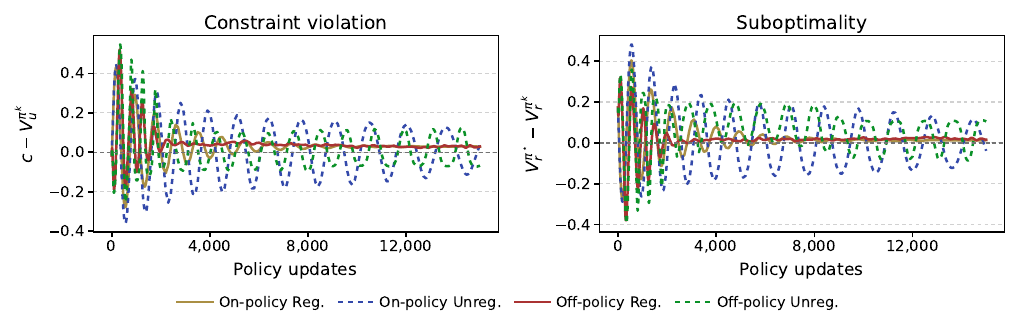}
    \caption{Online last-iterate performance in the synthetic linear CMDP over
    $1.5\times10^4$ policy updates.}
    \label{fig:small-linear-cmdp}
\end{figure}

We examine the last-iterate behaviour of the primal-dual dynamics on a small episodic linear CMDP whose optimal constrained policy
and value functions can be computed exactly.
We compare the regularised primal-dual method with a vanilla baseline using \(\tau=0\),
both trained at the same constraint threshold $c$. Both use sample-based optimistic
policy evaluation; exact evaluation is used only to assess the resulting policies.
Figure~\ref{fig:small-linear-cmdp} reports the signed constraint gap
\(c-V_u^{\pi^k}\) and signed reward gap
\(V_r^{\pi^\star}-V_r^{\pi^k}\).
Since pointwise averaging across seeds can mask fluctuations,
Figure~\ref{fig:small-linear-cmdp} shows one seed in each sampling mode.
The vanilla baseline (dashed lines) exhibits larger oscillations, whereas
the regularised configuration (solid lines) damps its initial oscillations
more quickly and yields substantially flatter trajectories in both
sampling modes.
We further compare within-run fluctuations of the unregularised and
regularised methods using the 30-seed statistic summarised in
Table~\ref{tab:small-linear-cmdp-variability} of
Appendix~\ref{app:experiment-details}.
In summary, regularised
configurations substantially reduce last-iterate oscillations in both settings.

\section{Conclusion}
We developed a model-free framework for last-iterate convergence in structured CMDPs,
with a computationally efficient linear-CMDP instantiation. By separating regularised
primal-dual dynamics from modular actor and critic conditions, it accommodates
on- and off-policy sampling and yields guarantees for linear CMDPs and general
function approximation. Empirically, our synthetic experiment illustrates stable last-iterate behaviour.
Extending the framework to continuous actions and neural actor-critic models
is a promising direction toward practical application.
\newpage

\subsection*{AI use statement}
We used AI tools to assist with the design and implementation of the
experimental code and to improve the manuscript's grammar, word choice,
and clarity. We take responsibility for the final content of this work,
including the text, claims, code, results, produced with the aid of generative AI.

\bibliographystyle{plainnat}
\bibliography{Ref}

\clearpage
\appendix

\thispagestyle{empty}
\renewcommand{\thetable}{\Alph{table}}
\setcounter{table}{0}

\renewcommand*\contentsname{Contents of Appendix}
\addtocontents{toc}{\protect\setcounter{tocdepth}{2}}
\doublespacing

\begingroup
\makeatletter
\renewcommand{\@pnumwidth}{2.5em}
\renewcommand{\@tocrmarg}{3.5em}
\makeatother
\tableofcontents
\endgroup
\singlespacing

\newpage
\section{Comprehensive Related Works}
\label{sec:Comprehensive related-work}
\begin{table}[htbp]
\centering
\caption{Comparison of online CMDP results with last-iterate guarantees for the final policy.
We include only works that explicitly account for exploration through online interaction
with an unknown CMDP. For a fair comparison with these prior tabular, model-based results,
we report sample complexity in the tabular setting using one-hot state-action features
for the actor and critic, with $\epsilon_{\rm act}=0$.
Sample complexity counts the complete episodes required to obtain an
$O(\epsilon)$-optimal, $O(\epsilon)$-feasible final policy.
Rates suppress logarithmic factors and dependence on problem parameters other than $\epsilon$.
The \textit{Structured CMDPs} column identifies guarantees
based on features or function classes, while \textit{Model free} indicates that
the method does not explicitly estimate the transition kernel.}
\label{tab:related-work-cmdp-comparison}
\small
\setlength{\tabcolsep}{4pt}
\begin{tabular}{@{}lccccc@{}}
\toprule
 &  & \makecell{Sample complexity}
& \makecell{Model free} & \makecell{Structured CMDPs}  \\
\midrule
\citet{müller2024trulynoregretlearningconstrained}  
 & & $\widetilde O(\epsilon^{-14})$ & \xmark & \xmark \\
 \citet{kitamura2024policygradientprimaldualalgorithm}  
 & & $\widetilde O(\epsilon^{-7})$ & \xmark & \xmark \\
\citet{zuo2026flexdome} 
& & $\widetilde O (\epsilon^{-6})$  & \xmark & \xmark  \\
\hline \\
\textbf{This work (off-policy)} 
& & $\widetilde O(\epsilon^{-6})$ 
& \cmark & \cmark  \\
\textbf{This work (on-policy)} 
& & $\widetilde O(\epsilon^{-10})$ 
& \cmark & \cmark  \\
\hline \\
\makecell{Lower bound \\ \cite{pmlr-v132-domingues21a} }
& & $\widetilde O(\epsilon^{-2})$ 
& --- & ---  \\
\bottomrule
\end{tabular}
\end{table}

\paragraph{Importance of last-iterate convergence.} Last-iterate guarantees are particularly valuable for a parameterised actor.
For the log-linear class defined above, the pointwise average
$K^{-1}\sum_{k=1}^{K}\pi_h(\cdot\mid s,\omega^k)$ generally cannot be
represented as $\pi_h(\cdot\mid s,\omega)$ by one parameter set $\omega$.
Sampling a past actor at the start of each episode implements a mixture,
but requires storing those actors; averaging their parameters produces
a log-linear policy without inheriting the mixture's reward or feasibility
guarantee. This distinction is relevant to constrained applications such
as language-model fine-tuning that optimises helpfulness subject to a
harmful-response cost \citep{dai2024safe} and wireless communication, where a deployed control policy must respect
average transmission-power limits \citep{bura2022dope}. 
In these applications, the single paramaterised policy deployed after training must satisfy
the constraints, making a last-iterate guarantee particularly desirable.

\paragraph{Last-iterate convergence \textit{without} online exploration}
A related line of work studies last-iterate convergence under specified
gradient, rollout, or coverage conditions.
\citet{ding2024lastiterateconvergentpolicygradient} establish nonasymptotic
guarantees for regularised and optimistic policy-gradient primal-dual updates;
their sample-based extension assumes access to policy rollouts from a simulator.
\citet{moskovitz2023reload} use optimistic ascent-descent to stabilise the
primal-dual iterates. Their ``optimism'' is a correction to the saddle-point
update, distinct from confidence bounds for an unknown CMDP.
For general parameterised policies, \citet{montenegro2024lastiterate}
analyse last-iterate convergence under gradient-domination and
gradient-estimator assumptions, while \citet{mondal2026lastiterate} use
compatible-approximation and Fisher non-degeneracy conditions together with
sampled rollouts. Recent augmented-Lagrangian work by \citet{lu2026augmented}
also establishes last-iterate guarantees under coverage and
gradient-evaluation assumptions. These results do not jointly establish
confidence-guided exploration from adaptively collected trajectories in an
unknown CMDP and a final-policy guarantee that accounts for the resulting
estimation error.

By contrast, our method explicitly addresses exploration in an unknown CMDP: it
constructs optimistic Bellman estimates from the collected trajectories, whose
high-probability certificates drive exploration and interface with the
last-iterate primal-dual contraction under both fresh on-policy sampling and
cumulative off-policy data. Our general-function result therefore broadens the critic and
statistical-exploration axes under structural Bellman assumptions, fixed-coreset
leverage control, and an explicit policy-space actor-projection error.

\paragraph{Last-iterate convergence \textit{with} online exploration.}
The closest works to ours in Table~\ref{tab:related-work-cmdp-comparison}
are \citet{müller2024trulynoregretlearningconstrained},
\citet{kitamura2024policygradientprimaldualalgorithm}, and
\citet{zuo2026flexdome}. All three address exploration in an unknown
tabular CMDP using optimistic, model-based methods. \citet{müller2024trulynoregretlearningconstrained} are
closest in their use of regularised primal-dual updates and a policy-dual
contraction; their online tuning yields $\widetilde O(\epsilon^{-14})$
episodes for an $O(\epsilon)$-optimal, $O(\epsilon)$-feasible final policy.
\citet{kitamura2024policygradientprimaldualalgorithm}  
obtain a $\widetilde O(\epsilon^{-7})$ uniform-PAC bound,
which controls the number of episodes with inaccurate or infeasible
policies. \citet{zuo2026flexdome} use decaying safety margins to control online strong
regret and constraint violation. Their online analysis yields
$\widetilde O(\epsilon^{-6})$ episode scaling, matching the accuracy
exponent of our off-policy bound, although their algorithmic approach
and analysis differ from ours.

Our work retains the optimistic, regularised primal-dual structure while
estimating Bellman quantities directly from trajectories, without explicitly
constructing a transition model. High-probability critic certificates account
for exploration error in the last-iterate analysis under both fresh on-policy
sampling and cumulative off-policy data. This gives \textit{model-free guarantees}
for linear CMDPs and general function classes under the stated structural
assumptions, with a compact log-linear actor. In the tabular specialisation 
with state-action features and
$\epsilon_{\rm act}=0$, our episode bounds are
$\widetilde O(\epsilon^{-6})$ off-policy and
$\widetilde O(\epsilon^{-10})$ on-policy. The off-policy rate matches
the state-of-the-art rate of \citet{zuo2026flexdome}.

\paragraph{Guarantee criteria for online CMDPs.}
Several notions of safety and convergence appear in the online CMDP literature and are not
interchangeable. Standard weak regret and
cumulative violation control signed sums \citep{efroni2020explorationexploitationconstrainedmdps} and therefore
allow cancellation: an unsafe episode may be offset by an overly conservative one. Strong regret
and hard constraint violation instead sum positive-part errors, while zero cumulative violation and
episode-wise safety impose still different requirements on the learning trajectory
\citep{wei2022tripleq,ghosh24aToward,stradi2025optimal,kitamura2025episodewise}.
Last-iterate convergence concerns the reward and feasibility of the policy returned for deployment,
rather than the complete path followed during learning
\citep{ding2024lastiterateconvergentpolicygradient,müller2024trulynoregretlearningconstrained,montenegro2024lastiterate,lu2026augmented,
zuo2026flexdome,kitamura2024policygradientprimaldualalgorithm,liu2025nearoptimal}.
A uniform tail guarantee for all sufficiently
late iterates would imply a strong regret.
However, sublinear strong regret provides a guarantee for the uniform
mixture of policies, $\tfrac{1}{K}\sum_{k=1}^{K}\pi_k$, but does not
directly guarantee the final policy $\pi_K$.

\paragraph{Primal-dual methods.}
CMDPs provide a standard framework for reinforcement learning with explicit safety constraints, 
with roots in stochastic control and broad use in safe RL and 
robotics~\citep{altman1999constrained,achiam2017constrained,tessler2019reward,brunke2022safe,gu2022review}. 
Classical algorithmic approaches include LP-based, dual, 
and primal-dual methods~\citep{efroni2020explorationexploitationconstrainedmdps, paternain2019constrainedreinforcementlearningzero,paternain2022safe}. LP-based methods often provide strong theoretical guarantees \citep{efroni2020explorationexploitationconstrainedmdps}, but require planning over occupancy measures and therefore scale poorly with large state spaces.
Dual methods optimise over the Lagrange multipliers, 
but evaluating each dual iterate typically requires solving the associated unconstrained MDP 
or policy-optimisation problem, leading to an inner planning 
or policy-optimisation loop~\citep{paternain2019constrainedreinforcementlearningzero,ying2022dual}.
Entropy-regularised dual formulations and cutting-plane methods can improve the optimisation
dependence when sufficiently accurate planning or policy-optimisation subroutines are available
\citep{ying2022dual,gladin2023algorithm}.
Primal-dual methods instead update the policy and multipliers together in a single loop, 
making them especially compatible with policy-gradient and actor-critic 
implementations in large-scale 
RL~\citep{ding2020natural,ding2022convergence,liu2021policy,bai2022achievingzeroconstraintviolation,qiu2020upper,ghosh2022provably}.
This computational convenience, however, comes with a known limitation: 
vanilla primal-dual dynamics may oscillate, and most existing guarantees control averaged values, 
mixture policies, or weak regret rather than 
the final policy iterate \citep{ding2024lastiterateconvergentpolicygradient}. 
Our work builds on the efficiency of primal-dual methods while targeting 
a last-iterate guarantee for the final deployed policy.

\paragraph{Technical novelty compared with \citep{müller2024trulynoregretlearningconstrained}}
The online tabular analysis of \citet{müller2024trulynoregretlearningconstrained} is the closest
regularised optimistic predecessor. Both methods combine entropy and quadratic dual regularisation
and are built around an exponentiated ideal policy update and a policy-dual Lyapunov argument.
\citet{müller2024trulynoregretlearningconstrained} estimate a
tabular transition model and perform optimistic dynamic programming, whereas our critic is obtained
directly from trajectory data by Bellman regression. The master theorem isolates the initial-state
evaluation error $\beta_k$, the local exponential moment $C_{\eta,\tau,\Lambda,k}$, and the actor
error $\epsilon_{\rm act}$, thereby separating the last-iterate argument from a particular statistical
estimator. This permits representation-dependent linear and general-function instantiations and a
fixed-dimensional actor instead of an exact tabular policy. A further analytical difference is the
fixed dual domain $\lambda_{\max}=2H\log(e|\cA|)/\xi$, which yields an $O(\tau)$ regularisation bias
without making the dual radius accuracy-dependent. Holding the dimensions, horizon, constraint
count, and Slater gap fixed, the exact-actor tabular specialisation has target-accuracy dependence
$\widetilde O(\epsilon^{-10})$ on-policy and $\widetilde O(\epsilon^{-6})$ with off-policy data, compared
with the $\widetilde O(\epsilon^{-14})$ terminal-accuracy dependence implied by their online
analysis. This comparison concerns the returned policy; their strong-regret theorem additionally
controls cumulative positive-part errors. Appendix~\ref{app:accuracy-improvement} and the following
comparison with the prior tabular analysis give the detailed rate accounting.

\section{Discussion of the Statistical Improvements}
\label{sec:statistical-improvement-discussion}

\subsection{On actor assumption}
\label{app:On actor assumption}

\paragraph{Comparison with uniform raw-critic approximation.}
The comparison is with Assumption~4.1 of
\citep{lin2025optimisticworkshop}, which uniformly approximates the raw
critic increment. For the same actor features and critic targets,
choosing $b\equiv0$ gives, for every $(k,h)$,
\begin{equation}
\begin{aligned}
&\inf_{\omega,b}\sup_{s,a}
\left|\langle\varphi(s,a),\omega-\omega_h^k\rangle
-\eta\bigl(Q_{z_k,h}^k(s,a)-b(s)\bigr)\right|
\\
&\qquad\leq
\inf_{\omega}\sup_{s,a}
\left|\langle\varphi(s,a),\omega-\omega_h^k\rangle
-\eta Q_{z_k,h}^k(s,a)\right|.
\end{aligned}
\label{eq:actor-baseline-raw-comparison}
\end{equation}
Here $\omega\in\mathbb R^{d_{\rm a}}$ and $b:\cS\to\mathbb R$.
Their bias tolerance corresponds to $\eta\epsilon_{\rm bias}$ in our
normalization. Their printed assumption places the infimum before the
supremum over iterations and stages, and therefore also implies the
per-update raw bound on the right of
\eqref{eq:actor-baseline-raw-comparison}. Consequently, their uniform
raw-critic requirement implies \eqref{eq:master-actor-bias}.
The example below shows that the converse fails. Thus, our approximation condition is strictly
weaker than the uniform raw-critic requirement of \citet{lin2025optimisticworkshop}.

\paragraph{Why action differences suffice.}
Fix $(k,h,\omega)$ and write
$e(s,a)=\langle\varphi(s,a),\omega-\omega_h^k\rangle
-\eta Q_{z_k,h}^k(s,a)$. Since $\cA$ is finite,
\begin{equation}
\begin{aligned}
\inf_b\sup_{s,a}|e(s,a)+\eta b(s)|
&=\frac12\sup_{s,a,a'}|e(s,a)-e(s,a')|
\\
&=\frac12\sup_{s,a,a'}
\left|\langle\Delta\varphi(s,a,a'),\omega-\omega_h^k\rangle
-\eta\Delta Q_{z_k,h}^k(s,a,a')\right|.
\end{aligned}
\label{eq:actor-baseline-difference-equivalence}
\end{equation}
Indeed, $|e(s,a)-e(s,a')|\leq2\sup_a|e(s,a)+\eta b(s)|$
for every $b$, and equality in the first line is attained by
$b(s)=-(\max_a e(s,a)+\min_a e(s,a))/(2\eta)$.
Thus \eqref{eq:master-actor-bias} is equivalent to a uniform
action-difference approximation bound with tolerance
$2\eta\epsilon_{\rm bias}$.
This identity explains both the baseline-free loss
\eqref{eq:explicit-actor-loss} and the factor $2$ in the certificate
$\epsilon_{\rm act}=2(\kappa_G+1)\epsilon_{\rm bias}$.
The proof of Proposition~\ref{prop:explicit-actor-certificate} below
establishes the coreset transfer and the KL bound.

\paragraph{A strict counterexample.}
Consider $H=2$, $I=1$, $\cA=\{-,+\}$, and four states: an initial state
$s_{\rm init}$ (so $s_1=s_{\rm init}$), two second-stage states
$s_{\rm zero}$ and $s_{\rm sign}$, and an absorbing state $s_{\rm abs}$.
Only at $s_{\rm sign}$ does the feature depend on the action. Specifically, let
$x(s_{\rm sign},+)=1$, $x(s_{\rm sign},-)=-1$, and $x(s,a)=0$ elsewhere.
Take $\phi(s,a)=\varphi(s,a)=(1,x(s,a))^\top/\sqrt2$, so
$d_{\rm c}=d_{\rm a}=2$. Set
\[
\begin{gathered}
P_1(\cdot\mid s,a)=\tfrac12\delta_{s_{\rm zero}}+\tfrac12\delta_{s_{\rm sign}},
\qquad P_2(\cdot\mid s,a)=\delta_{s_{\rm abs}},
\\
r_1=0,\qquad r_2(s,a)=\tfrac{1+x(s,a)}2,\qquad u_{1,h}=1,\qquad c=0.
\end{gathered}
\]
Both second-stage states are reachable. This is a linear CMDP satisfying
Assumption~\ref{ass:linear-cmdp}: take
$\upsilon_{r,1}=0$, $\upsilon_{r,2}=(1,1)^\top/\sqrt2$,
$\upsilon_{u_1,h}=(\sqrt2,0)^\top$, and
$\mu_h=(\sqrt2 P_h,0)^\top$. All stated norm bounds hold, and Slater's
condition holds with $\xi=1$.

Fix $0<\eta\tau\leq1$, initialize $\omega_1^1=0$ and
$\omega_2^1=(0,\sqrt2 v_1)^\top$ with $v_1>0$, and use exact evaluation.
At iteration $k$, write
$\omega_2^k=(0,\sqrt2 v_k)^\top$, so
$\pi_2^k(a\mid s)\propto e^{v_k x(s,a)}$. The terminal shaped critic is
\begin{equation}
Q_{z_k,2}^k(s,a)
=\underbrace{\frac12+\lambda_k
+\tau\log\sum_{a'\in\cA}e^{v_k x(s,a')}}_{b_k(s)}
+\left(\frac12-\tau v_k\right)x(s,a).
\label{eq:actor-counterexample-critic}
\end{equation}
Consequently, choosing the baseline $b_k$ and
$\omega=\omega_2^k+\eta\sqrt2(0,\frac12-\tau v_k)^\top$
makes the baseline-adjusted error zero. At the first stage, the uniform
policy is preserved and the shaped critic is constant, so its error is
also zero. In contrast,
\[
\begin{aligned}
\frac{\varphi(s_{\rm sign},+)+\varphi(s_{\rm sign},-)}2
&=\varphi(s_{\rm zero},+),
\\
\frac{Q_{z_k,2}^k(s_{\rm sign},+)+Q_{z_k,2}^k(s_{\rm sign},-)}2
-Q_{z_k,2}^k(s_{\rm zero},+)
&=\tau\log\cosh v_k.
\end{aligned}
\]
The triangle inequality therefore gives
\begin{equation}
\inf_{\theta\in\mathbb R^2}\sup_{s,a}
\left|\langle\varphi(s,a),\theta\rangle-Q_{z_k,2}^k(s,a)\right|
\geq\frac{\tau}{2}\log\cosh v_k>0.
\label{eq:actor-counterexample-raw-bias}
\end{equation}
For the increment $\eta Q_{z_k,2}^k$, this lower bound is multiplied by
$\eta$. Thus our approximation condition holds with
$\epsilon_{\rm bias}=0$, while the uniform raw-critic condition cannot hold
with zero bias.

A single coreset triplet $(s_{\rm sign},+,-)$ of weight one gives
$G_{\rm exp}=\operatorname{diag}(0,2)$ and $\kappa_G=1$.
The fitted policy equals the ideal update, with slope
\[
v_{k+1}=(1-\eta\tau)v_k+\eta/2,
\qquad v_k\longrightarrow\frac1{2\tau}\quad\text{as }k\to\infty.
\]
In particular, $v_k>0$ for every $k$. Moreover,
$\frac{\tau}{2}\log\cosh(1/(2\tau))\longrightarrow1/4$ as
$\tau\downarrow0$.
Hence the actual projection error is zero even though the raw-critic bias
certificate need not vanish. This does not assert that linear-CMDP
realizability alone makes $\epsilon_{\rm bias}$ small: nonlinear bonuses
and clipping may still require richer actor features.

\paragraph{Coreset coverage.}
Only the span of $\{\Delta\varphi(s,a,a')\}$ affects the policy;
its orthogonal complement consists of action-independent feature
directions. The range condition in
Assumption~\ref{ass:master-actor-oracle} ensures that coreset fitting
controls every direction in this span, while $G_{\rm exp}^{\dagger}$
allows irrelevant directions to remain singular.
For a finite feature image of rank $m\geq2$, applying
\citet[Theorem~4.4]{lattimore2020goodfeatures} in this span gives
$|\cD_{\rm exp}|=\widetilde O(m)$ and $\kappa_G^2\leq2m$.
For a bounded, possibly infinite image, the same asymptotic claim follows
by a finite-net argument. In span coordinates, choose image vectors
$y_1,\ldots,y_m$ forming a basis and set
$A=m^{-1}\sum_j y_jy_j^\top\succ0$. Take a finite
$\zeta$-net containing this basis with
$\zeta\leq\sqrt{\lambda_{\min}(A)}$, and let $G$ be its design matrix.
Then, for every image vector $y$ and a net point $y'$ within $\zeta$,
\[
\operatorname{tr}(G^{-1}A)\leq2m,
\qquad
\|G^{-1}\|_{\rm op}\leq\frac{2m}{\lambda_{\min}(A)},
\qquad
\|y\|_{G^{-1}}
\leq\sqrt{2m}+\zeta\sqrt{\frac{2m}{\lambda_{\min}(A)}}
\leq\sqrt{8m}.
\]
For $m=1$, a singleton whose feature norm is at least half the supremum
gives $\kappa_G\leq2$; for $m=0$, any singleton gives
$G_{\rm exp}=0$ and $\kappa_G=0$. Since $m\leq d_{\rm a}$ and
$\|\Delta\varphi\|_2\leq2$, this proves the stated support and leverage
orders without a compactness assumption. It is an existence statement;
the coreset is treated as given, not as obtainable from an arbitrary
infinite feature image at no computational cost.

\subsection{Origin of the sharper accuracy dependence}
\label{app:accuracy-improvement}
The improvement begins with the conversion from the regularised saddle point
to the original CMDP. The conversion argument in
Theorem~\ref{thm:master-recursion} bounds
the multiplier $\lambda_\tau^\star$ using Slater's condition, so the box radius
$\lambda_{\max}=2H\log(e|\cA|)/\xi$ need not grow as the target accuracy
decreases. The regularisation bias is consequently $O(\tau)$, allowing
$\tau=\Theta(\epsilon)$ when the other problem parameters are fixed.
The radius-dependent conversion of
\citet[Lemma~4.2]{müller2024trulynoregretlearningconstrained} instead balances
$1/\lambda_{\max}$ against $\tau\lambda_{\max}$, leading to
$\lambda_{\max}=\Theta(\epsilon^{-1})$ and $\tau=\Theta(\epsilon^2)$.
Remark~\ref{rem:fixed-domain-conversion} gives the two conversion inequalities.

To isolate this effect, take $\epsilon_{\rm act}=0$, hold all remaining
problem parameters fixed, and use the same oracle certificates in both
analyses. In the on-policy case, the master recursion requires
\[
\eta=\widetilde\Theta\!\left(
\frac{\tau\epsilon^2}{\lambda_{\max}^2}\right),\qquad
K=\widetilde O\!\left(\frac1{\eta\tau}\right),\qquad
N=\widetilde O\!\left(
\frac{\lambda_{\max}^4}{\tau^2\epsilon^4}\right).
\]
Here the stepsize controls the local-norm term, geometric contraction
controls the initial error, and the batch size controls critic error.
For off-policy sampling, the cumulative certificate gives
$K=\widetilde O(\lambda_{\max}^4/(\tau^2\epsilon^4))$ with
$\eta=\widetilde\Theta(1/(\tau K))$. Substitution yields the following
accuracy exponents; entries suppress universal constants and logarithmic
factors.
\begin{center}
\begin{tabular}{lcc}
\toprule
Quantity & Radius-dependent conversion & Fixed-domain conversion \\
\midrule
Dual radius $\lambda_{\max}$ & $\epsilon^{-1}$ & $1$ \\
Regularisation $\tau$ & $\epsilon^2$ & $\epsilon$ \\
On-policy iterations $K$ & $\epsilon^{-8}$ & $\epsilon^{-4}$ \\
On-policy batch size $N$ & $\epsilon^{-12}$ & $\epsilon^{-6}$ \\
On-policy episodes $NK$ & $\epsilon^{-20}$ & $\epsilon^{-10}$ \\
Off-policy episodes $K$ & $\epsilon^{-12}$ & $\epsilon^{-6}$ \\
\bottomrule
\end{tabular}
\end{center}
These are comparisons within our master analysis, not a restatement of the
published rates. Both columns retain geometric contraction, isolating the
effect of the conversion. For our fixed-domain guarantees with actor
approximation, $\epsilon_{\rm act}\lesssim\xi\epsilon^3/H^5$ keeps the
approximation floor in Theorem~\ref{thm:explicit-linear-last-iterate} at the
target order.

\paragraph{Comparison with the prior tabular analysis.}
\label{app:tabular-comparison}
The online model-based result of
\citet[Theorem~5.1]{müller2024trulynoregretlearningconstrained} uses one
episode per iteration and proves $\widetilde O(K^{13/14})$ strong regret.
Their Remark~5.1 also identifies a last-iterate guarantee: the slowest
terminal error term under that tuning is $\widetilde O(K^{-1/14})$, giving
$\widetilde O(\epsilon^{-14})$ episodes. Their exact-value Theorem~4.1
counts optimisation iterations and is therefore a different comparison.

For a direct tabular specialisation of our theorem, take one-hot critic and
actor features, so $d_{\rm c}=d_{\rm a}=|\cS||\cA|$ and
$\epsilon_{\rm bias}=\epsilon_{\rm act}=0$. The off-policy guarantee then gives
\[
K=\widetilde O\!\left(
\frac{(|\cS||\cA|)^3I^4H^{16}}{\xi^6\epsilon^6}
\right).
\]
The sharper accuracy dependence comes from the fixed-domain conversion,
geometric contraction, and terminal-error calibration. It is an improvement
in the analysis, not an inherent statistical advantage of model-free
evaluation: a suitable model-based oracle could use the same master
recursion. The guarantees also have different scopes. The cited online
theorem controls cumulative strong regret during learning, whereas our
theorem controls the returned policy and, in this specialisation, pays cubic
dependence on $|\cS||\cA|$.

\subsection{Actor- and action-space dependence}
\label{app:actor-action-dependence}
\begin{remark}[Statistical roles of the actor class and action space]
\label{rem:actor-action-dependence}
For the fresh-batch on-policy oracle, conditionally on the pre-iteration history, $\pi^k$ is fixed
and $\cD^k$ contains $N$ independent trajectories generated by this single policy.  The uniform
confidence argument therefore covers only the $d_{\rm c}$-dimensional critic coefficient and the
elliptical-bonus geometry, whose covering dimension is of order $d_{\rm c}^2+d_{\rm c}$.  No cover
of the actor class is needed, which yields the $d_{\rm c}^{3/2}/\sqrt N$ consistency term in
\eqref{eq:explicit-linear-on-policy-ope-rate}.  With off-policy data, by contrast, the fitted
value function depends on policies generated from the accumulated data.  Uniform control over the reachable
actor class is then required and contributes the $d_{\rm a}$ term in
$\sqrt{d_{\rm c}(d_{\rm c}^2+d_{\rm a})K}$ in
\eqref{eq:explicit-linear-off-policy-ope-rate}.  Thus, $d_{\rm a}$ is absent from the on-policy statistical
rate but may enter the off-policy rate; in either setting it may also affect the actor approximation
error and the computational cost of $\mathrm{ActorFit}$.  The general-function guarantees exhibit
the same split:
\eqref{eq:general-on-policy-sample-complexity} has no $d_{\rm a}$ term, whereas
\eqref{eq:general-replay-sample-complexity} contains one in the off-policy setting.

In both sampling regimes, the remaining explicit action-space dependence arises from entropy and
normalization terms through powers of $\log(e|\cA|)$.  In particular, retaining the escort moment
$|\cA|^{\eta\tau}$ and imposing the stated stepsize condition keep this factor bounded by a
universal constant.  Hence, for fixed representation dimensions, neither guarantee contains a
polynomial factor in $|\cA|$, whereas the online tabular guarantee of
\citet[Theorem~5.1 and Appendix~F]{müller2024trulynoregretlearningconstrained} is polynomial in
$|\cA|$ (its leading $K^{13/14}$ coefficient includes $|\cA|^{1/4}$).  This comparison concerns the
explicit dependence on action cardinality: with one-hot tabular features,
$d_{\rm c}=d_{\rm a}=|\cS||\cA|$, so the representation dimensions themselves recover polynomial
dependence on $|\cA|$.
\end{remark}


\section{Proofs for Section~\ref{section:master-rpgpd}}

We first transfer the coreset fitting error to the actor certificate, then
derive the primal and dual inequalities that give the master recursion.
Finally, we convert distance to the regularised saddle point into reward
and constraint bounds. The constants in these three steps are retained
because they specify the oracle interface and the fixed dual radius.

\subsection{Actor projection and the master recursion}
\label{app:actor-projection}

\begin{proof}[Proof of Proposition~\ref{prop:explicit-actor-certificate}]
The proof controls the fitted logit differences uniformly, then transfers
that error through softmax. Fix $(k,h)$ and $\nu>0$. By
Assumption~\ref{ass:master-actor-oracle}, choose
$\omega_h^{k,0}\in\mathbb R^{d_{\rm a}}$ and $b:\cS\to\mathbb R$ such that, with
\[
e_0(s,a)
=
\langle\varphi(s,a),\omega_h^{k,0}-\omega_h^k\rangle
-\eta\bigl(Q_{z_k,h}^k(s,a)-b(s)\bigr),
\]
we have $\|e_0\|_\infty\leq\eta\epsilon_{\rm bias}+\nu$.
The baseline cancels in $e_0(s,a)-e_0(s,a')$. The normal equation for the exact minimizer in
\eqref{eq:explicit-actor-loss} gives
\begin{equation}
G_{\rm exp}(\omega_h^{k+1}-\omega_h^{k,0})
=
-\sum_{(s,a,a')\in\cD_{\rm exp}}
\rho_{\rm exp}(s,a,a')\Delta\varphi(s,a,a')
\bigl(e_0(s,a)-e_0(s,a')\bigr).
\label{eq:linear-coreset-normal-equation}
\end{equation}
Taking the inner product of
\eqref{eq:linear-coreset-normal-equation} with
$\omega_h^{k+1}-\omega_h^{k,0}$ and applying Cauchy-Schwarz gives
\[
\|\omega_h^{k+1}-\omega_h^{k,0}\|_{G_{\rm exp}}
\leq
\left(\sum_{(s,a,a')\in\cD_{\rm exp}}\rho_{\rm exp}(s,a,a')
\bigl(e_0(s,a)-e_0(s,a')\bigr)^2\right)^{1/2}
\leq2(\eta\epsilon_{\rm bias}+\nu).
\]
Since every $\Delta\varphi(s,a,a')$ belongs to
$\operatorname{range}(G_{\rm exp})$, the pseudoinverse Cauchy-Schwarz
inequality and the leverage bound give
\begin{equation}
\begin{aligned}
\left|\left\langle\Delta\varphi(s,a,a'),
\omega_h^{k+1}-\omega_h^{k,0}\right\rangle\right|
&\leq
\|\Delta\varphi(s,a,a')\|_{G_{\rm exp}^{\dagger}}
\|\omega_h^{k+1}-\omega_h^{k,0}\|_{G_{\rm exp}}
\\
&\leq
2\kappa_G(\eta\epsilon_{\rm bias}+\nu).
\end{aligned}
\label{eq:linear-coreset-prediction-transfer}
\end{equation}
Consequently, for
\[
\Delta_{k,h}(s,a)
=\langle\varphi(s,a),\omega_h^{k+1}-\omega_h^k\rangle
-\eta Q_{z_k,h}^k(s,a),
\]
the triangle inequality gives
$\sup_{s,a,a'}|\Delta_{k,h}(s,a)-\Delta_{k,h}(s,a')|
\leq2(\kappa_G+1)(\eta\epsilon_{\rm bias}+\nu)$.
Letting $\nu\downarrow0$ yields
\begin{equation}
\sup_{s,a,a'}|\Delta_{k,h}(s,a)-\Delta_{k,h}(s,a')|
\leq2\eta(\kappa_G+1)\epsilon_{\rm bias}
=\eta\epsilon_{\rm act}.
\label{eq:actor-residual-action-difference}
\end{equation}
No bound on the action-independent part of $\Delta_{k,h}$ is needed, since
\[
\pi_h^{k+1}(a\mid s)
=
\frac{\widetilde\pi_h^{k+1}(a\mid s)e^{\Delta_{k,h}(s,a)}}
{\mathbb E_{a'\sim\widetilde\pi_h^{k+1}(\cdot\mid s)}
[e^{\Delta_{k,h}(s,a')}] }.
\]
Therefore, for every state $s$ and $p\in\Delta(\cA)$, suppressing $s$ in the
next display,
\begin{align*}
\KL(p\|\pi_h^{k+1})-\KL(p\|\widetilde\pi_h^{k+1})
&=-\mathbb E_p[\Delta_{k,h}]
+\log\mathbb E_{\widetilde\pi_h^{k+1}}[e^{\Delta_{k,h}}]
\\
&\leq\max_a\Delta_{k,h}(s,a)-\min_a\Delta_{k,h}(s,a)
\leq\eta\epsilon_{\rm act}.
\end{align*}
Taking $p=\pi_{\tau,h}^\star(\cdot\mid s)$
proves \eqref{eq:master-projection-bound} with
$\epsilon_{\rm act}=2(\kappa_G+1)\epsilon_{\rm bias}$.
\end{proof}

\begin{proof}[Proof of the recursion bound in Theorem~\ref{thm:master-recursion}]
Work on the joint event in Assumption~\ref{ass:master-ope} and fix
$k\in[K-1]$. We first derive a one-step inequality.

\emph{Primal progress.}
The entropy identity
$V_{\psi_k}^{\pi_\tau^\star}-V_{\psi^{\pi_\tau^\star}}^{\pi_\tau^\star}
=\sum_h\mathbb E_{d_h^{\pi_\tau^\star}}[\KL_{k,h}]$
separates the regularisation term from the shaped reward $z_k$.
Define the Bellman residual
\[
\mathcal E_{k,h}(s,a)
\triangleq
Q_{z_k,h}^k(s,a)
-(\Bell_{z_k,h}^{\pi^k}Q_{z_k,h+1}^k)(s,a)
\geq0.
\]
Bellman telescoping gives
\begin{align*}
V_{z_k}^{\pi_\tau^\star}-V_{z_k}^k
={}&
\sum_{h=1}^H
\mathbb E_{s\sim d_h^{\pi_\tau^\star}}
\left[
\left\langle
\pi_{\tau,h}^\star(\cdot\mid s)-\pi_h^k(\cdot\mid s),
Q_{z_k,h}^k(s,\cdot)
\right\rangle
\right]
\\
&-
\sum_{h=1}^H
\mathbb E_{(s,a)\sim d_h^{\pi_\tau^\star}}
[\mathcal E_{k,h}(s,a)].
\end{align*}
Dropping the nonnegative residual term and using
$V_{z_k}^k-V_{z_k}^{\pi^k}\leq\beta_k$ yields
\begin{align*}
L_\tau(\pi_\tau^\star,\lambda_k)-L_\tau(\pi^k,\lambda_k)
&=V_{z_k}^{\pi_\tau^\star}-V_{z_k}^{\pi^k}
-\tau\sum_{h=1}^H\mathbb E_{d_h^{\pi_\tau^\star}}[\KL_{k,h}]
\\
&\leq\sum_{h=1}^H\mathbb E_{d_h^{\pi_\tau^\star}}
\bigl[\langle\pi_{\tau,h}^\star-\pi_h^k,Q_{z_k,h}^k\rangle
-\tau\KL_{k,h}\bigr]+\beta_k.
\end{align*}

For the ideal exponential update, the log-normaliser identity gives
\[
\KL_{k,h}
-\KL(\pi_{\tau,h}^\star\|\widetilde\pi_h^{k+1})
=\eta\langle\pi_{\tau,h}^\star,Q_{z_k,h}^k\rangle
-\log\mathbb E_{\pi_h^k}[e^{\eta Q_{z_k,h}^k}].
\]
Since $z_k\geq0$ and $Q_{z_k,H+1}^k=0$, optimism implies
$Q_{z_k,h}^k\geq0$ by backward induction.
Using $e^x\leq1+x+x^2e^x/2$ for $x\geq0$, $\log(1+x)\leq x$,
and the local-norm bound \eqref{eq:master-local-norm}, we obtain
\[
\log\mathbb E_{\pi_h^k}[e^{\eta Q_{z_k,h}^k}]
\leq\eta\langle\pi_h^k,Q_{z_k,h}^k\rangle
+\frac{\eta^2}{2}C_{\eta,\tau,\Lambda,k}.
\]
Combining the preceding three displays gives
\begin{align}
L_\tau(\pi_\tau^\star,\lambda_k)-L_\tau(\pi^k,\lambda_k)
&\leq
\frac1\eta\sum_{h=1}^H
\mathbb E_{d_h^{\pi_\tau^\star}}
\left[
\KL(\pi_{\tau,h}^\star\|\pi_h^k)
-\KL(\pi_{\tau,h}^\star\|\widetilde\pi_h^{k+1})
\right]
\nonumber\\
&\quad
-\tau\sum_{h=1}^H
\mathbb E_{d_h^{\pi_\tau^\star}}
\left[\KL(\pi_{\tau,h}^\star\|\pi_h^k)\right]
+\frac{\eta H}{2}C_{\eta,\tau,\Lambda,k}+\beta_k.
\label{eq:master-primal-ideal-bound}
\end{align}
The state argument of each KL divergence is implicit. By
\eqref{eq:master-projection-bound}, the first line on the right-hand side of
\eqref{eq:master-primal-ideal-bound} is at most
\begin{equation}
\frac1\eta\sum_{h=1}^H
\mathbb E_{d_h^{\pi_\tau^\star}}
\left[
\KL(\pi_{\tau,h}^\star\|\pi_h^k)
-\KL(\pi_{\tau,h}^\star\|\pi_h^{k+1})
\right]
+H\epsilon_{\rm act}.
\label{eq:master-projected-primal-bound}
\end{equation}

\emph{Dual progress.}
Set $g_k=V_u^k-c+\tau\lambda_k$. The dual update is
$\lambda_{k+1}=\proj_\Lambda(\lambda_k-\eta g_k)$.
Because $V_u^k,c\in[0,H]^I$ and $\lambda_k\in\Lambda$,
$\|g_k\|_2\leq G_\tau$. Nonexpansiveness of projection therefore gives
\[
\langle V_u^k-c+\tau\lambda_k,\lambda_k-\lambda_\tau^\star\rangle
\leq
\frac{\|\lambda_k-\lambda_\tau^\star\|_2^2
-\|\lambda_{k+1}-\lambda_\tau^\star\|_2^2}{2\eta}
+\frac{\eta G_\tau^2}{2}.
\]
Replacing $V_u^k$ by $V_u^{\pi^k}$ costs at most
$\|\lambda_k-\lambda_\tau^\star\|_1\beta_k
\leq I\lambda_{\max}\beta_k$.
Moreover,
\begin{align*}
&L_\tau(\pi^k,\lambda_k)-L_\tau(\pi^k,\lambda_\tau^\star)
+\frac\tau2\|\lambda_k-\lambda_\tau^\star\|_2^2
\\
&\hspace{30mm}=
\left\langle
V_u^{\pi^k}-c+\tau\lambda_k,
\lambda_k-\lambda_\tau^\star
\right\rangle.
\end{align*}
Combining this identity with the projection inequality yields
\begin{align}
L_\tau(\pi^k,\lambda_k)-L_\tau(\pi^k,\lambda_\tau^\star)
+\frac\tau2\|\lambda_k-\lambda_\tau^\star\|_2^2
&\leq
\frac{\|\lambda_k-\lambda_\tau^\star\|_2^2
-\|\lambda_{k+1}-\lambda_\tau^\star\|_2^2}{2\eta}
\nonumber\\
&\quad+\frac{\eta G_\tau^2}{2}+I\lambda_{\max}\beta_k.
\label{eq:master-dual-bound}
\end{align}
\emph{Contraction.}
Substituting \eqref{eq:master-projected-primal-bound} into
\eqref{eq:master-primal-ideal-bound}, adding \eqref{eq:master-dual-bound}, using the saddle
inequality
$L_\tau(\pi_\tau^\star,\lambda_k)
\geq L_\tau(\pi_\tau^\star,\lambda_\tau^\star)
\geq L_\tau(\pi^k,\lambda_\tau^\star)$,
and multiplying by $\eta$ gives
\begin{equation}
\Phi_{k+1}
\leq
(1-\eta\tau)\Phi_k
+\frac{\eta^2}{2}(HC_{\eta,\tau,\Lambda,k}+G_\tau^2)
+\eta(1+I\lambda_{\max})\beta_k
+\eta H\epsilon_{\rm act}.
\label{eq:master-recursion-proof}
\end{equation}
Iterating \eqref{eq:master-recursion-proof} gives
\begin{align*}
\Phi_K
\leq{}&(1-\eta\tau)^{K-1}\Phi_1
\\
&+\sum_{k=1}^{K-1}(1-\eta\tau)^{K-1-k}
\left[
\frac{\eta^2}{2}(HC_{\eta,\tau,\Lambda,k}+G_\tau^2)
+\eta(1+I\lambda_{\max})\beta_k
+\eta H\epsilon_{\rm act}
\right].
\end{align*}
The inequality $(1-\eta\tau)^{K-1}\leq e^{-\eta\tau(K-1)}$ proves
\eqref{eq:master-recursion}.
\end{proof}

\subsection{Conversion from the regularised saddle point}

\begin{proof}[Proof of the conversion bounds in Theorem~\ref{thm:master-recursion}]
We first bound the bias at $\pi_\tau^\star$, then control the value change
from $\pi_\tau^\star$ to $\pi^k$ using the KL part of $\Phi_k$.

\emph{The regularised multiplier $\lambda_{\tau,i}^\star$ is strictly smaller than $\lambda_{\max}$.}
Primal optimality, Slater's condition, and nonnegativity of reward and
entropy give
\begin{equation}
L_\tau(\pi_\tau^\star,\lambda_\tau^\star)
\geq L_\tau(\bar\pi,\lambda_\tau^\star)
\geq\xi\|\lambda_\tau^\star\|_1.
\label{eq:master-slater-dual-lower}
\end{equation}
Dual optimality allows comparison with the zero multiplier:
\begin{equation}
L_\tau(\pi_\tau^\star,\lambda_\tau^\star)
\leq L_\tau(\pi_\tau^\star,0)
\leq H+\tau H\log|\cA|.
\label{eq:master-zero-dual-upper}
\end{equation}
Therefore
\begin{equation}
\|\lambda_\tau^\star\|_1
\leq\frac{H(1+\tau\log|\cA|)}{\xi}
\leq\frac{H(1+\log|\cA|)}{\xi}
<\lambda_{\max}.
\label{eq:master-regularised-dual-radius}
\end{equation}
\emph{Bias at the saddle point.}
By \eqref{eq:master-regularised-dual-radius},
$\lambda_{\tau,i}^\star<\lambda_{\max}$ (i.e.,
$\lambda_{\tau,i}^\star$ is \textit{strictly} smaller than
$\lambda_{\max}$) for every $i\in[I]$, so the upper box constraint
$\lambda_i\leq\lambda_{\max}$ does not bind.
 For fixed $\pi_\tau^\star$, the
minimization over the $i$-th multiplier coordinate is
\begin{equation*}
\lambda_{\tau,i}^\star
\in
\operatorname*{arg\,min}_{0\leq\lambda_i\leq\lambda_{\max}}
\left\{
\lambda_i\bigl(V_{u_i}^{\pi_\tau^\star}-c_i\bigr)
+\frac{\tau}{2}\lambda_i^2
\right\}.
\end{equation*}
The solution is therefore
\begin{equation}
\lambda_{\tau,i}^\star
=
\min\left\{
\lambda_{\max},
\left[
\frac{c_i-V_{u_i}^{\pi_\tau^\star}}{\tau}
\right]_+
\right\}
=
\left[
\frac{c_i-V_{u_i}^{\pi_\tau^\star}}{\tau}
\right]_+,
\label{eq:master-dual-kkt}
\end{equation}
where the final equality follows from
$\lambda_{\tau,i}^\star<\lambda_{\max}$. Equivalently, if
$\lambda_{\tau,i}^\star=0$, lower-bound optimality gives
$V_{u_i}^{\pi_\tau^\star}-c_i\geq0$; otherwise,
$0<\lambda_{\tau,i}^\star<\lambda_{\max}$, and stationarity gives
$V_{u_i}^{\pi_\tau^\star}-c_i=-\tau\lambda_{\tau,i}^\star$.
Hence,
\begin{equation}
\bigl[c_i-V_{u_i}^{\pi_\tau^\star}\bigr]_+
=
\tau\lambda_{\tau,i}^\star
\leq
\frac{\tau H(1+\tau\log|\cA|)}{\xi}.
\label{eq:master-saddle-violation}
\end{equation}

For reward, primal optimality at $\lambda_\tau^\star$ gives
\begin{align*}
V_r^{\pi^\star}
\leq{}&V_r^{\pi_\tau^\star}
+\lambda_\tau^{\star\top}(V_u^{\pi_\tau^\star}-c)
+\tau\bigl(V_{\psi^{\pi_\tau^\star}}^{\pi_\tau^\star}
-V_{\psi^{\pi^\star}}^{\pi^\star}\bigr)
\\
&-\lambda_\tau^{\star\top}(V_u^{\pi^\star}-c).
\end{align*}
The last term is nonpositive because $\pi^\star$ is feasible.
The first multiplier term equals $-\tau\|\lambda_\tau^\star\|_2^2$ by
\eqref{eq:master-dual-kkt}, and policy entropy lies in
$[0,H\log|\cA|]$. Consequently,
\begin{equation}
V_r^{\pi^\star}-V_r^{\pi_\tau^\star}
\leq\tau H\log|\cA|.
\label{eq:master-saddle-reward-bias}
\end{equation}
\emph{Transfer to the current iterate.}
For any per-stage function $f_h\in[0,1]$, the performance-difference
identity uses the comparator occupancy already present in $\Phi_k$:
\[
V_f^{\pi_\tau^\star}-V_f^{\pi^k}
=\sum_{h=1}^H\mathbb E_{d_h^{\pi_\tau^\star}}
\bigl[\langle\pi_{\tau,h}^\star-\pi_h^k,Q_{f,h}^{\pi^k}\rangle\bigr].
\]
Since $0\leq Q_{f,h}^{\pi^k}\leq H$, Pinsker's inequality bounds the
absolute value of each integrand by $H\sqrt{2\KL_{k,h}}$.
Cauchy-Schwarz over the stages and their state distributions gives
\begin{equation}
|V_f^{\pi^k}-V_f^{\pi_\tau^\star}|
\leq
H^{3/2}
\left(
2\sum_{h=1}^H
\mathbb E_{d_h^{\pi_\tau^\star}}
[\KL(\pi_{\tau,h}^\star\|\pi_h^k)]
\right)^{1/2}
\leq H^{3/2}\sqrt{2\Phi_k}.
\label{eq:master-value-from-kl}
\end{equation}
Applying \eqref{eq:master-value-from-kl} to $r$ and each $u_i$, then using
\eqref{eq:master-saddle-violation} and \eqref{eq:master-saddle-reward-bias},
gives the two inequalities for every $k$. Setting $k=K$ proves
Theorem~\ref{thm:master-recursion}. Since both right-hand sides are
nonnegative, the same bounds hold for the positive parts used below.
\end{proof}

\begin{remark}[Comparison with the radius-dependent conversion]
\label{rem:fixed-domain-conversion}
The prior analysis does not exclude the possibility that a regularised
multiplier reaches the boundary of $\Lambda$. In the extreme case,
$\lambda_\tau^\star=\lambda_{\max}e_i$, so only its $i$-th coordinate reaches
$\lambda_{\max}$. The clipped optimality condition then implies only
$\bigl[c_i-V_{u_i}^{\pi_\tau^\star}\bigr]_+
\geq \tau\lambda_{\max},
$ and therefore provides no upper bound on the constraint violation.

To cover this case,
\citet{müller2024trulynoregretlearningconstrained} compare the regularised
saddle point with the feasible multiplier $\lambda=\lambda_{\max}e_i$. After
dropping the nonnegative regularisation terms at $\pi^\star$, the regularised
saddle inequality gives
\begin{equation*}
\begin{aligned}
V_r^{\pi^\star}
+\lambda_\tau^{\star\top}(V_u^{\pi^\star}-c)
\leq
V_r^{\pi_\tau^\star}
+\lambda_{\max}
\bigl(V_{u_i}^{\pi_\tau^\star}-c_i\bigr)
+\frac{\tau}{2}\lambda_{\max}^2
+\tau V_{\psi^{\pi_\tau^\star}}^{\pi_\tau^\star}.
\end{aligned}
\end{equation*}
Let $\lambda^\star$ be an optimal multiplier for the unregularised CMDP.
Optimality of $(\pi^\star,\lambda^\star)$ gives
\begin{equation*}
V_r^{\pi_\tau^\star}-V_r^{\pi^\star}
\leq
\lambda^{\star\top}
\bigl(V_u^{\pi^\star}-V_u^{\pi_\tau^\star}\bigr).
\end{equation*}
Adding these two inequalities eliminates the reward terms. Moreover,
feasibility of $\pi^\star$ and nonnegativity of $\lambda_\tau^\star$ imply
$\lambda_\tau^{\star\top}(V_u^{\pi^\star}-c)\geq0. $ It follows that
\begin{equation*}
\lambda_{\max}\bigl(c_i-V_{u_i}^{\pi_\tau^\star}\bigr)
\leq
\frac{\tau}{2}\lambda_{\max}^2
+\lambda^{\star\top}
\bigl(V_u^{\pi^\star}-V_u^{\pi_\tau^\star}\bigr)
+\tau V_{\psi^{\pi_\tau^\star}}^{\pi_\tau^\star}.
\end{equation*}
Since
$\|V_u^{\pi^\star}-V_u^{\pi_\tau^\star}\|_\infty\leq H$ and
$\|\lambda^\star\|_1\leq H/\xi$, the inner product is at most
$H^2/\xi$. Also,
$V_{\psi^{\pi_\tau^\star}}^{\pi_\tau^\star}
\leq H\log|\mathcal A|$. The resulting right-hand side is
nonnegative; hence, taking positive parts and dividing by
$\lambda_{\max}$ gives
\begin{equation*}
\bigl[c_i-V_{u_i}^{\pi_\tau^\star}\bigr]_+
\leq
\frac{\tau}{2}\lambda_{\max}
+\frac{1}{\lambda_{\max}}
\left(\frac{H^2}{\xi}+\tau H\log|\mathcal A|\right).
\end{equation*}
Adding the policy-transfer error and relaxing
$\tau\lambda_{\max}/2\leq\tau\lambda_{\max}$ gives
\begin{equation}
\bigl[c_i-V_{u_i}^{\pi^k}\bigr]_+
\leq
H^{3/2}\sqrt{2\Phi_k}
+\tau\lambda_{\max}
+\frac{1}{\lambda_{\max}}
\left(
\frac{H^2}{\xi}
+\tau H\log|\cA|
\right).
\label{eq:radius-dependent-conversion-comparison}
\end{equation}

In contrast, our analysis rules out the saturated case by proving
$\lambda_{\tau,i}^\star<\lambda_{\max}$. The coordinatewise optimality
condition \eqref{eq:master-dual-kkt} therefore gives directly
\begin{equation*}
\bigl[c_i-V_{u_i}^{\pi_\tau^\star}\bigr]_+
=
\tau\lambda_{\tau,i}^\star
\leq
\frac{\tau H(1+\tau\log|\cA|)}{\xi}.
\end{equation*}
Thus, our conversion avoids both the $\tau\lambda_{\max}$ and
$1/\lambda_{\max}$ terms. Consequently, $\lambda_{\max}$ remains independent
of $\epsilon$, and $\tau=\Theta(\epsilon)$ suffices, whereas the
radius-dependent conversion of \citet{müller2024trulynoregretlearningconstrained} requires
$\lambda_{\max}=\Theta(\epsilon^{-1})$ and $\tau=\Theta(\epsilon^2)$.
Combined with the geometric contraction in
Theorem~\ref{thm:master-recursion}, this yields the sharper accuracy
dependence detailed in Appendix~\ref{app:accuracy-improvement}.
\end{remark}

\paragraph{Conventions for the complexity proofs.}
The remaining proofs use the global notation of the main paper. Universal
numerical constants are denoted by $C$ or $c$ only locally; their values may
change between inequalities. The notation $\widetilde O$ suppresses
logarithmic factors in the horizon, representation complexities, action and
constraint counts, inverse confidence and accuracy levels, and the
parameter radii appearing in covering arguments. Polynomial dependence on
the displayed problem parameters is retained.

For the simplified rates, choose a Slater witness with $0<\xi\leq1$.
A larger certified margin can always be replaced by $\min\{\xi,1\}$;
the resulting complexity is evaluated at that chosen margin. The
accuracy-dependent comparison is informative when
$\epsilon_{\rm act}\leq\xi/H^2$: otherwise the displayed approximation
floor $H^{5/3}\xi^{-1/3}\epsilon_{\rm act}^{1/3}$ is at least $H$, which
already bounds every reward gap and constraint violation. The sample
bounds retain the outer maximum induced by the cap $\tau\leq1$; this
maximum can be removed when the cap is inactive.


\section{Proofs for Section~\ref{section:linear-cmdp-explicit-actor}}

We first verify the three oracle conditions, then substitute their bounds into the master
recursion. Fresh batches require uniform confidence only over the critic class; off-policy data
also require uniformity over the actor class. The local-norm calculation is common to
both modes and is proved once below. Throughout this section, $C,c>0$ denote universal
constants.

\subsection{Technical lemmas for uniform confidence}
\label{app:explicit-linear-confidence-lemmas}

For any bounded $V$, define
\begin{equation}
\theta_{V,h}
\triangleq
\int_{\cS}V(s')\,\mu_h(ds').
\label{eq:explicit-continuation-coefficient}
\end{equation}
Assumption~\ref{ass:linear-cmdp} gives
\begin{equation}
P_hV(s,a)=\langle\phi(s,a),\theta_{V,h}\rangle,
\qquad
\|\theta_{V,h}\|_2\leq\sqrt{d_{\rm c}}\|V\|_\infty.
\label{eq:explicit-continuation-realizability}
\end{equation}
For $j\in\{r,u_1,\ldots,u_I,\psi\}$, set
\begin{equation}
\begin{aligned}
\theta_{j,h}(V)
&\triangleq
\begin{cases}
\upsilon_{r,h}+\theta_{V,h}, & j=r,\\
\upsilon_{u_i,h}+\theta_{V,h}, & j=u_i,\ i\in[I],\\
\theta_{V,h}, & j=\psi,
\end{cases}
\\
R_{j,h}
&\triangleq
\begin{cases}
H+1-h, & j\in\{r,u_1,\ldots,u_I\},\\
(H-h)\log|\cA|, & j=\psi.
\end{cases}
\end{aligned}
\label{eq:explicit-common-component-parameters}
\end{equation}
For the backward recursion, write
$\theta_{j,h}^k\triangleq\theta_{j,h}(V_{j,h+1}^k)$.
Given a policy $\pi$, define
$\ell_{j,h}^{\pi}(s,a)=0$ for $j\in\{r,u_1,\ldots,u_I\}$ and
$\ell_{\psi,h}^{\pi}(s,a)=-\log\pi_h(a\mid s)$. For $B>0$, let
\begin{equation}
\begin{aligned}
v_{j,h}^{\pi,w,M}(s)
&\triangleq
\mathbb E_{a\sim\pi_h(\cdot\mid s)}
\Bigl[
\ell_{j,h}^{\pi}(s,a)
\\[-2pt]
&\hspace{20mm}{}
+\mathrm{Truncate}_{[0,R_{j,h}]}
\left(
\langle\phi(s,a),w\rangle
+\alpha_j\sqrt{\phi(s,a)^\top M\phi(s,a)}
\right)
\Bigr],
\\
\mathcal V_{j,h}(\pi,B)
&\triangleq
\left\{
v_{j,h}^{\pi,w,M}\;\middle|\;
\|w\|_2\leq B,\ 0\preceq M\preceq I_{d_{\rm c}}
\right\}.
\end{aligned}
\label{eq:explicit-fixed-policy-value-class}
\end{equation}
Let $\mathcal V_h(\pi,B)\triangleq
\bigcup_{j\in\{r,u_1,\ldots,u_I,\psi\}}\mathcal V_{j,h}(\pi,B)$, and at the terminal
stage set $\mathcal V_{j,H+1}(\pi,B)=\mathcal V_{H+1}(\pi,B)=\{0\}$.

\begin{lemma}[Covering the linear OPE value classes]
\label{lem:explicit-linear-value-cover}
Let $\bar\alpha=\max\{\alpha_r,\alpha_u,\alpha_\psi\}$. For every fixed full-support policy
$\pi$, $B>0$, and $\varepsilon\in(0,1]$,
\begin{equation}
\log\mathcal N_\infty(\varepsilon,\mathcal V_h(\pi,B))
\leq
C(d_{\rm c}^2+d_{\rm c})
\log\left(\frac{C(1+B+\bar\alpha+d_{\rm c}+I)}{\varepsilon}\right).
\label{eq:explicit-fixed-policy-value-cover}
\end{equation}
If a class $\mathcal F_h$ induces policies $\pi(\omega)$ and corresponding class elements satisfy
$\|v_{j,h}^{\pi(\omega),w,M}-v_{j,h}^{\pi(\omega'),w,M}\|_\infty
\leq L_h\|f_{\omega,h}-f_{\omega',h}\|_\infty$, then an
$\varepsilon/(2L_h)$-net of $\mathcal F_h$, together with $\varepsilon/2$-nets of the
corresponding fixed-policy classes, is an $\varepsilon$-net of their union.
\end{lemma}

\begin{proof}
Truncation and policy averaging are nonexpansive. If
$\|w-w'\|_2\leq\varepsilon/2$, then
$|\langle\phi,w-w'\rangle|\leq\varepsilon/2$. Moreover,
\begin{equation*}
\left|
\sqrt{\phi^\top M\phi}-\sqrt{\phi^\top M'\phi}
\right|
\leq
\sqrt{\|M-M'\|_{\rm op}}
\leq
\sqrt{\|M-M'\|_{\rm F}}.
\end{equation*}
Hence an $\varepsilon/2$-cover of the radius-$B$ Euclidean ball and an
$\varepsilon^2/(4\bar\alpha^2)$-cover of the positive-semidefinite contractions in Frobenius norm
give an $\varepsilon$-cover of each component class. These parameter sets have dimensions
$d_{\rm c}$ and $d_{\rm c}(d_{\rm c}+1)/2$ and radii $B$ and $\sqrt{d_{\rm c}}$, respectively.
Their volumetric covering bounds, followed by a union over the $I+2$ components, give
\begin{equation*}
\log\mathcal N_\infty(\varepsilon,\mathcal V_h(\pi,B))
\leq
d_{\rm c}\log\left(1+\frac{4B}{\varepsilon}\right)
+\frac{d_{\rm c}(d_{\rm c}+1)}2
\log\left(1+\frac{8\sqrt{d_{\rm c}}\,\bar\alpha^2}{\varepsilon^2}\right)
+\log(I+2),
\end{equation*}
which implies \eqref{eq:explicit-fixed-policy-value-cover}.

For the second claim, take an $\varepsilon/(2L_h)$-net of $\mathcal F_h$. At every actor in this
net, take an $\varepsilon/2$-net of the corresponding fixed-policy value class. The assumed
transfer inequality and the triangle inequality show that the resulting product net has radius
$\varepsilon$, which proves the second assertion.
\end{proof}

For an adapted stage-$h$ sample prefix and any value function $V$, write
\begin{equation*}
\begin{aligned}
\Sigma_{h,m}
&=I_{d_{\rm c}}+\sum_{\ell=1}^m\phi_h^\ell(\phi_h^\ell)^\top,
\\
\widehat\theta_{j,h,m}(V)
&=\Sigma_{h,m}^{-1}\sum_{\ell=1}^m
\phi_h^\ell\bigl(y_{j,h}^{\ell}+V(s_{h+1}^{\ell})\bigr),
\\
\zeta_{j,h}^{\ell}(V)
&=y_{j,h}^{\ell}+V(s_{h+1}^{\ell})
-\langle\phi_h^\ell,\theta_{j,h}(V)\rangle.
\end{aligned}
\end{equation*}
Here $\mathcal G_{\ell,h}$ denotes the information available through
$(s_h^\ell,a_h^\ell)$, before observing $(y_{j,h}^{\ell},s_{h+1}^{\ell})$.

\begin{lemma}[Uniform confidence over a value function class]
\label{lem:explicit-uniform-linear-confidence}
Fix $(j,h)$ and a class $\mathcal W_{h+1}$ that is deterministic, or fixed before
a fresh batch is collected. Suppose that, uniformly over $V\in\mathcal W_{h+1}$ and $\ell$,
$\mathbb E[\zeta_{j,h}^{\ell}(V)\mid\mathcal G_{\ell,h}]=0$,
$|y_{j,h}^{\ell}+V(s_{h+1}^{\ell})|\leq Y_j$, and
$\|\theta_{j,h}(V)\|_2\leq S_j$. If
$\log\mathcal N_\infty(\varepsilon_0,\mathcal W_{h+1})\leq\Gamma$ and
$(\sqrt{m_{\max}}+\sqrt{d_{\rm c}})\varepsilon_0\leq1$, then, with probability at least
$1-\delta_0$, uniformly over $m\leq m_{\max}$, $V\in\mathcal W_{h+1}$, and $(s,a)$,
\begin{equation}
\left|
\langle\phi(s,a),\widehat\theta_{j,h,m}(V)-\theta_{j,h}(V)\rangle
\right|
\leq
\alpha_j\|\phi(s,a)\|_{\Sigma_{h,m}^{-1}},
\label{eq:explicit-generic-uniform-confidence}
\end{equation}
for
$\alpha_j=C\{%
Y_j\sqrt{d_{\rm c}\log(1+m_{\max}/d_{\rm c})+\Gamma
+\log(m_{\max}/\delta_0)}
+S_j+1\}$.
\end{lemma}

\begin{proof}
Let $\mathcal C$ be an $\varepsilon_0$-net of $\mathcal W_{h+1}$. For every
$V\in\mathcal C$, the ridge normal equations give
\begin{equation}
\widehat\theta_{j,h,m}(V)-\theta_{j,h}(V)
=
\Sigma_{h,m}^{-1}
\left(
\sum_{\ell=1}^m\phi_h^\ell\zeta_{j,h}^{\ell}(V)
-\theta_{j,h}(V)
\right).
\label{eq:explicit-generic-regression-decomposition}
\end{equation}
Self-normalized concentration and a union bound over $\mathcal C$ and $m\leq m_{\max}$ yield
\begin{equation*}
\left\|
\sum_{\ell=1}^m\phi_h^\ell\zeta_{j,h}^{\ell}(V)
\right\|_{\Sigma_{h,m}^{-1}}
\leq
C Y_j\sqrt{
d_{\rm c}\log(1+m_{\max}/d_{\rm c})
+\Gamma+\log(m_{\max}/\delta_0)}
\end{equation*}
simultaneously for all net elements and prefixes. The ridge term is at most
$\|\theta_{j,h}(V)\|_2$. For any $V'$ with
$\|V-V'\|_\infty\leq\varepsilon_0$,
\begin{equation*}
\left\|
\sum_{\ell=1}^m\phi_h^\ell(V-V')(s_{h+1}^\ell)
\right\|_{\Sigma_{h,m}^{-1}}
\leq\sqrt m\,\varepsilon_0,
\qquad
\|\theta_{V,h}-\theta_{V',h}\|_2
\leq\sqrt{d_{\rm c}}\,\varepsilon_0.
\end{equation*}
These terms are absorbed by the displayed choice of $\alpha_j$. Combining the preceding bounds
with \eqref{eq:explicit-generic-regression-decomposition} and Cauchy-Schwarz proves
\eqref{eq:explicit-generic-uniform-confidence}. If $\mathcal W_{h+1}$ is conditionally fixed,
the argument is applied after conditioning on the preceding history.
\end{proof}

\subsection{Analysis of the on-policy oracle}
\label{app:explicit-on-policy-oracle}

\begin{proof}[Proof of Lemma~\ref{lem:explicit-linear-ope-certificates}, on-policy branch]
Condition on the history $\mathcal F_{k-1}$ before iteration $k$. Then $\pi^k$ and
$\lambda_k$ are fixed, and the $N$ trajectories collected at iteration $k$ are conditionally
independent. Define
\begin{equation}
\iota_{\rm on}
\triangleq
\log\left(
\frac{C K N H(I+1)|\cA|d_{\rm c}(1+\lambda_{\max})}
{\delta\min\{1,\tau\}}
\right).
\label{eq:explicit-on-policy-log-factor}
\end{equation}
Choose, for a sufficiently large universal constant $C$,
\begin{equation}
\alpha_r=\alpha_u=C H\sqrt{(d_{\rm c}^2+d_{\rm c})\iota_{\rm on}},
\qquad
\alpha_\psi=C H\log(e|\cA|)\sqrt{(d_{\rm c}^2+d_{\rm c})\iota_{\rm on}}.
\label{eq:explicit-on-policy-confidence-radii}
\end{equation}
We prove the oracle properties on a common event of probability at least $1-\delta$.

\paragraph{Value class and uniform confidence.}
The reward and utility regression targets are bounded by $H$, while the entropy value
target is bounded by $H\log(e|\cA|)$. Since $\|\phi\|_2\leq1$ and
$0\preceq(\Sigma_h^k)^{-1}\preceq I_{d_{\rm c}}$, the normal equations give
\begin{equation}
\|\widehat\theta_{j,h}^k\|_2
\leq
\sum_{n=1}^N
\|\phi_h^{k,n}\|_2
\left|y_{j,h}^{k,n}+V_{j,h+1}^k(s_{h+1}^{k,n})\right|
\leq
NH\log(e|\cA|)
\triangleq B_N.
\label{eq:explicit-regression-vector-radius}
\end{equation}
Define
\begin{equation}
\mathcal V_{k,j,h}^{\rm on}
\triangleq
\mathcal V_{j,h}(\pi^k,B_N),
\qquad
\mathcal V_{k,h}^{\rm on}
\triangleq
\mathcal V_h(\pi^k,B_N).
\label{eq:explicit-on-policy-value-class}
\end{equation}
The backward recursion satisfies
\begin{equation}
V_{j,h}^k
=
v_{j,h}^{\pi^k,\widehat\theta_{j,h}^k,(\Sigma_h^k)^{-1}}
\in
\mathcal V_{k,j,h}^{\rm on}
\subseteq\mathcal V_{k,h}^{\rm on}.
\label{eq:explicit-on-policy-value-membership}
\end{equation}
Conditionally on $\mathcal F_{k-1}$, this class is fixed. Applying
Lemma~\ref{lem:explicit-linear-value-cover} with $B=B_N$ gives
\begin{equation}
\log\mathcal N_\infty(\varepsilon,\mathcal V_{k,h}^{\rm on})
\leq
C(d_{\rm c}^2+d_{\rm c})
\left[\iota_{\rm on}+\log(1+1/\varepsilon)\right].
\label{eq:explicit-on-policy-value-cover}
\end{equation}
Use the bounds
$Y_r=Y_u=H$, $Y_\psi=H\log(e|\cA|)$,
$S_r,S_u\leq2H\sqrt{d_{\rm c}}$, and
$S_\psi\leq H\sqrt{d_{\rm c}}\log(e|\cA|)$.
Applying Lemma~\ref{lem:explicit-uniform-linear-confidence} conditionally with
$\mathcal W_{h+1}=\mathcal V_{k,j,h+1}^{\rm on}$, $m_{\max}=N$,
$\varepsilon_0=(2KNHd_{\rm c})^{-2}$, and
$\delta_0=\delta/[2KH(I+2)]$, followed by a union bound over $(k,h,j)$, gives, with probability
at least $1-\delta/2$,
\begin{equation}
\left|
\left\langle
\phi(s,a),\widehat\theta_{j,h}^k-\theta_{j,h}^k
\right\rangle
\right|
\leq
\alpha_j\|\phi(s,a)\|_{(\Sigma_h^k)^{-1}}
=b_{j,h}^k(s,a)
\label{eq:explicit-on-policy-uniform-confidence}
\end{equation}
simultaneously for all $(k,h,j,s,a)$. The resulting radii are precisely
\eqref{eq:explicit-on-policy-confidence-radii}. No actor cover is required because $\pi^k$ is fixed
before the fresh batch is collected.

\paragraph{Optimism and consistency.}
The population targets defined by \eqref{eq:explicit-common-component-parameters} lie in their
corresponding truncation intervals. Therefore,
\eqref{eq:explicit-on-policy-uniform-confidence} and
\eqref{eq:explicit-linear-ope-components} imply
\begin{equation}
\begin{aligned}
Q_{r,h}^k&\geq r_h+P_hV_{r,h+1}^k,
&
Q_{u_i,h}^k&\geq u_{i,h}+P_hV_{u_i,h+1}^k,
&
Q_{\psi,h}^k&\geq\psi_{k,h}+P_hV_{\psi,h+1}^k.
\end{aligned}
\label{eq:explicit-on-policy-component-optimism}
\end{equation}
Since $\lambda_k\geq0$ and $\tau>0$, summing the component inequalities gives
\begin{equation}
Q_{z_k,h}^k
\geq
\Bell_{z_k,h}^{\pi^k}Q_{z_k,h+1}^k,
\qquad h\in[H].
\label{eq:explicit-on-policy-shaped-optimism}
\end{equation}

For $j=r,u_i,\psi$, let
$(Q_{j,h}^{\pi^k},V_{j,h}^{\pi^k},V_j^{\pi^k})$ denote the corresponding action value,
stagewise value, and total value; for $j=\psi$, this means
$(Q_{\psi_k,h}^{\pi^k},V_{\psi_k,h}^{\pi^k},V_{\psi_k}^{\pi^k})$. Also set
$V_j^k\triangleq V_{j,1}^k(s_1)$. Backward induction in
\eqref{eq:explicit-on-policy-component-optimism} gives
$Q_{j,h}^k\geq Q_{j,h}^{\pi^k}$. The upper side of the confidence interval gives
\begin{equation}
0\leq
Q_{j,h}^k(s,a)-Q_{j,h}^{\pi^k}(s,a)
\leq
P_h\bigl(V_{j,h+1}^k-V_{j,h+1}^{\pi^k}\bigr)(s,a)+2b_{j,h}^k(s,a).
\label{eq:explicit-on-policy-component-error}
\end{equation}
The known immediate entropy term cancels from the difference when $j=\psi$. Taking the
$\pi^k$-expectation of \eqref{eq:explicit-on-policy-component-error} and telescoping over $h$
gives
\begin{equation}
0\leq V_j^k-V_j^{\pi^k}
\leq
2\sum_{h=1}^H
\mathbb E_{(s,a)\sim d_h^{\pi^k}}
[b_{j,h}^k(s,a)].
\label{eq:explicit-on-policy-component-consistency}
\end{equation}

To bound the expected bonuses, define the population feature covariance
\begin{equation}
\Gamma_{k,h}
\triangleq
\mathbb E_{(s,a)\sim d_h^{\pi^k}}
[\phi(s,a)\phi(s,a)^\top].
\label{eq:explicit-on-policy-population-covariance}
\end{equation}
and write $\phi_h^{k,n}=\phi(s_h^{k,n},a_h^{k,n})$. We now derive the required empirical
covariance bound. Set
\begin{equation*}
a\triangleq\frac{C\iota_{\rm on}}N,
\qquad
A_{k,h}\triangleq\Gamma_{k,h}+aI_{d_{\rm c}},
\end{equation*}
and define the normalized population and empirical covariances
\begin{equation*}
G_{k,h}
\triangleq
A_{k,h}^{-1/2}\Gamma_{k,h}A_{k,h}^{-1/2},
\qquad
\widehat G_{k,h}
\triangleq
\frac1N\sum_{n=1}^N
A_{k,h}^{-1/2}\phi_h^{k,n}(\phi_h^{k,n})^\top A_{k,h}^{-1/2}.
\end{equation*}
Because $A_{k,h}\succeq aI_{d_{\rm c}}$ and $\|\phi_h^{k,n}\|_2\leq1$, matrix Bernstein gives
\begin{equation}
\mathbb P\left(
\|\widehat G_{k,h}-G_{k,h}\|_{\rm op}>\frac12
\,\middle|\,
\mathcal F_{k-1}
\right)
\leq
2d_{\rm c}e^{-cNa}
\leq
\frac{\delta}{2KH},
\label{eq:explicit-on-policy-whitened-concentration}
\end{equation}
where the final inequality follows by choosing the universal constant in $a$ sufficiently large.
On the complementary event,
$G_{k,h}\preceq\widehat G_{k,h}+\frac12I_{d_{\rm c}}$. Multiplying this inequality on both
sides by $A_{k,h}^{1/2}$ gives
\begin{equation*}
\Gamma_{k,h}
\preceq
\frac1N\sum_{n=1}^N\phi_h^{k,n}(\phi_h^{k,n})^\top
+\frac12\bigl(\Gamma_{k,h}+aI_{d_{\rm c}}\bigr).
\end{equation*}
Moving $\Gamma_{k,h}/2$ to the left and substituting $Na=C\iota_{\rm on}$ yields
\begin{equation}
N\Gamma_{k,h}
\preceq
2\sum_{n=1}^N\phi_h^{k,n}(\phi_h^{k,n})^\top
+C\iota_{\rm on}I_{d_{\rm c}}.
\label{eq:explicit-on-policy-covariance-comparison}
\end{equation}
A union bound over $(k,h)$ and the tower property show that
\eqref{eq:explicit-on-policy-covariance-comparison} holds jointly with
\eqref{eq:explicit-on-policy-uniform-confidence} with probability at least $1-\delta$.
Since
$\Sigma_h^k=I_{d_{\rm c}}+\sum_{n=1}^N\phi_h^{k,n}(\phi_h^{k,n})^\top$ and
$0\prec(\Sigma_h^k)^{-1}\preceq I_{d_{\rm c}}$, multiplying
\eqref{eq:explicit-on-policy-covariance-comparison} by $(\Sigma_h^k)^{-1}$ and taking the trace
gives
\begin{align}
\operatorname{tr}((\Sigma_h^k)^{-1}\Gamma_{k,h})
&\leq
\frac2N\operatorname{tr}\left((\Sigma_h^k)^{-1}(\Sigma_h^k-I_{d_{\rm c}})\right)
+\frac{C\iota_{\rm on}}N\operatorname{tr}((\Sigma_h^k)^{-1})
\nonumber\\
&\leq
C\frac{d_{\rm c}\iota_{\rm on}}N.
\label{eq:explicit-on-policy-trace-width}
\end{align}
Moreover,
\begin{equation*}
\mathbb E_{d_h^{\pi^k}}
[\|\phi(s,a)\|_{(\Sigma_h^k)^{-1}}^2]
=
\operatorname{tr}((\Sigma_h^k)^{-1}\Gamma_{k,h}).
\end{equation*}
Jensen's inequality and \eqref{eq:explicit-on-policy-trace-width} therefore yield
\begin{align}
\mathbb E_{d_h^{\pi^k}}
[\|\phi(s,a)\|_{(\Sigma_h^k)^{-1}}]
&\leq
\sqrt{\operatorname{tr}((\Sigma_h^k)^{-1}\Gamma_{k,h})}
\nonumber\\
&\leq
C\sqrt{\frac{d_{\rm c}\iota_{\rm on}}{N}}.
\label{eq:explicit-on-policy-population-width}
\end{align}
By \eqref{eq:explicit-on-policy-component-consistency},
\begin{align*}
0\leq V_{z_k}^k-V_{z_k}^{\pi^k}
&\leq
2\bigl(\alpha_r+I\lambda_{\max}\alpha_u+\tau\alpha_\psi\bigr)
\sum_{h=1}^H
\mathbb E_{d_h^{\pi^k}}
[\|\phi(s,a)\|_{(\Sigma_h^k)^{-1}}],
\\
\max_{i\in[I]}|V_{u_i}^k-V_{u_i}^{\pi^k}|
&\leq
2\alpha_u
\sum_{h=1}^H
\mathbb E_{d_h^{\pi^k}}
[\|\phi(s,a)\|_{(\Sigma_h^k)^{-1}}].
\end{align*}
Substituting \eqref{eq:explicit-on-policy-confidence-radii} and
\eqref{eq:explicit-on-policy-population-width}, the common choice
\begin{equation}
\beta_k
\leq
C\bigl(1+I\lambda_{\max}+\tau\log(e|\cA|)\bigr)H^2
\sqrt{\frac{d_{\rm c}^3}{N}}\,\iota_{\rm on}.
\label{eq:explicit-on-policy-beta-full}
\end{equation}
controls both quantities required by Assumption~\ref{ass:master-ope}(ii). The utility truncation
also gives $0\leq V_{u_i}^k\leq H$.

\paragraph{Local norm.}
The truncation levels give, for $p_a=\pi_h^k(a\mid s)$,
\begin{equation}
0\leq Q_{z_k,h}^k(s,a)
\leq
H(1+I\lambda_{\max}+\tau\log|\cA|)+\tau\log(1/p_a).
\label{eq:explicit-on-policy-shaped-envelope}
\end{equation}
For this calculation only, set $B=H(1+I\lambda_{\max}+\tau\log|\cA|)$ and
$\alpha=\eta\tau$. Equation~\eqref{eq:explicit-on-policy-shaped-envelope} gives
$e^{\eta Q_{z_k,h}^k(s,a)}\leq e^{\eta B}p_a^{-\alpha}$ and
$(Q_{z_k,h}^k(s,a))^2\leq2B^2+2\tau^2\log^2(1/p_a)$. Hence
\begin{align}
\sum_a p_a e^{\eta Q_{z_k,h}^k(s,a)}(Q_{z_k,h}^k(s,a))^2
&\leq
2e^{\eta B}
\left[
B^2\sum_a p_a^{1-\alpha}
+\tau^2\sum_a p_a^{1-\alpha}\log^2(1/p_a)
\right]
\nonumber\\
&\leq
C\left[B^2+\tau^2\log^2(e|\cA|)\right].
\label{eq:explicit-on-policy-escort-moment}
\end{align}
The first sum satisfies $\sum_a p_a^{1-\alpha}\leq|\cA|^\alpha$ by concavity. If
$|\cA|=1$, the second sum is zero. Otherwise, set $t=1/\log(e|\cA|)$. The inequality
$x^2\leq4e^{-2}t^{-2}e^{tx}$ for $x\geq0$ gives
\begin{equation}
\sum_a p_a^{1-\alpha}\log^2(1/p_a)
\leq
C\log^2(e|\cA|)\sum_a p_a^{1-\alpha-t}
\leq
C\log^2(e|\cA|)|\cA|^{\alpha+t}.
\label{eq:explicit-on-policy-escort-log-moment}
\end{equation}
The stepsize conditions imply $\eta B\leq1/4$ and
$\alpha\log|\cA|\leq1/4$, while $|\cA|^t\leq e$. Substituting these three bounds proves the
second line of \eqref{eq:explicit-on-policy-escort-moment}. Therefore,
\begin{equation}
C_{\eta,\tau,\Lambda,k}
\leq
C H^2\bigl(1+I\lambda_{\max}+\tau\log(e|\cA|)\bigr)^2.
\label{eq:explicit-on-policy-local-norm-full}
\end{equation}
Since $0<\tau\leq 1$,
\begin{equation*}
1+I\lambda_{\max}+\tau\log(e|\cA|)
\leq
2\bigl(1+I\lambda_{\max}+\tau\log|\cA|\bigr),
\end{equation*}
so \,\eqref{eq:explicit-on-policy-local-norm-full} has the form stated in the
lemma.
Equations \eqref{eq:explicit-on-policy-shaped-optimism},
\eqref{eq:explicit-on-policy-beta-full}, and
\eqref{eq:explicit-on-policy-local-norm-full} prove
Assumption~\ref{ass:master-ope}; specifically, they establish the on-policy
consistency bound in \eqref{eq:explicit-linear-on-policy-ope-rate} and the
common local-norm bound in Lemma~\ref{lem:explicit-linear-ope-certificates}.
\end{proof}

\subsection{Proof of the on-policy final rate}
\label{app:explicit-on-policy-final-rate}

Fix $\epsilon,\delta\in(0,1)$ and choose the parameters at the orders
\begin{equation}
\begin{aligned}
\lambda_{\max}
&=\frac{2H\log(e|\cA|)}{\xi},
\quad
\tau
=\Theta\left(
\min\left\{1,
\max\left\{
\frac{\xi\epsilon}{H\log(e|\cA|)},
\frac{H^{2/3}\xi^{2/3}\epsilon_{\rm act}^{1/3}}
{\log^{2/3}(e|\cA|)}
\right\}
\right\}
\right),
\\
\eta
&=\widetilde\Theta\left(\frac{\tau\xi^2\epsilon^2}{I^2H^8}\right),
\qquad
K=\widetilde\Theta\left(\frac{I^2H^8}{\tau^2\xi^2\epsilon^2}\right),
\qquad
N=\widetilde\Theta\left(\frac{d_{\rm c}^3I^4H^{14}}
{\tau^2\xi^4\epsilon^4}\right).
\end{aligned}
\label{eq:explicit-on-policy-final-parameters}
\end{equation}

\begin{proof}[Proof of Theorem~\ref{thm:explicit-linear-last-iterate}, on-policy branch]
We first choose the iteration and batch budgets to control the potential, and then tune $\tau$
to convert that control into a guarantee for the original CMDP. We use the normalized Slater
witness specified in the complexity-proof conventions above.
For the logarithmic factors only, write $L=\log(e|\cA|)$ and take
$L_\epsilon=\log(e+H^4L/\epsilon^2+H^5L^2/(\xi^2\epsilon^2))$.

\paragraph{Potential control.}
Uniform policy initialization and $\lambda_1=0$ give
\begin{equation}
\Phi_1
\leq
H\log|\cA|+\frac{H^2L^2}{2\xi^2}.
\label{eq:explicit-on-policy-initial-potential}
\end{equation}
Thus $L_\epsilon$ bounds $\log(e+H^3\Phi_1/\epsilon^2)$ using only known problem parameters.
The geometric sum in Theorem~\ref{thm:master-recursion} is at most $1/(\eta\tau)$.
Thus the uniform on-policy certificate gives
\begin{equation}
\begin{aligned}
\Phi_K
\leq {}&
e^{-\eta\tau(K-1)}\Phi_1
+\frac{\eta}{2\tau}
\sup_{k<K}\bigl(HC_{\eta,\tau,\Lambda,k}+G_\tau^2\bigr)
\\
&+\frac{1+I\lambda_{\max}}{\tau}\sup_{k<K}\beta_k
+\frac{H\epsilon_{\rm act}}{\tau}.
\end{aligned}
\label{eq:explicit-on-policy-unrolled-potential}
\end{equation}
Using $\lambda_{\max}=2HL/\xi$, the oracle bounds reduce to
\begin{equation}
\begin{aligned}
\sup_{k<K}\bigl(HC_{\eta,\tau,\Lambda,k}+G_\tau^2\bigr)
&\leq C\frac{I^2H^5L^2}{\xi^2},
\\
(1+I\lambda_{\max})\sup_{k<K}\beta_k
&\leq C\frac{I^2H^4L^2}{\xi^2}
\sqrt{\frac{d_{\rm c}^3}{N}}\,\iota_{\rm on},
\end{aligned}
\label{eq:explicit-on-policy-rate-inputs}
\end{equation}
where $\iota_{\rm on}$ is defined in \eqref{eq:explicit-on-policy-log-factor}.
Choose the stepsize at the following upper-bound scale, and take the smallest integer batch
and iteration budgets satisfying
\begin{equation}
\begin{aligned}
\eta
&\leq \frac{c\tau\xi^2\epsilon^2}{I^2H^8L^2},
\qquad
K=1+\left\lceil\frac{C L_\epsilon}{\eta\tau}\right\rceil,
\\
N
&\geq
\frac{C d_{\rm c}^3I^4H^{14}L^4}
{\tau^2\xi^4\epsilon^4}\,\iota_{\rm on}^2.
\end{aligned}
\label{eq:explicit-on-policy-proof-parameters}
\end{equation}
The dependence of $\iota_{\rm on}$ on $N,K$ is logarithmic, so these choices can be satisfied
without changing their displayed polynomial orders. Taking $c$ sufficiently small also
ensures the two oracle stepsize conditions. Each of the first three terms in
\eqref{eq:explicit-on-policy-unrolled-potential} is then $O(\epsilon^2/H^3)$, giving
\begin{equation}
\Phi_K
\leq
C\frac{\epsilon^2}{H^3}
+\frac{H\epsilon_{\rm act}}{\tau}.
\label{eq:explicit-on-policy-final-potential}
\end{equation}
In particular, the iteration and batch budgets are
\begin{equation}
K
=O\left(\frac{I^2H^8L^2L_\epsilon}
{\tau^2\xi^2\epsilon^2}\right),
\qquad
N
=O\left(\frac{d_{\rm c}^3I^4H^{14}L^4\iota_{\rm on}^2}
{\tau^2\xi^4\epsilon^4}\right),
\label{eq:explicit-on-policy-parameter-substitution}
\end{equation}
which verifies the parameter orders in \eqref{eq:explicit-on-policy-final-parameters}.

\paragraph{Temperature and actor approximation.}
The conversion bounds in Theorem~\ref{thm:master-recursion} and
\eqref{eq:explicit-on-policy-final-potential} imply
\begin{equation}
\max\left\{
[V_r^{\pi^\star}-V_r^{\pi^K}]_+,\,
\max_{i\in[I]}[c_i-V_{u_i}^{\pi^K}]_+
\right\}
\leq
C\left[
\epsilon+\frac{HL}{\xi}\tau
+H^2\sqrt{\frac{\epsilon_{\rm act}}{\tau}}
\right].
\label{eq:explicit-on-policy-temperature-envelope}
\end{equation}
The last two terms express the only temperature tradeoff: a larger $\tau$ increases the
regularisation error but decreases the effect of actor approximation.
They balance at $\tau=(H^2\xi^2\epsilon_{\rm act}/L^2)^{1/3}$.
The target-accuracy scale is $\xi\epsilon/(HL)$.
Taking the larger of these two scales, capped at one, gives
\eqref{eq:explicit-on-policy-final-parameters} and the error bound
\eqref{eq:explicit-linear-final-error}. Indeed, below the cap, the regularisation and actor
terms are both bounded by
$O(\epsilon+H^{5/3}L^{1/3}\epsilon_{\rm act}^{1/3}/\xi^{1/3})$.
If the cap is active, the actor-error term in this bound is already at least a constant
multiple of $H$, and the same conclusion follows from the range $[0,H]$ of reward and utility
values. A genuine $\widetilde O(\epsilon)$ guarantee requires
$\epsilon_{\rm act}=\widetilde O(\xi\epsilon^3/H^5)$.

\paragraph{Sample count.}
Multiplying the two budgets in \eqref{eq:explicit-on-policy-parameter-substitution} gives
\begin{equation}
NK
=
O\left(
\frac{d_{\rm c}^3I^6H^{22}L^6}
{\xi^6\epsilon^6\tau^4}\,
\iota_{\rm on}^2L_\epsilon
\right).
\label{eq:explicit-on-policy-samples-before-temperature}
\end{equation}
The capped temperature satisfies
\begin{equation}
\frac1{\tau^4}
=
O\left(
\max\left\{1,\,
\min\left\{
\frac{H^4L^4}{\xi^4\epsilon^4},\,
\frac{L^{8/3}}{H^{8/3}\xi^{8/3}\epsilon_{\rm act}^{4/3}}
\right\}\right\}
\right).
\label{eq:explicit-on-policy-inverse-temperature}
\end{equation}
For $\epsilon_{\rm act}=0$, the second entry of the minimum is interpreted
as $+\infty$. Substitution into
\eqref{eq:explicit-on-policy-samples-before-temperature} gives
\eqref{eq:explicit-on-policy-sample-complexity} after suppressing
polylogarithmic factors. In the accuracy-sensitive regime
$\epsilon_{\rm act}\leq\xi/H^2$, the cap is inactive and the outer maximum
can be removed. Each episode contains $H$ transitions, so the transition
count is $HNK$.
\end{proof}

\subsection{Analysis of the off-policy oracle}
\label{app:explicit-off-policy-oracle}

\begin{proof}[Proof of Lemma~\ref{lem:explicit-linear-ope-certificates}, off-policy branch]
In the off-policy mode, $\cD^k$ contains the entire data prefix, including samples that
influenced the current actor. We therefore first construct a confidence event uniform over both
actors and critics, and then bound the cumulative evaluation error through predictable elliptical
widths. Set
\begin{equation}
\iota_K
\triangleq
\log\left(
\frac{C K H(I+1)|\cA|(d_{\rm c}+d_{\rm a})(1+\lambda_{\max})
(1+\kappa_G+\epsilon_{\rm bias})}
{\delta\min\{1,\tau\}}
\right),
\label{eq:explicit-oracle-log-factor}
\end{equation}
and take
\begin{equation}
\begin{aligned}
\alpha_r=\alpha_u
&=C H\sqrt{(d_{\rm c}^2+d_{\rm c}+d_{\rm a})\iota_K},
\\
\alpha_\psi
&=C H\log(e|\cA|)
\sqrt{(d_{\rm c}^2+d_{\rm c}+d_{\rm a})\iota_K}.
\end{aligned}
\label{eq:explicit-confidence-radii}
\end{equation}

\paragraph{Uniform confidence.}
For every $k$, the off-policy dataset satisfies $|\cD_h^k|\leq K$. The normal-equation
calculation in \eqref{eq:explicit-regression-vector-radius} therefore gives
\begin{equation}
\|\widehat\theta_{j,h}^k\|_2
\leq KH\log(e|\cA|)
\triangleq B_K.
\label{eq:explicit-off-policy-regression-vector-radius}
\end{equation}
For a fixed actor, the critic still varies only through a regression coefficient
$w\in\mathbb R^{d_{\rm c}}$ and a symmetric matrix
$0\preceq M\preceq I_{d_{\rm c}}$, as in
\eqref{eq:explicit-fixed-policy-value-class}. The additional step is to cover the actors that can be
computed from $\cD^k$.

Define the actor-fitting residual
\begin{equation*}
\Delta_{k,h}(s,a)
\triangleq
\langle\varphi(s,a),\omega_h^{k+1}-\omega_h^k\rangle
-\eta Q_{z_k,h}^k(s,a).
\end{equation*}
Equation~\eqref{eq:actor-residual-action-difference} gives
$\sup_{s,a,a'}|\Delta_{k,h}(s,a)-\Delta_{k,h}(s,a')|
\leq\eta\epsilon_{\rm act}$. By the component truncation bounds,
\begin{equation*}
Q_{z_k,h}^k(s,a)
=\overline Q_{k,h}(s,a)-\tau\log\pi_h^k(a\mid s),
\qquad
0\leq\overline Q_{k,h}(s,a)\leq
H(1+I\lambda_{\max}+\tau\log|\cA|).
\end{equation*}
Subtracting the fitted-logit identities for actions $a$ and $a'$ therefore gives
\begin{equation*}
\begin{aligned}
\log\frac{\pi_h^{k+1}(a\mid s)}{\pi_h^{k+1}(a'\mid s)}
={}&
(1-\eta\tau)
\log\frac{\pi_h^k(a\mid s)}{\pi_h^k(a'\mid s)}
\\
&+\eta\bigl(\overline Q_{k,h}(s,a)-\overline Q_{k,h}(s,a')\bigr)
+\Delta_{k,h}(s,a)-\Delta_{k,h}(s,a').
\end{aligned}
\end{equation*}
Taking absolute values yields
\begin{equation}
\begin{aligned}
\left|\log\frac{\pi_h^{k+1}(a\mid s)}{\pi_h^{k+1}(a'\mid s)}\right|
\leq {}&
(1-\eta\tau)
\left|\log\frac{\pi_h^k(a\mid s)}{\pi_h^k(a'\mid s)}\right|
\\
&+\eta\left[
H(1+I\lambda_{\max}+\tau\log|\cA|)
+\epsilon_{\rm act}
\right].
\end{aligned}
\label{eq:explicit-log-odds-radius}
\end{equation}
Since $\omega_h^1=0$, iterating \eqref{eq:explicit-log-odds-radius} gives
\begin{equation}
\sup_{k,s,a,a'}
\left|
\log\frac{\pi_h^k(a\mid s)}{\pi_h^k(a'\mid s)}
\right|
\leq
\frac{H(1+I\lambda_{\max}+\tau\log|\cA|)+\epsilon_{\rm act}}{\tau}
\triangleq L_{\rm act}.
\label{eq:explicit-reachable-log-odds-radius}
\end{equation}

For a stage-$h$ actor parameter $\omega_h\in\mathbb R^{d_{\rm a}}$, define
\begin{equation}
f_{\omega,h}(s,a)
\triangleq
\langle\varphi(s,a),\omega_h\rangle
-\frac1{|\cA|}\sum_{b\in\cA}\langle\varphi(s,b),\omega_h\rangle.
\label{eq:explicit-actor-function}
\end{equation}
The subtracted term is independent of $a$, so \eqref{eq:master-actor-policy} gives
\begin{equation}
\pi_h(a\mid s,\omega)
=
\frac{\exp(f_{\omega,h}(s,a))}
{\sum_{b\in\cA}\exp(f_{\omega,h}(s,b))},
\qquad
f_{\omega,h}(s,a)
=
\frac1{|\cA|}\sum_{b\in\cA}
\log\frac{\pi_h(a\mid s,\omega)}{\pi_h(b\mid s,\omega)}.
\label{eq:explicit-actor-function-policy}
\end{equation}
Equations \eqref{eq:explicit-reachable-log-odds-radius} and
\eqref{eq:explicit-actor-function-policy} imply
$\|f_{\omega^k,h}\|_\infty\leq L_{\rm act}$. Define
\begin{equation}
\mathcal F_h^{\rm off}
\triangleq
\left\{
f_{\omega,h}:\omega_h\in\mathbb R^{d_{\rm a}},\;
\|f_{\omega,h}\|_\infty\leq L_{\rm act}
\right\}.
\label{eq:explicit-actor-function-class}
\end{equation}
Thus $f_{\omega^k,h}\in\mathcal F_h^{\rm off}$ for every $k$ and every realization of
$\cD^k$. Since $\omega_h\mapsto f_{\omega,h}$ is linear and its image has dimension at most
$d_{\rm a}$, a volumetric bound gives
\begin{equation}
\log\mathcal N_\infty(\varepsilon,\mathcal F_h^{\rm off})
\leq
d_{\rm a}\log\left(1+\frac{2L_{\rm act}}{\varepsilon}\right),
\label{eq:explicit-reachable-actor-cover}
\end{equation}
for every $\varepsilon>0$.

The fixed-policy classes in \eqref{eq:explicit-fixed-policy-value-class} already contain every
value function that can be produced by the backward recursion. Define
\begin{equation}
\mathcal V_{j,h}^{\rm off}(\omega)
\triangleq
\mathcal V_{j,h}(\pi(\omega),B_K),
\qquad
\mathcal V_{j,h}^{\rm off}
\triangleq
\bigcup_{\omega:\,f_{\omega,h}\in\mathcal F_h^{\rm off}}
\mathcal V_{j,h}^{\rm off}(\omega),
\qquad
\mathcal V_h^{\rm off}
\triangleq
\bigcup_{j\in\{r,u_1,\ldots,u_I,\psi\}}\mathcal V_{j,h}^{\rm off}.
\label{eq:explicit-off-policy-value-class}
\end{equation}
At the terminal stage, set
$\mathcal V_{j,H+1}^{\rm off}=\mathcal V_{H+1}^{\rm off}=\{0\}$.
The bounds defining \eqref{eq:explicit-actor-function-class} and
\eqref{eq:explicit-off-policy-value-class} depend only on the problem parameters and algorithmic
budgets; hence both classes are fixed independently of the realized data.
The oracle definition and \eqref{eq:explicit-off-policy-regression-vector-radius} imply
\begin{equation}
V_{j,h}^k
=
v_{j,h}^{\pi^k,\widehat\theta_{j,h}^k,(\Sigma_h^k)^{-1}}
\in
\mathcal V_{j,h}^{\rm off}
\subseteq\mathcal V_h^{\rm off}
\label{eq:explicit-off-policy-value-membership}
\end{equation}
for every $(k,h,j)$.

Suppose $\|f_{\omega,h}-f_{\omega',h}\|_\infty\leq\nu\leq1$. Then
\begin{equation}
\left|
\log\frac{\pi_h(a\mid s,\omega)}{\pi_h(a\mid s,\omega')}
\right|
\leq2\nu,
\qquad
\|\pi_h(\cdot\mid s,\omega)-\pi_h(\cdot\mid s,\omega')\|_1
\leq e^{2\nu}-1
\leq C\nu.
\label{eq:explicit-softmax-cover-transfer}
\end{equation}
Using \eqref{eq:explicit-actor-function-policy} and
$\|f_{\omega,h}\|_\infty,\|f_{\omega',h}\|_\infty\leq L_{\rm act}$ further gives
\begin{equation}
\begin{aligned}
&\Bigg|
\mathbb E_{a\sim\pi_h(\cdot\mid s,\omega)}
[-\log\pi_h(a\mid s,\omega)]
\\[-2pt]
&\hspace{27mm}{}
-\mathbb E_{a\sim\pi_h(\cdot\mid s,\omega')}
[-\log\pi_h(a\mid s,\omega')]
\Bigg|
\leq C(1+L_{\rm act})\nu,
\\
&\left|
\mathbb E_{\pi_h(\cdot\mid s,\omega)}[Q]
-\mathbb E_{\pi_h(\cdot\mid s,\omega')}[Q]
\right|
\leq C\|Q\|_\infty\nu.
\end{aligned}
\label{eq:explicit-policy-average-cover-transfer}
\end{equation}
Consequently, for the same $(j,w,M)$,
\begin{equation}
\left\|
v_{j,h}^{\pi(\omega),w,M}-v_{j,h}^{\pi(\omega'),w,M}
\right\|_\infty
\leq
C\bigl(H\log(e|\cA|)+L_{\rm act}\bigr)\nu.
\label{eq:explicit-actor-to-value-cover-transfer}
\end{equation}
Applying Lemma~\ref{lem:explicit-linear-value-cover} with
$B=B_K$, $L_h=C(H\log(e|\cA|)+L_{\rm act})$, and
$\mathcal F_h=\mathcal F_h^{\rm off}$, and then using
\eqref{eq:explicit-reachable-actor-cover}, gives
\begin{equation}
\log\mathcal N_\infty(\varepsilon,\mathcal V_h^{\rm off})
\leq
C(d_{\rm c}^2+d_{\rm c}+d_{\rm a})
\left[\iota_K+\log(1+1/\varepsilon)\right].
\label{eq:explicit-adaptive-value-cover}
\end{equation}
The $d_{\rm c}^2+d_{\rm c}$ terms come from the critic class, while $d_{\rm a}$ comes from
covering the actor class.

Since each $\mathcal V_{j,h+1}^{\rm off}$ is deterministic and contains every corresponding
$V_{j,h+1}^k$, Lemma~\ref{lem:explicit-uniform-linear-confidence}, applied with
$\mathcal W_{h+1}=\mathcal V_{j,h+1}^{\rm off}$, $m_{\max}=K$,
$\varepsilon_0=(2KHd_{\rm c})^{-2}$, $\delta_0=\delta/[2H(I+2)]$, and the same
$(Y_j,S_j)$ bounds as in the on-policy case, gives, with probability at least
$1-\delta/2$,
\begin{equation}
\left|
\langle\phi(s,a),\widehat\theta_{j,h}^k-\theta_{j,h}^k\rangle
\right|
\leq
\alpha_j\|\phi(s,a)\|_{(\Sigma_h^k)^{-1}},
\qquad j\in\{r,u_1,\ldots,u_I,\psi\},
\label{eq:explicit-uniform-self-normalized-event}
\end{equation}
simultaneously for every iterate and $(s,a,h)$. Uniformity over
$\mathcal V_{j,h+1}^{\rm off}$ and \eqref{eq:explicit-off-policy-value-membership} permits the
 $V_{j,h+1}^k$ to depend on the full dataset $\cD^k$ used in the regression; the
actor $\pi^k$ depends on the preceding prefix $\cD^{k-1}\subseteq\cD^k$.

\paragraph{Optimism and evaluation error.}
Applying the componentwise derivation
\eqref{eq:explicit-on-policy-component-optimism}-
\eqref{eq:explicit-on-policy-component-consistency} on the event
\eqref{eq:explicit-uniform-self-normalized-event} gives, for every component $j$,
\begin{equation}
0\leq V_j^k-V_j^{\pi^k}
\leq
2\sum_{h=1}^H
\mathbb E_{d_h^{\pi^k}}[b_{j,h}^k],
\label{eq:explicit-off-policy-component-consistency}
\end{equation}
where $V_j^{\pi^k}=V_{\psi_k}^{\pi^k}$ for $j=\psi$.
Since $\Sigma_h^k\succeq\Sigma_h^{k-1}$,
$\|\phi(s,a)\|_{(\Sigma_h^k)^{-1}}
\leq\|\phi(s,a)\|_{(\Sigma_h^{k-1})^{-1}}$.
Define
\begin{equation}
W_k
\triangleq
\sum_{h=1}^H
\mathbb E_{(s,a)\sim d_h^{\pi^k}}
\left[\|\phi(s,a)\|_{(\Sigma_h^{k-1})^{-1}}\right].
\label{eq:explicit-predictable-width}
\end{equation}
Therefore, \eqref{eq:explicit-off-policy-component-consistency} and
\eqref{eq:explicit-confidence-radii} imply
\begin{align*}
0\leq V_{z_k}^k-V_{z_k}^{\pi^k}
&\leq
2\bigl(\alpha_r+I\lambda_{\max}\alpha_u+\tau\alpha_\psi\bigr)W_k,
\\
\max_{i\in[I]}|V_{u_i}^k-V_{u_i}^{\pi^k}|
&\leq2\alpha_uW_k.
\end{align*}
The first right-hand side dominates the second and satisfies
\begin{equation}
\begin{aligned}
2\bigl(\alpha_r+I\lambda_{\max}\alpha_u+\tau\alpha_\psi\bigr)W_k
\leq {}&
C H\sqrt{(d_{\rm c}^2+d_{\rm c}+d_{\rm a})\iota_K}
\\[-2pt]
&\quad{}\times
\bigl(1+I\lambda_{\max}+\tau\log(e|\cA|)\bigr)W_k.
\end{aligned}
\label{eq:explicit-z-consistency-width}
\end{equation}
We take the right-hand side of \eqref{eq:explicit-z-consistency-width} as $\beta_k$. Utility
truncation gives $0\leq V_{u_i}^k\leq H$.

\paragraph{Cumulative width.}
Let $X_{k,h}=\|\phi(s_h^k,a_h^k)\|_{(\Sigma_h^{k-1})^{-1}}$.
Both $\pi^k$ and $\Sigma_h^{k-1}$ are measurable before trajectory $k$, so
$W_k=\mathbb E[\sum_h X_{k,h}\mid\mathcal F_{k-1}]$. Since
$\sum_hX_{k,h}\in[0,H]$, Azuma-Hoeffding gives, with probability at least $1-\delta/2$,
\begin{equation}
\sum_{k=1}^{K-1}W_k
\leq
\sum_{k=1}^{K-1}\sum_{h=1}^HX_{k,h}
+H\sqrt{2(K-1)\log(2/\delta)}.
\label{eq:explicit-width-martingale}
\end{equation}
For each stage, the elliptical-potential identity and
$\operatorname{tr}(\Sigma_h^{K-1})\leq d_{\rm c}+K-1$ imply
\begin{equation}
\sum_{k=1}^{K-1}X_{k,h}^2
\leq
2\log\det(\Sigma_h^{K-1})
\leq
2d_{\rm c}\log\left(1+\frac{K-1}{d_{\rm c}}\right).
\label{eq:explicit-stagewise-elliptical-potential}
\end{equation}
Cauchy-Schwarz over $(k,h)$ gives
\begin{equation}
\sum_{k=1}^{K-1}\sum_{h=1}^HX_{k,h}
\leq
H\sqrt{
2d_{\rm c}(K-1)
\log\left(1+\frac{K-1}{d_{\rm c}}\right)
}.
\label{eq:explicit-realized-cumulative-width}
\end{equation}
Combining \eqref{eq:explicit-width-martingale} and
\eqref{eq:explicit-realized-cumulative-width} yields
\begin{align}
\sum_{k=1}^{K-1}W_k
&\leq
C H\sqrt{d_{\rm c}K\log(eK/\delta)}.
\label{eq:explicit-cumulative-width}
\end{align}
Combining \eqref{eq:explicit-z-consistency-width} and
\eqref{eq:explicit-cumulative-width} proves the cumulative consistency bound in
\eqref{eq:explicit-linear-off-policy-ope-rate}, since
$d_{\rm c}^2+d_{\rm c}+d_{\rm a}=\Theta(d_{\rm c}^2+d_{\rm a})$.
The confidence and width events hold jointly with probability at least $1-\delta$.

\paragraph{Local norm.}
This calculation is independent of the sampling mode. For $p_a=\pi_h^k(a\mid s)$, clipping gives
\begin{equation}
0\leq Q_{z_k,h}^k(s,a)
\leq
H(1+I\lambda_{\max}+\tau\log|\cA|)+\tau\log(1/p_a).
\label{eq:explicit-shaped-q-envelope}
\end{equation}
Applying the entropy-moment argument
\eqref{eq:explicit-on-policy-escort-moment}-\eqref{eq:explicit-on-policy-escort-log-moment}
under the two stated stepsize conditions therefore yields
\begin{equation}
\sum_a p_a e^{\eta Q_{z_k,h}^k(s,a)}
\bigl(Q_{z_k,h}^k(s,a)\bigr)^2
\leq
C H^2\bigl(1+I\lambda_{\max}+\tau\log(e|\cA|)\bigr)^2.
\label{eq:explicit-escort-moment}
\end{equation}
This verifies the remaining oracle condition and completes the off-policy branch.
\end{proof}

\subsection{Proof of the off-policy final rate}
\label{app:explicit-off-policy-final-rate}

Fix $\epsilon,\delta\in(0,1)$ and choose the parameters as
\begin{equation}
\begin{aligned}
\lambda_{\max}
&=\frac{2H\log(e|\cA|)}{\xi},
\quad
\tau
=\Theta\left(
\min\left\{1,
\max\left\{
\frac{\xi\epsilon}{H\log(e|\cA|)},
\frac{H^{2/3}\xi^{2/3}\epsilon_{\rm act}^{1/3}}
{\log^{2/3}(e|\cA|)}
\right\}
\right\}
\right),
\\
\eta
&=\widetilde\Theta\left(
\frac{\tau\xi^4\epsilon^4}
{d_{\rm c}(d_{\rm c}^2+d_{\rm c}+d_{\rm a})I^4H^{14}}
\right),
\qquad
K=\widetilde\Theta\left(\frac1{\eta\tau}\right).
\end{aligned}
\label{eq:explicit-final-parameters}
\end{equation}

\begin{proof}[Proof of Theorem~\ref{thm:explicit-linear-last-iterate}, off-policy branch]
Off-policy sampling supplies a cumulative $O(\sqrt K)$ oracle certificate rather than a
uniform $O(N^{-1/2})$ certificate. Set $L=\log(e|\cA|)$ and
$L_\epsilon=\log(e+H^4L/\epsilon^2+H^5L^2/(\xi^2\epsilon^2))$.
Uniform policy initialization, $\lambda_1=0$, and
\eqref{eq:master-regularised-dual-radius} give
\begin{equation}
\Phi_1
\leq
H\log|\cA|+\frac{H^2L^2}{2\xi^2}.
\label{eq:explicit-initial-potential-bound}
\end{equation}

\paragraph{Potential control.}
For the local-norm and actor terms, sum the geometric weights in
Theorem~\ref{thm:master-recursion}. For the nonnegative oracle errors, bound each weight
by one. This gives
\begin{equation}
\begin{aligned}
\Phi_K
\leq {}&
e^{-\eta\tau(K-1)}\Phi_1
+\frac{\eta}{2\tau}
\sup_{k<K}\bigl(HC_{\eta,\tau,\Lambda,k}+G_\tau^2\bigr)
\\
&+\eta(1+I\lambda_{\max})\sum_{k=1}^{K-1}\beta_k
+\frac{H\epsilon_{\rm act}}{\tau}.
\end{aligned}
\label{eq:explicit-unrolled-final-potential}
\end{equation}
Substituting $\lambda_{\max}=2HL/\xi$ into the off-policy certificate gives
\begin{equation}
\begin{aligned}
\sup_{k<K}\bigl(HC_{\eta,\tau,\Lambda,k}+G_\tau^2\bigr)
&\leq C\frac{I^2H^5L^2}{\xi^2},
\\
(1+I\lambda_{\max})\sum_{k=1}^{K-1}\beta_k
&\leq C\frac{I^2H^4L^2}{\xi^2}
\sqrt{d_{\rm c}(d_{\rm c}^2+d_{\rm a})K}\,\iota_K.
\end{aligned}
\label{eq:explicit-final-D-B}
\end{equation}
Here $\iota_K$ is defined in \eqref{eq:explicit-oracle-log-factor}; it also absorbs the
logarithm in the cumulative-width bound.
Take
\[
\eta
=\frac{c\tau\xi^4\epsilon^4}
{d_{\rm c}(d_{\rm c}^2+d_{\rm a})I^4H^{14}L^4\iota_K^2L_\epsilon},
\qquad
K=1+\left\lceil\frac{C L_\epsilon}{\eta\tau}\right\rceil.
\]
These implicit logarithmic choices have the polynomial orders in
\eqref{eq:explicit-final-parameters}. The normalized parameter ranges imply that this
stepsize is no larger than a constant multiple of
$\tau\xi^2\epsilon^2/(I^2H^8L^2)$, which controls the local-norm contribution,
and is small enough for the oracle stability conditions.
Moreover, $K=O(L_\epsilon/(\eta\tau))$ gives
\[
\eta(1+I\lambda_{\max})\sum_{k=1}^{K-1}\beta_k
\leq
C\frac{I^2H^4L^2}{\xi^2}
\sqrt{\frac{d_{\rm c}(d_{\rm c}^2+d_{\rm a})\eta L_\epsilon}{\tau}}\,
\iota_K
=O(\epsilon^2/H^3).
\]
The contraction term has the same order by the choice of $K$. Consequently,
\begin{equation}
\Phi_K
\leq
C\frac{\epsilon^2}{H^3}
+\frac{H\epsilon_{\rm act}}{\tau}.
\label{eq:explicit-final-potential-bound}
\end{equation}

\paragraph{Temperature and sample count.}
Applying \eqref{eq:master-regularisation-conversion} to
\eqref{eq:explicit-final-potential-bound} and using
$\sqrt{x+y}\leq\sqrt x+\sqrt y$ yields
\begin{equation}
\max\left\{
[V_r^{\pi^\star}-V_r^{\pi^K}]_+,\,
\max_{i\in[I]}[c_i-V_{u_i}^{\pi^K}]_+
\right\}
\leq
C\left[
\epsilon+\frac{HL}{\xi}\tau
+H^2\sqrt{\frac{\epsilon_{\rm act}}{\tau}}
\right].
\label{eq:explicit-error-temperature-envelope}
\end{equation}
The two temperature-dependent terms balance at
$\tau=(H^2\xi^2\epsilon_{\rm act}/L^2)^{1/3}$, while the target-accuracy
scale is $\xi\epsilon/(HL)$. Taking the larger scale and capping it at one
gives \eqref{eq:explicit-final-parameters} and
\eqref{eq:explicit-linear-final-error}.

Each off-policy iteration adds one episode to $\cD^k$, and each episode contains $H$ transitions.
Before substituting the temperature, the total transition count is
\begin{equation}
HK
=
O\left(
\frac{d_{\rm c}(d_{\rm c}^2+d_{\rm a})I^4H^{15}L^4}
{\xi^4\epsilon^4\tau^2}\,
\iota_K^2L_\epsilon^2
\right).
\label{eq:explicit-sample-before-temperature}
\end{equation}
This count already dominates the local-norm and stability budgets, as shown by the
stepsize comparisons above. The temperature contributes
\begin{equation}
\frac1{\tau^2}
=
O\left(
\max\left\{1,\,
\min\left\{
\frac{H^2L^2}{\xi^2\epsilon^2},\,
\frac{L^{4/3}}{H^{4/3}\xi^{4/3}\epsilon_{\rm act}^{2/3}}
\right\}\right\}
\right).
\label{eq:explicit-inverse-temperature-rate}
\end{equation}
For $\epsilon_{\rm act}=0$, the second entry of the minimum is interpreted
as $+\infty$. Substituting \eqref{eq:explicit-inverse-temperature-rate} into
\eqref{eq:explicit-sample-before-temperature}, dividing by $H$, and suppressing
polylogarithmic factors gives \eqref{eq:explicit-final-sample-complexity}.
When $\epsilon_{\rm act}\leq\xi/H^2$, the cap is inactive and the outer
maximum can be removed.
\end{proof}


\section{Proofs for Section~\ref{section:general-function-approximation}}
\label{app:general-function-approximation}
We prove the general oracle certificates in three steps: control the fitted Bellman targets
uniformly over the value function used by the oracle, bound the resulting confidence widths, and
substitute these bounds into the master contraction.  Throughout, $\mathcal X$, $\mathcal J$,
$\mathcal F_{j,h}$, and the actor notation are those of the main text.

\subsection{General function classes and complexity measures}
\label{app:general-complexity}
We first specify the two complexity measures used in the main theorem.  Following
\citet[Definitions~3-4]{russo2013eluder}, a point $x\in\mathcal X$ is $\rho$-dependent on
$x_1,\ldots,x_m$ with respect to $\mathcal F$ if every $f,f'\in\mathcal F$ satisfying
$\bigl(\sum_{q=1}^m(f(x_q)-f'(x_q))^2\bigr)^{1/2}\leq\rho$ also satisfies
$|f(x)-f'(x)|\leq\rho$; otherwise it is $\rho$-independent.

\begin{definition}[Scale-dependent eluder dimension]
\label{def:general-eluder-dimension}
The eluder dimension $d_{\rm E}(\mathcal F,\rho)$ is the supremum of the integers $L\geq0$ for which
there exist $x_1,\ldots,x_L$ and a common $\rho'\geq\rho$ such that every $x_\ell$ is
$\rho'$-independent of its predecessors.  For the component critic classes, define
\begin{equation}
d_{\rm E}(\rho)
\triangleq
1\vee\max_{j\in\mathcal J,\,h\in[H]}
d_{\rm E}(\mathcal F_{j,h},\rho).
\label{eq:general-aggregate-eluder}
\end{equation}
\end{definition}

\begin{definition}[Domain metric entropy]
\label{def:general-domain-entropy}
For a critic class $\mathcal F$, let
$d_{\mathcal F}(x,x')=\sup_{f\in\mathcal F}|f(x)-f(x')|$.  The aggregate domain metric entropy is
\begin{equation}
\mathfrak h_{\mathcal X}(\rho)
\triangleq
1\vee\max_{j\in\mathcal J,\,h\in[H]}
\log\mathcal N\!\left(
\rho,\mathcal X,d_{\mathcal F_{j,h}}
\right),
\label{eq:general-domain-entropy}
\end{equation}
where $\mathcal N(\rho,\mathcal X,d_{\mathcal F})$ is the smallest cardinality of a
$\rho$-cover of $\mathcal X$ in the class-induced pseudometric $d_{\mathcal F}$.
\end{definition}

These measures have different roles.  The eluder dimension controls how often confidence widths
can remain large along a sequence of observations.  The domain entropy bounds the number of possible rounded representatives of
state-action pairs $(s,a)\in\mathcal X$, which is used below to bound the number of possible data-dependent bonuses.
This is the domain-cover term in \citet[Assumption~2]{wang2020generalvalue}.
The sup-norm critic cover $\mathcal N_\infty(\rho,\mathcal F_{j,h})$ need not be imposed
separately because \citet[Theorem~14]{hanneke2024star} bounds it by the scale-dependent eluder
dimension; see Lemma~\ref{lem:general-eluder-controls-critic-cover}.  For a finite state-action
domain, $\mathfrak h_{\mathcal X}(\rho)\leq1\vee\log(|\cS||\cA|)$.

\paragraph{Accuracy-dependent resolutions.}
Fix the target accuracy $\epsilon$ and choose $\tau$ as in
\eqref{eq:general-on-policy-parameters}. Write $L=\log(e|\cA|)$, $B=HL$, and
$L_\epsilon=\log(e+H^4L/\epsilon^2+H^5L^2/(\xi^2\epsilon^2))$.
For a sufficiently small fixed universal constant $c_0>0$, define, in this order,
\begin{equation}
\begin{aligned}
r_\epsilon
&=\min\left\{1,\frac{c_0\tau\xi^2\epsilon^2}
{L_\epsilon I^2H^6L^2}\right\},
&d_0&=d_{\rm E}(r_\epsilon),
\\
q_\epsilon&=\frac{r_\epsilon}{\sqrt{d_0}},
&\nu_\epsilon&=\frac{c_0q_\epsilon^2}{B}.
\end{aligned}
\label{eq:general-accuracy-resolutions}
\end{equation}
The quantities in Theorem~\ref{thm:general-last-iterate} are evaluated at
\begin{equation}
d_{\rm E}
=d_{\rm E}(\nu_\epsilon/24),
\qquad
\mathfrak h_{\mathcal X}
=\mathfrak h_{\mathcal X}(\nu_\epsilon).
\label{eq:general-complexity-convention}
\end{equation}
We consider classes for which these positive-resolution quantities are finite; no restriction
on their growth as the resolution decreases is imposed. All scales are fixed before choosing
the sample cap $M=N$ or $M=K$, and $d_0\leq d_{\rm E}$. The confidence level and sample cap
enter only logarithmic factors, not the arguments in
\eqref{eq:general-complexity-convention}. In particular, the stated rates retain any
accuracy dependence of $d_{\rm E}$ and $\mathfrak h_{\mathcal X}$; these quantities are not
treated as constants for an arbitrary class.

The construction below retains a squared-norm error of order $Mq_\epsilon^2$ instead of
forcing that error to vanish with $M$. Eluder counting at resolution $r_\epsilon$ converts
it into a width contribution of order $Mq_\epsilon\sqrt{d_0}=Mr_\epsilon$.
The choice of $r_\epsilon$ makes its contribution to the master potential
$O(\epsilon^2/H^3)$.

\begin{remark}[Linear specialization]
Consider the bounded linear class
$\mathcal F_{j,h}^{\rm lin}
=\{x\mapsto\langle\phi(x),\theta\rangle:
\|\theta\|_2\leq R,\ 0\leq\langle\phi(x),\theta\rangle\leq B_{j,h}
\text{ for all }x\}$, where $B_{r,h}=B_{u_i,h}=H+1-h$,
$B_{\psi,h}=(H-h)\log|\cA|$, and $\|\phi(x)\|_2\leq1$.  Then
\[
d_{\mathcal F_{j,h}^{\rm lin}}(x,x')
\leq
R\|\phi(x)-\phi(x')\|_2.
\]
A Euclidean cover of the feature image therefore gives
\[
\log\mathcal N(\rho,\mathcal X,d_{\mathcal F_{j,h}^{\rm lin}})
\leq
d_{\rm c}\log\!\left(1+\frac{2R}{\rho}\right).
\]
The linear-CMDP Bellman targets of Section~\ref{section:linear-cmdp-explicit-actor} belong to
these classes with $R\leq H\sqrt{d_{\rm c}}\log(e|\cA|)$.  Consequently,
\begin{equation}
\mathfrak h_{\mathcal X}(\rho)
\leq
1\vee\left[d_{\rm c}\log\!\left(
1+\frac{2H\sqrt{d_{\rm c}}\log(e|\cA|)}{\rho}
\right)\right]
=\widetilde O(d_{\rm c}).
\label{eq:general-linear-domain-entropy}
\end{equation}
The cover depends only on the feature image, so $\cS$ need not be finite.  Since the classes
consist of bounded linear functions, the linear eluder-dimension bound also gives
$d_{\rm E}(\rho)=\widetilde O(d_{\rm c})$ \citep{russo2013eluder}.  The oracle clips its
optimistic estimates separately; clipping is not part of the critic-class definition used here.
\end{remark}

\subsection{Stable optimistic policy evaluation}
\label{app:general-ope}
For a finite multiset $Z=(x_\ell)_{\ell=1}^n$ of points in $\mathcal X$,
write $\|f\|_Z^2=\sum_{\ell=1}^n f(x_\ell)^2$, so repeated points
contribute separately ($Z$ may contain repeated points in $\mathcal X$). If occurrence $x_\ell$ has weight $w_\ell\ge0$,
use $\|f\|_Z^2=\sum_{\ell=1}^n w_\ell f(x_\ell)^2$.  We adapt the sensitivity sampling and rounding of
\citet[Algorithm~3]{wang2020generalvalue} to the fixed resolutions in
\eqref{eq:general-accuracy-resolutions}. Unlike their budget-dependent construction, we
fix the error $Mq_\epsilon^2$ in the confidence radius.
Lemma~\ref{lem:general-stable-bonus} states the two properties needed below: domination of the
full-data confidence width and a deterministic bound on the number of possible bonuses.

\begin{algorithm}[H]
\caption{$\mathrm{StableBonus}(\mathcal F,\overline f,Z,\alpha,M,\delta_0)$}
\label{alg:general-stable-bonus}
\begin{algorithmic}[1]
\STATE\label{algline:stable-bonus-input} Write $Z=(x_\ell)_{\ell=1}^n$, where $n=|Z|\le M$ and
$\alpha\ge Mq_\epsilon^2$. Set
\[R_{\mathcal F}
=\sup_{x\in\mathcal X}\sup_{f,g\in\mathcal F}|f(x)-g(x)|.\]
\STATE\label{algline:stable-bonus-edge} If $R_{\mathcal F}=0$, return $b\equiv0$.
If $n=0$, return $b(x)=\sup_{f,g\in\mathcal F}|f(x)-g(x)|$.
\item[] \textcolor{blue}{\textit{\# Use sensitivity sampling to decide which $x_\ell$ enter $Z'$.}}
\FOR{$\ell=1,\ldots,n$}
    \STATE\label{algline:stable-bonus-sensitivity} Compute $\operatorname{sens}_\ell$ by
    \eqref{eq:general-fixed-scale-sensitivity} and $p_\ell$ by
    \eqref{eq:general-sensitivity-sampling}.
    \STATE\label{algline:stable-bonus-draw} Independently draw
    $\zeta_\ell\sim\operatorname{Bernoulli}(p_\ell)$.
\ENDFOR
\STATE\label{algline:stable-bonus-weighted} Let $Z'$ contain $x_\ell$ with weight $1/p_\ell$
whenever $\zeta_\ell=1$.
\STATE\label{algline:stable-bonus-cap} Set $J_M$ by \eqref{eq:general-compression-size}.
\IF{$\sum_{\ell=1}^n\zeta_\ell>J_M$ or
    $\sum_{\ell=1}^n\zeta_\ell/p_\ell>2n$}
    \STATE\label{algline:stable-bonus-fallback} \textbf{return} $b\equiv R_{\mathcal F}$.
\ENDIF
\item[] \textcolor{blue}{\textit{\# Replace each selected $x_\ell$ by a nearby cover point}}
\STATE Initialize $Z^\sharp$ as an empty weighted multiset.
\FOR{$\ell=1,\ldots,n$}
    \IF{$\zeta_\ell=1$}
        \STATE\label{algline:stable-bonus-round-point} Choose $x_\ell^\sharp$ from a fixed
        $\nu_\epsilon$-cover of $(\mathcal X,d_{\mathcal F})$
        such that $d_{\mathcal F}(x_\ell,x_\ell^\sharp)
        \le\nu_\epsilon$.
        \STATE\label{algline:stable-bonus-round-weight} Add $x_\ell^\sharp$ to $Z^\sharp$
        with weight $1/p_\ell$.
    \ENDIF
\ENDFOR
\STATE\label{algline:stable-bonus-round-center} Choose $f^\sharp$ from a fixed internal sup-norm
$\nu_\epsilon$-cover of $\mathcal F$ such that
$\|\overline f-f^\sharp\|_\infty\le\nu_\epsilon$.
\STATE\label{algline:stable-bonus-confidence} Set $\widehat{\mathcal F}
=\{f\in\mathcal F:\|f-f^\sharp\|_{Z^\sharp}^2\le7\alpha\}$.
\STATE\label{algline:stable-bonus-output} \textbf{return}
$b(x)=\sup_{f,g\in\widehat{\mathcal F}}|f(x)-g(x)|$
for $x\in\mathcal X$.
\end{algorithmic}
\end{algorithm}

Algorithm~\ref{alg:general-stable-bonus} takes a critic class
$\mathcal F\subseteq(\mathcal X\to[0,B])$, a fitted center
$\overline f\in\mathcal F$, data $Z=(x_\ell)_{\ell=1}^n$, a squared
empirical-norm radius $\alpha$, a sample cap $M$, and a compression
failure budget $\delta_0$.
It returns a bonus $b:\mathcal X\to[0,B]$.

Line~\ref{algline:stable-bonus-input} specifies $n\le M$,
$\alpha\ge Mq_\epsilon^2$, and the class diameter $R_{\mathcal F}$;
line~\ref{algline:stable-bonus-edge} handles a zero-range class or
empty data.
Lines~\ref{algline:stable-bonus-sensitivity}--\ref{algline:stable-bonus-weighted}
compute the sensitivity $\operatorname{sens}_\ell$, select $x_\ell$
with probability $p_\ell$, and assign selected points weight
$1/p_\ell$. Larger sensitivity increases the selection probability,
while inverse-probability weighting compensates for subsampling.
The budget $\delta_0$ enters $p_\ell$ and $J_M$ through their
definitions. Lines~\ref{algline:stable-bonus-cap}--\ref{algline:stable-bonus-fallback}
return $R_{\mathcal F}$ if the selected count exceeds $J_M$ or the
total weight exceeds $2n$.
Lines~\ref{algline:stable-bonus-round-point}--\ref{algline:stable-bonus-round-weight}
round selected points into $Z^\sharp$ without changing their weights;
line~\ref{algline:stable-bonus-round-center} rounds $\overline f$
to $f^\sharp$ which is its nearest point in the cover of $\mathcal F$.
Lines~\ref{algline:stable-bonus-confidence}--\ref{algline:stable-bonus-output}
form $\widehat{\mathcal F}$ and return its pointwise width as $b$.
The checks and fixed covers give a deterministic finite family of
possible bonuses. Regression still uses all supplied samples.

Given the stable bonus computation in Algorithm~\ref{alg:general-stable-bonus}, Algorithm~\ref{alg:general-ope} follows the backward stagewise structure
of the linear oracle in Algorithm~\ref{alg:explicit-linear-ope}: it fits
each component critic, adds an optimism bonus, truncates the estimate,
and computes its policy average. Here $\widehat f_{j,h}^k$ is fitted over
$\mathcal F_{j,h}$, while Algorithm~\ref{alg:general-stable-bonus}
supplies its bonus $b_{j,h}^k$.

Write $\cD=(\cD_1,\ldots,\cD_H)$, where $\cD_h$ contains the
stage-$h$ transition from \textit{every trajectory} in $\cD$.
Importantly, we note that it does not split $\cD$
into separate training subsets and the stage-$h$ regression uses every
transition in $\cD_h$.
Set
$y_{r,h}^\ell=r_h^\ell$, $y_{u_i,h}^\ell=u_{i,h}^\ell$, and
$y_{\psi,h}^\ell=0$. Given $V_{j,h+1}^k$, define
\begin{equation}
\widehat f_{j,h}^k
\in
\argmin_{f\in\mathcal F_{j,h}}
\sum_{\ell\in\cD_h}
\left(
f(s_h^\ell,a_h^\ell)-y_{j,h}^\ell
-V_{j,h+1}^k(s_{h+1}^\ell)
\right)^2.
\label{eq:general-td-loss}
\end{equation}
We assume measurable exact minimizers. An additive least-squares
optimisation tolerance can instead be included in the confidence radius.

\begin{algorithm}[H]
\caption{Stable General $\mathrm{OPE}_{\rm gen}(\pi^k,\lambda_k,\cD)$}
\label{alg:general-ope}
\begin{algorithmic}[1]
\STATE \textbf{Input}: $\pi^k,\lambda_k$, the stage-indexed dataset $\cD$, critic classes,
radii $\{\alpha_{j,h}\}$, sample cap $M$, and failure budgets $\{\delta_{j,h}\}$.
\STATE Set $V_{j,H+1}^k=0$ for every $j\in\mathcal J$.
\FOR{$h=H,H-1,\ldots,1$}
\FOR{$j\in\mathcal J$}
\STATE Compute $\widehat f_{j,h}^k$ by \eqref{eq:general-td-loss} using all of $\cD_h$.
\STATE $b_{j,h}^k\leftarrow
\mathrm{StableBonus}(\mathcal F_{j,h},\widehat f_{j,h}^k,
\{(s_h^\ell,a_h^\ell):\ell\in\cD_h\},\alpha_{j,h},M,\delta_{j,h})$.
\STATE Set
$\overline Q_{j,h}^k
=\mathrm{Truncate}_{[0,B_{j,h}]}(\widehat f_{j,h}^k+b_{j,h}^k)$,
where $B_{r,h}=B_{u_i,h}=H+1-h$ and
$B_{\psi,h}=(H-h)\log|\cA|$.
\ENDFOR
\STATE Set $Q_{r,h}^k=\overline Q_{r,h}^k$,
$Q_{u_i,h}^k=\overline Q_{u_i,h}^k$, and
$Q_{\psi,h}^k(s,a)=-\log\pi_h^k(a\mid s)+\overline Q_{\psi,h}^k(s,a)$.
\STATE Set
$V_{j,h}^k(s)=\mathbb E_{a\sim\pi_h^k(\cdot\mid s)}[Q_{j,h}^k(s,a)]$
for every $j\in\mathcal J$.
\ENDFOR
\STATE Return
$Q_{z_k,h}^k=Q_{r,h}^k+\lambda_k^\top Q_{u,h}^k+\tau Q_{\psi,h}^k$
and $V_u^k=(V_{u_1,1}^k(s_1),\ldots,V_{u_I,1}^k(s_1))$.
\end{algorithmic}
\end{algorithm}

\subsection{Technical lemmas for uniform confidence}
\label{app:general-confidence}
The stage-$h$ regression uses $V_{j,h+1}^k$, which is computed
from the same trajectories. We therefore need a confidence bound that
holds for every value function the oracle might produce, rather than
for one fixed value function. This subsection shows that the stable
bonus $b_{j,h}^k$ controls the empirical critic width and belongs to
a deterministic finite family. Together with covers of the critics
and actors, this gives a class containing $V_{j,h+1}^k$ over which
the regression bound holds uniformly.

For
$\mathcal F\subseteq(\mathcal X\to[0,B])$, a finite multiset $Z\subset\mathcal X$, and
$\gamma\geq0$, define the empirical confidence set and pointwise width as
\[
\mathcal C_Z(\overline f,\gamma)
=
\{f\in\mathcal F:\|f-\overline f\|_Z^2\leq\gamma\},
\qquad
w(\mathcal C,x)
=
\sup_{f,f'\in\mathcal C}|f(x)-f'(x)|.
\]
Recall the pseudometric
$d_{\mathcal F}(x,x')=\sup_{f\in\mathcal F}|f(x)-f(x')|$.

\begin{lemma}[Critic entropy from eluder dimension]
\label{lem:general-eluder-controls-critic-cover}
Let $\mathcal F\subseteq(\mathcal X\to[0,B])$ and $0<\rho\leq B$.  Then
\begin{equation}
\log\mathcal N_\infty(\rho,\mathcal F)
\leq
C\left[1+d_{\rm E}\!\left(\mathcal F,\frac{\rho}{6}\right)\right]
\log\!\left(
\frac{eB\left[1+d_{\rm E}(\mathcal F,\rho/6)\right]}{\rho}
\right).
\label{eq:general-eluder-controls-critic-cover}
\end{equation}
Here $\mathcal N_\infty$ is the internal sup-norm covering number used by the stable-bonus
construction.
\end{lemma}

\begin{proof}
Rescale to $\mathcal F/B\subseteq[0,1]$ and let
$m=\log\widetilde{\mathcal N}_\infty(\rho/2,\mathcal F)$, where the tilde denotes an external
cover whose centers need not lie in $\mathcal F$. If $m\leq1$, the right-hand side of
\eqref{eq:general-eluder-controls-critic-cover} is at least a universal constant, so the result holds.
Otherwise, \citet[Theorem~14]{hanneke2024star}, at scale $\rho/(6B)$, gives
\[
\left\lfloor
\frac{2m}{\log(144B^2m/\rho^2)}
\right\rfloor
\leq d_{\rm E}(\mathcal F,\rho/6).
\]
The dimension used in that theorem restricts one function in each witnessing pair to a fixed
center.  Removing this restriction can only increase the longest admissible sequence, so it is
bounded by Definition~\ref{def:general-eluder-dimension}.  Inverting the preceding inequality gives
\[
m
\leq
C\left[1+d_{\rm E}(\mathcal F,\rho/6)\right]
\log\!\left(
\frac{eB[1+d_{\rm E}(\mathcal F,\rho/6)]}{\rho}
\right).
\]
Choose one function of $\mathcal F$ from each nonempty external cover cell.  These functions
form an internal $\rho$-cover of the same cardinality, proving the result.
\end{proof}

\begin{lemma}[Stable-width compression]
\label{lem:general-stable-bonus}
For a component class $\mathcal F$, $|Z|\leq M$, $\overline f\in\mathcal F$, and
$\delta_0\in(0,1)$, set $\alpha_{\mathcal F}(M,\delta_0)=Mq_\epsilon^2$.
For every fixed $\alpha\geq\alpha_{\mathcal F}(M,\delta_0)$,
Algorithm~\ref{alg:general-stable-bonus} satisfies, with probability at least
$1-\delta_0/(16M)$,
\begin{equation}
w\bigl(\mathcal C_Z(\overline f,\alpha),x\bigr)
\leq b(x)
\leq
w\bigl(\mathcal C_Z(\overline f,34\alpha),x\bigr),
\qquad
\forall x\in\mathcal X.
\label{eq:general-stable-width-sandwich}
\end{equation}
For each such radius, every bonus $b$ returned by
Algorithm~\ref{alg:general-stable-bonus} belongs to a deterministic
finite family $\mathcal W(\mathcal F,M,\delta_0;\alpha)$ satisfying
\begin{equation}
\sup_{\alpha\geq\alpha_{\mathcal F}(M,\delta_0)}
\log|\mathcal W(\mathcal F,M,\delta_0;\alpha)|
=\widetilde O(d_{\rm E}^2\mathfrak h_{\mathcal X}).
\label{eq:general-stable-compression-orders}
\end{equation}
The complexity scales are those of \eqref{eq:general-complexity-convention};
only logarithmic factors depend on $M$ and $\delta_0$.
\end{lemma}

\begin{proof}
We follow the sensitivity-sampling strategy of
\citet[Appendix~A.1]{wang2020generalvalue}, but use the fixed denominator floor
$|Z|q_\epsilon^2$ and prove the resulting additive-error bounds explicitly.
The empty dataset and zero-range class have exact widths, so suppose
$Z=(x_1,\ldots,x_n)$ with $1\leq n\leq M$. The component range is at most
$B=H\log(e|\cA|)$.

\paragraph{Sensitivity and retained support.}
For each index $\ell\in[n]$, define
\begin{equation}
\operatorname{sens}_{\ell}
=\sup_{f,g\in\mathcal F}
\frac{(f(x_\ell)-g(x_\ell))^2}
{\max\{\|f-g\|_Z^2,nq_\epsilon^2\}}.
\label{eq:general-fixed-scale-sensitivity}
\end{equation}
Each sensitivity is at most one. For every positive sensitivity, choose a witness
$v_\ell=f_\ell-g_\ell$ whose quotient is at least
$\operatorname{sens}_{\ell}/2$. If $|v_\ell(x_\ell)|\leq q_\epsilon$, then
$\operatorname{sens}_{\ell}\leq2/n$. Partition the remaining indices by
$a<|v_\ell(x_\ell)|\leq2a$, where $a=2^jq_\epsilon$.
Process the points in this group in their original order. Append each
$x_\ell$ to the first existing subsequence for which it is
$a$-independent of the points already assigned to that subsequence.
If there is no such subsequence, start a new one with $x_\ell$.
Every subsequence has length at most $d_{\rm E}(\mathcal F,a)$.
If $x_\ell$ enters subsequence $m\geq2$, it is $a$-dependent on each of the
$m-1$ preceding subsequences. Since $|v_\ell(x_\ell)|>a$, their disjointness gives
\[
\|v_\ell\|_Z^2>(m-1)a^2,
\qquad
\operatorname{sens}_{\ell}
\leq\min\{1,8/(m-1)\}.
\]
Summing over the at most $d_{\rm E}(\mathcal F,a)$ entries of each subsequence and
then the dyadic groups yields
\begin{equation}
\sum_{\ell=1}^n\operatorname{sens}_{\ell}
\leq C\left[
1+d_{\rm E}(\mathcal F,q_\epsilon)
\log(eB/q_\epsilon)\log(en)
\right].
\label{eq:general-sensitivity-sum}
\end{equation}
Let $\ell_M=\log(64M\mathcal N_\infty(\nu_\epsilon,\mathcal F)^2/\delta_0)$.
For a sufficiently large universal constant $C_{\rm s}$, take
\begin{equation}
a_\ell=\min\{1,C_{\rm s}\ell_M(\operatorname{sens}_{\ell}+1/n)\},
\qquad
p_\ell=\frac1{\lfloor1/a_\ell\rfloor},
\qquad
\zeta_\ell\sim\operatorname{Bernoulli}(p_\ell)
\quad\text{independently}.
\label{eq:general-sensitivity-sampling}
\end{equation}
Then $a_\ell\leq p_\ell\leq2a_\ell$, and each retained weight $1/p_\ell$ is an
integer in $[n]$. By \eqref{eq:general-sensitivity-sum} and Bernstein's inequality,
the number of retained indices is at most
\begin{equation}
J_M
=\left\lceil
C\ell_M\left[
1+d_{\rm E}(\mathcal F,q_\epsilon)
\log(eB/q_\epsilon)\log(eM)
\right]\right\rceil
\label{eq:general-compression-size}
\end{equation}
except on an event of probability at most $\delta_0/(64M)$.
Moreover, for indices with $p_\ell<1$,
$1/p_\ell\leq n/(C_{\rm s}\ell_M)$. Thus
\[
\sum_{\ell=1}^n\operatorname{Var}(\zeta_\ell/p_\ell)
\leq \frac{n^2}{C_{\rm s}\ell_M},
\qquad
\Pr\left\{\sum_{\ell=1}^n\zeta_\ell/p_\ell>2n\right\}
\leq\frac{\delta_0}{64M}.
\]

\paragraph{Uniform preservation of empirical distances.}
Let $Z'$ contain each retained point with weight $1/p_\ell$.
For a fixed difference $v=f-g$, put
$T_v=\max\{\|v\|_Z^2,nq_\epsilon^2\}$ and
$Y_\ell=(\zeta_\ell/p_\ell-1)v(x_\ell)^2$.
The indices with $p_\ell=1$ have $Y_\ell=0$, and the others satisfy
\[
|Y_\ell|\leq\frac{T_v}{C_{\rm s}\ell_M},
\qquad
\sum_{\ell=1}^n\mathbb E[Y_\ell^2]
\leq\frac{T_v\|v\|_Z^2}{C_{\rm s}\ell_M}
\leq\frac{T_v^2}{C_{\rm s}\ell_M}.
\]
Bernstein's inequality and a union bound over pairs from a fixed internal
$\nu_\epsilon$-net give
$|\|v\|_{Z'}^2-\|v\|_Z^2|\leq T_v/32$ for all net pairs, with failure probability
at most $\delta_0/(64M)$. For any pair $f,g$, choose net approximants
$\widetilde f,\widetilde g$ and write $\widetilde v=\widetilde f-\widetilde g$.
On the total-weight event,
\[
\|v-\widetilde v\|_Z\leq2\nu_\epsilon\sqrt n,
\qquad
\|v-\widetilde v\|_{Z'}\leq2\nu_\epsilon\sqrt{2n}.
\]
Applying the triangle inequality in these two seminorms and
$2ab\leq t a^2+b^2/t$, with a fixed small $t>0$, extends the net bound to
\begin{equation}
\frac12\|f-g\|_Z^2-nq_\epsilon^2
\leq\|f-g\|_{Z'}^2
\leq\frac32\|f-g\|_Z^2+nq_\epsilon^2,
\qquad f,g\in\mathcal F.
\label{eq:general-compressed-norm}
\end{equation}
Here $\nu_\epsilon\leq c_0q_\epsilon$, so the net-extension errors are absorbed
by $nq_\epsilon^2$ when $c_0$ is sufficiently small.

\paragraph{Rounding and confidence sets.}
Round each retained point in $d_{\mathcal F}$ and the center $\overline f$ in
sup norm, both at resolution $\nu_\epsilon$. The total retained weight is at most $2n$,
so, for $U=\|f-\overline f\|_Z$ and $R=\|f-f^\sharp\|_{Z^\sharp}$,
\[
\left|R-\|f-\overline f\|_{Z'}\right|
\leq3\nu_\epsilon\sqrt{2n}\leq q_\epsilon\sqrt n.
\]
Equation~\eqref{eq:general-compressed-norm} therefore implies
\[
R^2\leq3U^2+4nq_\epsilon^2,
\qquad
U^2\leq4R^2+6nq_\epsilon^2.
\]
Since $\alpha\geq Mq_\epsilon^2\geq nq_\epsilon^2$, these inequalities give
\begin{equation}
\mathcal C_Z(\overline f,\alpha)
\subseteq\mathcal C_{Z^\sharp}(f^\sharp,7\alpha)
\subseteq\mathcal C_Z(\overline f,34\alpha).
\label{eq:general-compressed-confidence-inclusions}
\end{equation}
Taking pairwise widths proves \eqref{eq:general-stable-width-sandwich}.
The preceding failure probabilities sum to at most $\delta_0/(16M)$.

\paragraph{Number of possible bonuses.}
For fixed $\alpha$, a bonus constructed from $Z^\sharp$ is determined
by the rounded center $f^\sharp$, at most $J_M$ rounded state-action
pairs $x_\ell^\sharp$, and their integer weights $1/p_\ell\le M$.
Including the constant failure output and the empty-data width gives
\begin{equation}
|\mathcal W(\mathcal F,M,\delta_0;\alpha)|
\leq
2+\mathcal N_\infty(\nu_\epsilon,\mathcal F)
\sum_{m=0}^{J_M}
\left[M\mathcal N(\nu_\epsilon,\mathcal X,d_{\mathcal F})\right]^m.
\label{eq:general-stable-family-entropy}
\end{equation}
Lemma~\ref{lem:general-eluder-controls-critic-cover} gives
$\log\mathcal N_\infty(\nu_\epsilon,\mathcal F)=\widetilde O(d_{\rm E})$.
Also $d_{\rm E}(\mathcal F,q_\epsilon)\leq d_{\rm E}$, so
$J_M=\widetilde O(d_{\rm E}^2)$. Taking logarithms in \eqref{eq:general-stable-family-entropy} proves
\eqref{eq:general-stable-compression-orders}. Its right-hand side is independent of
the numerical value of the fixed radius. The deterministic support and total-weight checks
ensure that the same family contains every output, including on the failure event.
\end{proof}

Using $B\vee1$ accommodates component ranges smaller than one without changing the stated
rates.  A zero-range component, in particular the terminal entropy-value function, is known exactly
and uses a zero bonus.

For a value function $V$, define the component Bellman target
\begin{equation}
f_{j,h}[V](s,a)
\triangleq
\begin{cases}
r_h(s,a)+P_hV(s,a), & j=r,\\
u_{i,h}(s,a)+P_hV(s,a), & j=u_i,\ i\in[I],\\
P_hV(s,a), & j=\psi.
\end{cases}
\label{eq:general-component-bellman-target}
\end{equation}
For $j=r,u_i,\psi$, respectively, let $V_j^{\pi^k}$ denote
$V_r^{\pi^k}$, $V_{u_i}^{\pi^k}$, and $V_{\psi_k}^{\pi^k}$.
Thus, for
$f_{j,k,h}^\star\triangleq f_{j,h}[V_{j,h+1}^k]$,
Assumption~\ref{ass:general-bellman-closeness} gives
$f_{j,k,h}^\star\in\mathcal F_{j,h}$, and the corresponding regression target lies in
$[0,B_{j,h}]$.

We now establish uniform least-squares confidence.  Uniformity over the value function is what
allows the backward recursion to reuse the same trajectories at every stage.

\begin{lemma}[Uniform adaptive-value function confidence]
\label{lem:general-square-loss-confidence}
Let $\mathcal V$ be a deterministic class of functions $\mathcal S\to[0,B]$, and suppose that for
every $V\in\mathcal V$ the Bellman target
$f_V(x)=y(x)+P V(x)$ belongs to
$\mathcal F\subseteq(\mathcal X\to[0,B])$.  Let
$(X_t,Y_t,S'_t)_{t=1}^M$ be a sequential dataset. Conditionally on the history
$\mathscr F_{t-1}$ and the possibly adaptive $X_t$, assume
$\mathbb E[Y_t\mid X_t,\mathscr F_{t-1}]=y(X_t)$ and
$S'_t\sim P(\cdot\mid X_t)$.  Assume also that
$Y_t+V(S'_t)\in[0,B]$ for every $V\in\mathcal V$.
For each $m\leq M$ and $V$, let
\[
\widehat f_{V,m}
\in
\argmin_{f\in\mathcal F}
\sum_{t=1}^m
\bigl(f(X_t)-Y_t-V(S'_t)\bigr)^2.
\]
Then, for every $\nu\in(0,1)$, with probability at least $1-\delta_0$,
simultaneously for all $m\leq M$ and $V\in\mathcal V$,
\begin{equation}
\|\widehat f_{V,m}-f_V\|_{Z_m}^2
\leq
C B^2
\log\!\left(
\frac{
M
\mathcal N_\infty(\nu,\mathcal F)
\mathcal N_\infty(\nu,\mathcal V)
}{\delta_0}
\right)
+CMB\nu,
\qquad
Z_m=\{X_1,\ldots,X_m\}.
\label{eq:general-uniform-regression-confidence}
\end{equation}
The same conclusion holds conditionally when the class $\mathcal V$ is fixed by a sigma-field
preceding the dataset.
\end{lemma}

\begin{proof}
For fixed $f\in\mathcal F$ and $V\in\mathcal V$, write
$\varepsilon_t(V)=Y_t+V(S'_t)-f_V(X_t)$.  Conditionally on
$\mathscr F_{t-1}$ and $X_t$, this is a mean-zero random variable bounded in magnitude by $B$.
The exponential-supermartingale bound for
$\varepsilon_t(V)(f(X_t)-f_V(X_t))$, followed by a union bound over sup-norm $\nu$-nets of
$\mathcal F$ and $\mathcal V$ and over $m\leq M$, gives
\[
2\sum_{t=1}^m
\varepsilon_t(V)(f(X_t)-f_V(X_t))
\leq
\frac12\|f-f_V\|_{Z_m}^2
+
CB^2\log\!\left(
\frac{M\mathcal N_\infty(\nu,\mathcal F)\mathcal N_\infty(\nu,\mathcal V)}{\delta_0}
\right).
\]
This event holds simultaneously for all net pairs.  Since
$\|f_V-f_{\widetilde V}\|_\infty\leq\|V-\widetilde V\|_\infty$, replacing either net center
by a function within distance $\nu$ changes the displayed loss comparison by at most
$CMB\nu$.  The resulting inequality therefore holds for every $(f,V)$.

By realizability, $f_V$ is an admissible competitor for the least-squares fit, so
\[
\|\widehat f_{V,m}-f_V\|_{Z_m}^2
\leq
2\sum_{t=1}^m\varepsilon_t(V)
\bigl(\widehat f_{V,m}(X_t)-f_V(X_t)\bigr).
\]
Substitute $f=\widehat f_{V,m}$ into the uniform inequality and move the half squared-norm term
to the left to obtain \eqref{eq:general-uniform-regression-confidence}.  The event is uniform
in $V$, so it also covers a value function chosen after observing the data. If $\mathcal V$ is
measurable with respect to a sigma-field preceding the dataset, condition on that sigma-field,
apply the uniform bound, and then integrate the conditional probability.
\end{proof}

\subsubsection{Value function and actor covers}
For a policy $\pi_h$, critic $f\in\mathcal F_{j,h}$, and stable bonus
$b\in\mathcal W(\mathcal F_{j,h},M,\delta_0;\alpha)$ for a fixed radius $\alpha$, the oracle forms
\begin{align*}
V_{j,h}^{\pi,f,b}(s)
&=
\sum_{a\in\cA}\pi_h(a\mid s)
\mathrm{Truncate}_{[0,B_{j,h}]}\bigl(f(s,a)+b(s,a)\bigr),
\qquad j\neq\psi,
\\
V_{\psi,h}^{\pi,f,b}(s)
&=
\operatorname{Ent}\bigl(\pi_h(\cdot\mid s)\bigr)
+
\sum_{a\in\cA}\pi_h(a\mid s)
\mathrm{Truncate}_{[0,B_{\psi,h}]}\bigl(f(s,a)+b(s,a)\bigr).
\end{align*}
Here $\operatorname{Ent}(p)=-\sum_a p(a)\log p(a)$.  Only the value function part of the entropy
critic is fitted; its immediate entropy contribution is evaluated exactly.
For later use, write
\begin{equation}
\mathcal W_{j,h}^{(M)}(\alpha)
\triangleq
\mathcal W(\mathcal F_{j,h},M,\delta_0;\alpha),
\qquad
\mathcal V_{j,h}(\pi_h;M,\delta_0,\alpha)
\triangleq
\left\{
V_{j,h}^{\pi,f,b}:f\in\mathcal F_{j,h},\ b\in\mathcal W_{j,h}^{(M)}(\alpha)
\right\}.
\label{eq:general-fixed-policy-continuation-class}
\end{equation}
At the terminal stage, set $\mathcal V_{j,H+1}(\cdot;M,\delta_0,\alpha)=\{0\}$.

\begin{lemma}[Stable value function cover]
\label{lem:general-continuation-cover}
Fix $\alpha\geq\alpha_{\mathcal F_{j,h}}(M,\delta_0)$. For a fixed policy $\pi_h$, the class
$\mathcal V_{j,h}(\pi_h;M,\delta_0,\alpha)$ has a $\nu$-cover whose
logarithmic size is at most
\begin{equation}
\log\mathcal N_\infty(\nu,\mathcal F_{j,h})
+
\log|\mathcal W(\mathcal F_{j,h},M,\delta_0;\alpha)|.
\label{eq:general-fixed-policy-continuation-cover}
\end{equation}
For a varying log-linear policy, combine a $\nu/2$-cover of the critic with a cover of its
logits at sup-norm resolution
$\nu'=\nu/[4(B_{j,h}+\log(e|\cA|))]$, allowing state-dependent additive shifts.
The logarithmic size of this logit cover is added to
\eqref{eq:general-fixed-policy-continuation-cover}, with the critic resolution replaced by
$\nu/2$.  This construction covers both ordinary and entropy value function.
\end{lemma}

\begin{proof}
Clipping and averaging under a fixed policy are nonexpansive in the sup norm.  Thus a critic
cover, paired with every member of the finite bonus family, proves the fixed-policy statement.
For varying policies, define
$\|\pi-\pi'\|_{1,\infty}=\sup_s\sum_a|\pi(a\mid s)-\pi'(a\mid s)|$.
For $Q,Q'\in[0,B_{j,h}]$,
\[
\left|
\sum_a\pi(a\mid s)Q(s,a)
-
\sum_a\pi'(a\mid s)Q'(s,a)
\right|
\leq
\|Q-Q'\|_\infty+B_{j,h}\|\pi-\pi'\|_{1,\infty}.
\]
If two logits differ by at most $\nu'$ after a state-dependent shift, interpolation between
them gives $\|\pi-\pi'\|_{1,\infty}\leq2\nu'$.  Along this interpolation, with logit
increment $\Delta g$ and induced policy $\pi_t$,
\[
\frac{d}{dt}\operatorname{Ent}(\pi_t)
=
-\operatorname{Cov}_{\pi_t}(\Delta g,\log\pi_t).
\]
The derivative is bounded in magnitude by
$2\|\Delta g\|_\infty\operatorname{Ent}(\pi_t)
\leq2\log|\cA|\,\|\Delta g\|_\infty$.
Integrating and adding the critic-cover error gives a total error of at most
$\nu/2+2(B_{j,h}+\log|\cA|)\nu'\leq\nu$.  Thus a logit cover controls the entropy term
without requiring a lower bound on the action probabilities.
\end{proof}

For off-policy data, the actor depends on the preceding data prefix. Recall the centered actor
function $f_{\omega,h}$, radius $L_{\rm act}$, and deterministic class
$\mathcal F_h^{\rm off}$ from
\eqref{eq:explicit-actor-function}-\eqref{eq:explicit-actor-function-class}, and define
\begin{equation}
\Pi_h^{\rm off}
\triangleq
\left\{
\pi_h(\cdot\mid\cdot,\omega):f_{\omega,h}\in\mathcal F_h^{\rm off}
\right\}.
\label{eq:general-off-policy-actor-class}
\end{equation}

\begin{lemma}[Off-policy log-linear actor cover]
\label{lem:general-reachable-actor-cover}
Run Algorithm~\ref{alg:master-rpgpd} with Algorithm~\ref{alg:general-ope},
Assumption~\ref{ass:master-actor-oracle}, and $0<\eta\tau\leq1$.  For every stage $h$, the
iterates initialized at $\omega_h^1=0$ satisfy $\pi_h^k\in\Pi_h^{\rm off}$ and
\begin{equation}
\sup_{s}\max_{a,a'}
\left|
\left\langle
\varphi(s,a)-\varphi(s,a'),\omega_h^k
\right\rangle
\right|
\leq
L_{\rm act}
=\frac{H(1+I\lambda_{\max}+\tau\log|\cA|)+\epsilon_{\rm act}}{\tau}.
\label{eq:general-reachable-logit-radius}
\end{equation}
Moreover, for every $\nu>0$,
\begin{equation}
\log\mathcal N\bigl(\nu,\Pi_h^{\rm off},\|\cdot\|_{1,\infty}\bigr)
\leq
d_{\rm a}
\log\!\left(
1+\frac{C L_{\rm act}}{\nu}
\right).
\label{eq:general-reachable-actor-cover}
\end{equation}
The cover is obtained from logits modulo state-dependent shifts, so the same covering bound
also supplies the logit cover required by Lemma~\ref{lem:general-continuation-cover}.
\end{lemma}

\begin{proof}
The actor and fitting rule are unchanged from Section~\ref{section:linear-cmdp-explicit-actor}.
Hence the residual bound in
\eqref{eq:actor-residual-action-difference} and the log-odds recursion
\eqref{eq:explicit-log-odds-radius} give
\eqref{eq:general-reachable-logit-radius}. Equations
\eqref{eq:explicit-actor-function-policy}-\eqref{eq:explicit-reachable-actor-cover} then give
$\pi_h^k\in\Pi_h^{\rm off}$ and the stated cover. The centered representation removes exactly
the state-dependent shifts that leave softmax policies unchanged.
\end{proof}

\begin{lemma}[Value-function-class complexity]
\label{lem:general-continuation-complexity}
At the fixed resolution $\nu_\epsilon/2$, the logarithmic covering numbers of the
fixed-policy and off-policy value function classes are bounded, respectively, by
\begin{equation}
\Gamma_{\rm on}
=\widetilde O(d_{\rm E}^2\mathfrak h_{\mathcal X}),
\qquad
\Gamma_{\rm off}
=\widetilde O(d_{\rm E}^2\mathfrak h_{\mathcal X}+d_{\rm a}).
\label{eq:general-continuation-complexity-orders}
\end{equation}
These bounds hold uniformly over admissible deterministic radii. Choosing the squared-loss
radius $C(B_{j,h}\vee1)^2\Gamma_{\rm on}+CNq_\epsilon^2$ or
$C(B_{j,h}\vee1)^2\Gamma_{\rm off}+CKq_\epsilon^2$, respectively, satisfies both the
stable-bonus threshold and the uniform regression bound in
\eqref{eq:general-uniform-regression-confidence}.
\end{lemma}

\begin{proof}
For $M=N,K$, let
$\overline w^{(M)}=\max_{j,h}\sup_{\alpha\geq\alpha_{\mathcal F_{j,h}}(M,\delta_0)}
\log|\mathcal W(\mathcal F_{j,h},M,\delta_0;\alpha)|$.
A zero-range component contributes zero to this maximum. Define
\begin{equation}
\begin{aligned}
\Gamma_{\rm on}
&=1+\log\frac{N}{\delta_0}
+2\max_{j,h}\log\mathcal N_\infty(\nu_\epsilon/4,\mathcal F_{j,h})
+\overline w^{(N)},
\\
\Gamma_{\rm off}
&=1+\log\frac{K}{\delta_0}
+2\max_{j,h}\log\mathcal N_\infty(\nu_\epsilon/4,\mathcal F_{j,h})
+\overline w^{(K)}
\\
&\qquad
+d_{\rm a}\log\!\left(
1+\frac{C L_{\rm act}(B+\log(e|\cA|))}{\nu_\epsilon}
\right).
\end{aligned}
\label{eq:general-proof-continuation-complexities}
\end{equation}
For a fixed admissible radius and policy, Lemma~\ref{lem:general-continuation-cover} gives
\[
\log\mathcal N_\infty\!\left(
\nu_\epsilon/2,\mathcal V_{j,h}(\pi_h;N,\delta_0,\alpha)
\right)
\leq
\log\mathcal N_\infty(\nu_\epsilon/2,\mathcal F_{j,h})+\overline w^{(N)}
\leq\Gamma_{\rm on}.
\]
For the union over $\pi_h\in\Pi_h^{\rm off}$, use critic resolution
$\nu_\epsilon/4$ and logit resolution
$\nu_\epsilon/[8(B_{j,h}+\log(e|\cA|))]$ in that lemma. Equation
\eqref{eq:general-reachable-actor-cover} bounds the added logit-cover term by the last term
of $\Gamma_{\rm off}$. The separate maxima over $(j,h)$ in
\eqref{eq:general-proof-continuation-complexities} also cover a stage-$h$ critic paired with
a stage-$(h+1)$ value function.

Apply Lemma~\ref{lem:general-square-loss-confidence} at resolution
$\nu=\nu_\epsilon/2$. Its logarithmic term is bounded by the corresponding $\Gamma$, and
the definition of $\nu_\epsilon$ gives
\[
CM B_{j,h}\nu_\epsilon\leq CM B\nu_\epsilon\leq CMq_\epsilon^2.
\]
The additive radius $CMq_\epsilon^2$ therefore covers both this approximation term and
$\alpha_{\mathcal F_{j,h}}(M,\delta_0)=Mq_\epsilon^2$. For a zero-range component, the
target, fit, and bonus are identically zero and no confidence radius is needed.
Finally, Lemma~\ref{lem:general-eluder-controls-critic-cover}, at resolution
$\nu_\epsilon/4$, bounds the critic entropy by $\widetilde O(d_{\rm E})$;
Lemma~\ref{lem:general-stable-bonus} gives
$\overline w^{(M)}=\widetilde O(d_{\rm E}^2\mathfrak h_{\mathcal X})$.
The actor term contributes $\widetilde O(d_{\rm a})$. This proves
\eqref{eq:general-continuation-complexity-orders}; the budget enters only through logarithms,
not through the scales at which the complexity measures are evaluated.
\end{proof}

\begin{lemma}[Common boundedness certificate]
\label{lem:general-common-boundedness}
Let $S\triangleq1+I\lambda_{\max}+\tau\log(e|\cA|)$. If
$0<\eta\tau\leq1/4$ and $\eta HS\leq1/4$, every output of
Algorithm~\ref{alg:general-ope} satisfies
\begin{equation}
0\leq V_{u_i}^k\leq H,
\qquad
\sum_{a\in\cA}\pi_h^k(a\mid s)e^{\eta Q_{z_k,h}^k(s,a)}
\bigl(Q_{z_k,h}^k(s,a)\bigr)^2
\leq C H^2S^2.
\label{eq:general-common-local-moment}
\end{equation}
\end{lemma}

\begin{proof}
Utility boundedness follows from truncation. For $p_a=\pi_h^k(a\mid s)$, the component
ranges give
\begin{equation*}
Q_{z_k,h}^k(s,a)
=q_{k,h}(s,a)+\tau\log(1/p_a),
\qquad
0\leq q_{k,h}(s,a)\leq HS.
\end{equation*}
Set $\alpha=\eta\tau$ and $L=\log(e|\cA|)$. Then
\begin{align*}
&\sum_a p_a e^{\eta Q_{z_k,h}^k(s,a)}
\bigl(Q_{z_k,h}^k(s,a)\bigr)^2
\\
&\qquad\leq
2e^{\eta HS}
\left[
H^2S^2\sum_a p_a^{1-\alpha}
+\tau^2\sum_a p_a^{1-\alpha}\log^2(1/p_a)
\right].
\end{align*}
Concavity gives $\sum_a p_a^{1-\alpha}\leq|\cA|^\alpha$. Moreover,
$x^2\leq C L^2e^{x/(2L)}$ for $x\geq0$ implies
\begin{equation*}
\sum_a p_a^{1-\alpha}\log^2(1/p_a)
\leq
C L^2|\cA|^{\alpha+1/(2L)}.
\end{equation*}
The two stepsize conditions give
$\eta HS\leq1/4$ and $\alpha\log|\cA|\leq1/4$, while
$|\cA|^{1/(2L)}\leq e^{1/2}$ and $\tau L\leq S$. Substitution proves
\eqref{eq:general-common-local-moment}.
\end{proof}

\subsection{Eluder-width bounds}
\label{app:general-eluder-widths}
We now bound the size of the bonus under the current policy, which
determines the cost of optimism in the oracle guarantee. For a fresh
on-policy batch, we control
$\mathbb E_{x\sim d_h^{\pi^k}}[b_{j,h}^k(x)]$
using an eluder-width bound. With off-policy data, the bonus depends
on the newly observed trajectory; we upper-bound it by a width
determined from the preceding data $Z_{k-1}$, then control the sum of
its conditional expectations.

\begin{lemma}[Population eluder width]
\label{lem:general-population-eluder-width}
Let $\mathcal F\subseteq(\mathcal X\to[0,B])$, let $X_1,\ldots,X_N$ be i.i.d. from $\nu$, fix
$\gamma\geq0$, and set
\[
w_N(x)
=
\sup\left\{
|f(x)-f'(x)|:
f,f'\in\mathcal F,\ 
\sum_{n=1}^N(f(X_n)-f'(X_n))^2\leq\gamma
\right\}.
\]
For every fixed $\rho>0$, with probability at least $1-\delta_0$,
\begin{equation}
\mathbb E_{X\sim\nu}[w_N(X)]
\leq
\rho+\frac{B d}N
+4\sqrt{\frac{\gamma d}N}
+B\sqrt{\frac{2\log(1/\delta_0)}N},
\qquad
d=d_{\rm E}(\mathcal F,\rho).
\label{eq:general-population-eluder-width}
\end{equation}
\end{lemma}

\begin{proof}
Set $Z_{n-1}=\{X_1,\ldots,X_{n-1}\}$ and define $w_{n-1}$ by the same pairwise constraint
on $Z_{n-1}$. We adapt the counting argument of
\citet[Proposition~3 and Lemma~2]{russo2013eluder} to the fixed cutoff $\rho$.
For $t\geq\rho$, greedily partition the points with $w_{n-1}(X_n)>t$ into
$t$-independent subsequences. At each such point there is a pair $(f,f')$ with
$|f(X_n)-f'(X_n)|>t$ and $\|f-f'\|_{Z_{n-1}}^2\leq\gamma$.
If this point were $t$-dependent on $m$ disjoint preceding subsequences, the same pair
would satisfy
\[
\gamma\geq\|f-f'\|_{Z_{n-1}}^2
>mt^2.
\]
Thus $\lfloor\gamma/t^2\rfloor+1$ subsequences suffice, and each contains at most
$d_{\rm E}(\mathcal F,t)\leq d$ points. Consequently,
\begin{equation}
\sum_{n=1}^N\mathbf 1\{w_{n-1}(X_n)>t\}
\leq d\left(1+\frac{\gamma}{t^2}\right),
\qquad t\geq\rho.
\label{eq:general-fixed-cutoff-width-count}
\end{equation}
If $\rho\geq B$, the desired width bound is immediate. Otherwise, integrating the count gives
\begin{align*}
\sum_{n=1}^N w_{n-1}(X_n)
&\leq N\rho+\int_\rho^B
\min\left\{N,d+\frac{\gamma d}{t^2}\right\}\,dt
\\
&\leq N\rho+Bd+
\int_0^\infty\min\left\{N,\frac{\gamma d}{t^2}\right\}\,dt
\leq N\rho+Bd+2\sqrt{\gamma dN}.
\end{align*}
Adding observations only shrinks the feasible pair set, so
$w_N\leq w_{n-1}$ pointwise.  Independence therefore gives, for every
realized batch,
\[
N\mathbb E_{X\sim\nu}[w_N(X)]
\leq
\sum_{n=1}^N
\mathbb E[w_{n-1}(X_n)\mid X_1,\ldots,X_{n-1}].
\]
Each conditional expectation differs from $w_{n-1}(X_n)$ by a martingale
difference bounded in magnitude by $B$.  Azuma-Hoeffding and the preceding
width sum prove \eqref{eq:general-population-eluder-width}.
\end{proof}

\begin{lemma}[Sequential predictable width]
\label{lem:general-sequential-width}
Let $\mathcal F\subseteq(\mathcal X\to[0,B])$, let $X_k$ be generated
sequentially, and let $\mathscr F_{k-1}$ contain the history before $X_k$
is drawn.  Suppose $b_k$ is $\mathscr F_{k-1}$-measurable,
$0\leq b_k\leq B$, and
\[
b_k(x)
\leq
\sup\left\{
|f(x)-f'(x)|:
f,f'\in\mathcal F,\quad
\sum_{\ell<k}(f(X_\ell)-f'(X_\ell))^2\leq\gamma
\right\}.
\]
For every fixed $\rho>0$, with probability at least $1-\delta_0$,
\begin{equation}
\sum_{k=1}^{K-1}
\mathbb E[b_k(X_k)\mid\mathscr F_{k-1}]
\leq
C\left[
K\rho+B d
+\sqrt{\gamma dK}
+B\sqrt{K\log(1/\delta_0)}
\right],
\quad
d=d_{\rm E}(\mathcal F,\rho).
\label{eq:general-sequential-width}
\end{equation}
\end{lemma}

\begin{proof}
The deterministic count \eqref{eq:general-fixed-cutoff-width-count} applies to the
pairwise widths dominating $b_k$, regardless of how the points $X_k$ are selected.
Integrating it as in the preceding proof gives
\[
\sum_{k=1}^{K-1}b_k(X_k)
\leq
K\rho+Bd+2\sqrt{\gamma dK}.
\]
By predictability,
$\mathbb E[b_k(X_k)\mid\mathscr F_{k-1}]-b_k(X_k)$ is a martingale difference bounded by $B$.
Azuma-Hoeffding proves \eqref{eq:general-sequential-width}.  This
conditional-expectation bound is the form needed by the off-policy oracle.
\end{proof}

\begin{lemma}[Predictable envelope for a post-data width]
\label{lem:general-post-data-width-envelope}
Let $Z_k=Z_{k-1}\cup\{X_k\}$ be multisets in $\mathcal X$, let
$\overline f_k\in\mathcal F$, and let $\gamma\geq0$.  Then, for every $x\in\mathcal X$,
\begin{equation}
w\bigl(\mathcal C_{Z_k}(\overline f_k,\gamma),x\bigr)
\leq
\sup_{\substack{f,f'\in\mathcal F\\
\|f-f'\|_{Z_{k-1}}^2\leq4\gamma}}
|f(x)-f'(x)|.
\label{eq:general-post-data-width-envelope}
\end{equation}
In particular, if $Z_{k-1}$ is known before $X_k$ is drawn and $\gamma$ is deterministic, the
right-hand side is predictable even when $\overline f_k$ depends on $X_k$.
\end{lemma}

\begin{proof}
For any $f,f'\in\mathcal C_{Z_k}(\overline f_k,\gamma)$, the triangle inequality for the empirical
seminorm gives
\[
\|f-f'\|_{Z_k}
\leq
\|f-\overline f_k\|_{Z_k}
+\|f'-\overline f_k\|_{Z_k}
\leq2\sqrt{\gamma}.
\]
Restricting the empirical norm from $Z_k$ to $Z_{k-1}$ can only decrease it, so the pair
$(f,f')$ is feasible on the right-hand side of
\eqref{eq:general-post-data-width-envelope}.  Taking the supremum proves the claim.
\end{proof}

\subsection{On-policy oracle certificate and tuning}
\label{app:general-on-policy}
Throughout these certificates, take $\lambda_{\max}\geq1$, as satisfied by
the prescribed choice $2H\log(e|\cA|)/\xi$. Write
$S=1+I\lambda_{\max}+\tau\log(e|\cA|)$ and use the fixed scales
$r_\epsilon,q_\epsilon,\nu_\epsilon$ of
\eqref{eq:general-accuracy-resolutions}.

\begin{proposition}[Non-split general on-policy oracle]
\label{prop:general-on-policy-ope}
Under Assumption~\ref{ass:general-bellman-closeness}, choose the radii and stable-bonus budgets as
below. If $N\geq d_0d_{\rm E}^2\mathfrak h_{\mathcal X}$,
$0<\tau\leq1$, $0<\eta\tau\leq1/4$, and
$\eta H(1+I\lambda_{\max}+\tau\log(e|\cA|))\leq1/4$, then, with probability at least
$1-\delta$, Algorithm~\ref{alg:general-ope} satisfies Assumption~\ref{ass:master-ope} and
\begin{equation}
\begin{aligned}
\sup_{k<K}\beta_k
&\leq
\widetilde O\left(
SH^2
\sqrt{\frac{d_{\rm E}^3\mathfrak h_{\mathcal X}}{N}}
\right)+CSHr_\epsilon,
\\
\sup_{k<K}C_{\eta,\tau,\Lambda,k}
&=
\widetilde O\left(
H^2\bigl(1+I\lambda_{\max}+\tau\log|\cA|\bigr)^2
\right).
\end{aligned}
\label{eq:general-on-policy-certificate}
\end{equation}
\end{proposition}

Here and below, $\widetilde O$ hides universal powers of the confidence, contraction, and cover-scale
logarithms.  In particular, all polynomial dependence on $d_{\rm E}$,
$\mathfrak h_{\mathcal X}$, and $d_{\rm a}$ is displayed.  Fix $\epsilon,\delta\in(0,1)$ and choose
\begin{equation}
\begin{aligned}
\lambda_{\max}
&=\frac{2H\log(e|\cA|)}{\xi},
\\
\tau
&=\Theta\left(
\min\left\{1,
\max\left\{
\frac{\xi\epsilon}{H\log(e|\cA|)},
\frac{H^{2/3}\xi^{2/3}\epsilon_{\rm act}^{1/3}}
{\log^{2/3}(e|\cA|)}
\right\}
\right\}
\right),
\\
\eta
&=\widetilde\Theta\left(\frac{\tau\xi^2\epsilon^2}{I^2H^8}\right),
\qquad
K=\widetilde\Theta\left(\frac{I^2H^8}{\tau^2\xi^2\epsilon^2}\right),
\\
N
&=\widetilde\Theta\left(
\frac{d_{\rm E}^3\mathfrak h_{\mathcal X}I^4H^{14}}
{\tau^2\xi^4\epsilon^4}
\right).
\end{aligned}
\label{eq:general-on-policy-parameters}
\end{equation}

These choices yield the on-policy part of Theorem~\ref{thm:general-last-iterate}.

\begin{proof}[Proof of Proposition~\ref{prop:general-on-policy-ope}]
We verify optimism, consistency, and boundedness in that order.  Allocate
$\delta_0=\delta/[C K H(I+2)]$ to each component, stage, and iteration,
set $\delta_{j,h}=\delta_0$, take $M=N$, and choose
\begin{equation}
\alpha_{j,h}
=C\left[(B_{j,h}\vee1)^2\Gamma_{\rm on}+Nq_\epsilon^2\right]
\qquad B_{j,h}>0,
\label{eq:general-on-policy-radii-proof}
\end{equation}
where $\Gamma_{\rm on}$ is the value-function-cover complexity in
Lemma~\ref{lem:general-continuation-complexity}. A sufficiently large universal $C$ makes these
radii satisfy both the stable-bonus and regression-confidence requirements.
For $B_{j,h}=0$, set $\alpha_{j,h}=b_{j,h}^k=0$.

For $h\in[H]$, define the conditional value function class
\begin{equation*}
\mathcal V_{k,j,h}^{\rm on}
\triangleq
\mathcal V_{j,h}(\pi_h^k;N,\delta_0,\alpha_{j,h}),
\qquad
\mathcal V_{k,j,H+1}^{\rm on}\triangleq\{0\}.
\end{equation*}
By construction of Algorithm~\ref{alg:general-ope},
$V_{j,h}^k\in\mathcal V_{k,j,h}^{\rm on}$ for every $h\in[H+1]$.

Fix $k$ and condition on the history before its batch is collected.  The
policy $\pi^k$ is fixed and the $N$ trajectories are independent, but the
fitted value function $V_{j,h+1}^k$ still depends on this batch. Let
$Z_h^k\triangleq\{(s_h^{k,n},a_h^{k,n}):n\in[N]\}$.
Applying
Lemma~\ref{lem:general-square-loss-confidence} conditionally to
$\mathcal V_{k,j,h+1}^{\rm on}$ at resolution $\nu_\epsilon/2$
and using Lemma~\ref{lem:general-continuation-complexity}, backward induction gives,
simultaneously for every component and stage,
\begin{equation}
\left\|
\widehat f_{j,h}^k-f_{j,k,h}^\star
\right\|_{Z_h^k}^2
\leq
\alpha_{j,h}.
\label{eq:general-target-in-confidence-ball}
\end{equation}
The uniform value function cover is what permits every stage to use the same
batch without sample splitting.

On the joint confidence and stable-bonus event, both
$\widehat f_{j,h}^k$ and $f_{j,k,h}^\star$ belong to the inner full-data
confidence ball.  The lower width sandwich gives
$|\widehat f_{j,h}^k-f_{j,k,h}^\star|\leq b_{j,h}^k$ pointwise.
Truncation preserves this optimism because the target lies in the
truncation interval:
\begin{equation}
0
\leq
\overline Q_{j,h}^k(s,a)-f_{j,k,h}^\star(s,a)
\leq
2b_{j,h}^k(s,a).
\label{eq:general-component-optimism}
\end{equation}
For the entropy component, the known immediate cost $\psi_{k,h}$ is
included exactly.  Scalarizing with the nonnegative weights
$(1,\lambda_k,\tau)$ therefore proves
Assumption~\ref{ass:master-ope}(i).

For consistency, the upper width sandwich and the empirical-norm triangle inequality give
\begin{equation}
b_{j,h}^k(x)
\leq
\sup_{\substack{f,f'\in\mathcal F_{j,h}\\
\|f-f'\|_{Z_h^k}^2
\leq136\alpha_{j,h}}}
|f(x)-f'(x)|.
\label{eq:general-on-policy-pairwise-width}
\end{equation}
Applying Lemma~\ref{lem:general-population-eluder-width} at cutoff
$r_\epsilon$ conditionally under the occupancy $d_h^{\pi^k}$ yields, for reward and utility components,
\begin{equation*}
\mathbb E_{d_h^{\pi^k}}[b_{j,h}^k]
\leq
C H\left[
\sqrt{\frac{d_0\Gamma_{\rm on}}N}
+\frac{d_0}N
+\sqrt{\frac{\log(1/\delta_0)}N}
\right]+C r_\epsilon.
\end{equation*}
The additional radius contributes $Cq_\epsilon\sqrt{d_0}=Cr_\epsilon$;
also $d_0\leq d_{\rm E}$.
The entropy-value function bound can be enlarged by $\log(e|\cA|)$;
at $h=H$ this value function and its bonus are zero.
Telescoping the Bellman residuals in
\eqref{eq:general-component-optimism} now gives
\[
0\leq V_{j,1}^k(s_1)-V_j^{\pi^k}
\leq
2\sum_{h=1}^H
\mathbb E_{d_h^{\pi^k}}[b_{j,h}^k],
\]
where $V_j^{\pi^k}$ denotes $V_{\psi_k}^{\pi^k}$ for $j=\psi$.
Consequently, a common tolerance for both consistency requirements is
\[
\beta_k
=
2\sum_{h=1}^H\left[
\mathbb E_{d_h^{\pi^k}}b_{r,h}^k
+\max_i\mathbb E_{d_h^{\pi^k}}b_{u_i,h}^k
+\sum_{i=1}^I\lambda_{k,i}\mathbb E_{d_h^{\pi^k}}b_{u_i,h}^k
+\tau\mathbb E_{d_h^{\pi^k}}b_{\psi,h}^k
\right].
\]
Since \eqref{eq:general-continuation-complexity-orders} gives
$\Gamma_{\rm on}=\widetilde O(d_{\rm E}^2\mathfrak h_{\mathcal X})$ and
$N\geq d_0d_{\rm E}^2\mathfrak h_{\mathcal X}$,
\[
\beta_k
\leq
\widetilde O\left(
SH^2\sqrt{\frac{d_0d_{\rm E}^2\mathfrak h_{\mathcal X}}{N}}
\right)+CSHr_\epsilon.
\]
Finally, Lemma~\ref{lem:general-common-boundedness} gives the utility and local-moment bounds.
Since $d_0\leq d_{\rm E}$ and $0<\tau\leq1$,
$1+I\lambda_{\max}+\tau\log(e|\cA|)
\leq2(1+I\lambda_{\max}+\tau\log|\cA|)$, so this bound has the form stated in
\eqref{eq:general-on-policy-certificate}. A union bound over the allocated events completes the
oracle certificate.
\end{proof}

\subsection{Off-policy oracle certificate and tuning}
\label{app:general-replay}

\begin{proposition}[Non-split general off-policy oracle]
\label{prop:general-replay-ope}
Under Assumptions~\ref{ass:master-actor-oracle} and
\ref{ass:general-bellman-closeness}, choose the radii and stable-bonus budgets as below.
If $0<\tau\leq1$, $0<\eta\tau\leq1/4$, and
$\eta H(1+I\lambda_{\max}+\tau\log(e|\cA|))\leq1/4$, then, with probability at least
$1-\delta$, Algorithm~\ref{alg:general-ope} on the cumulative off-policy dataset $\cD^k$ satisfies
Assumption~\ref{ass:master-ope} and
\begin{equation}
\begin{aligned}
\sum_{k=1}^{K-1}\beta_k
&\leq
\widetilde O\left(
SH^2
\left[
\sqrt{d_{\rm E}(d_{\rm E}^2\mathfrak h_{\mathcal X}+d_{\rm a})K}
+d_{\rm E}(d_{\rm E}^2\mathfrak h_{\mathcal X}+d_{\rm a})
\right]
\right)+CSHKr_\epsilon,
\\
\sup_{k<K}C_{\eta,\tau,\Lambda,k}
&=
\widetilde O\left(
H^2\bigl(1+I\lambda_{\max}+\tau\log|\cA|\bigr)^2
\right).
\end{aligned}
\label{eq:general-replay-certificate}
\end{equation}
\end{proposition}

Choose $\lambda_{\max}$ and $\tau$ as in
\eqref{eq:general-on-policy-parameters}, and set
\begin{equation}
\eta
=
\widetilde\Theta\left(
\frac{\tau\xi^4\epsilon^4}
{d_{\rm E}(d_{\rm E}^2\mathfrak h_{\mathcal X}+d_{\rm a})I^4H^{14}}
\right),
\qquad
K
=
\widetilde\Theta\left(\frac1{\eta\tau}\right).
\label{eq:general-replay-parameters}
\end{equation}

These choices yield the off-policy part of Theorem~\ref{thm:general-last-iterate}.
The proof below dominates the additive complexity term in
\eqref{eq:general-replay-certificate} and bounds the contribution of $Kr_\epsilon$ separately.

\begin{proof}[Proof of Proposition~\ref{prop:general-replay-ope}]
The confidence argument is uniform over both value function and the off-policy actor class; the
width argument must additionally account for the current
sample in the fitted bonus.  Take $M=K$, allocate
$\delta_0=\delta/[C K H(I+2)]$, set $\delta_{j,h}=\delta_0$, and choose
\[
\alpha_{j,h}
=C\left[(B_{j,h}\vee1)^2\Gamma_{\rm off}+Kq_\epsilon^2\right]
\qquad B_{j,h}>0,
\]
where $\Gamma_{\rm off}$ is defined in
Lemma~\ref{lem:general-continuation-complexity}; set $\alpha_{j,h}=b_{j,h}^k=0$ when
$B_{j,h}=0$. For $h\in[H]$, define the deterministic class
\begin{equation*}
\mathcal V_{j,h}^{\rm off}
\triangleq
\bigcup_{\pi_h\in\Pi_h^{\rm off}}
\mathcal V_{j,h}(\pi_h;K,\delta_0,\alpha_{j,h}),
\qquad
\mathcal V_{j,H+1}^{\rm off}\triangleq\{0\}.
\end{equation*}
Algorithm~\ref{alg:general-ope} and
Lemma~\ref{lem:general-reachable-actor-cover} give
$V_{j,h}^k\in\mathcal V_{j,h}^{\rm off}$ for every $h\in[H+1]$.
For $Z_h^k\triangleq\{(s_h^\ell,a_h^\ell):\ell\in[k]\}$, applying
Lemma~\ref{lem:general-square-loss-confidence} at resolution $\nu_\epsilon/2$ over this deterministic
class and all dataset prefixes gives, simultaneously for every $(k,h,j)$,
\begin{equation}
\|\widehat f_{j,h}^k-f_{j,k,h}^\star\|_{Z_h^k}^2
\leq\alpha_{j,h}.
\label{eq:general-off-policy-target-in-confidence-ball}
\end{equation}
Uniformity permits substituting the value function fitted on $\cD^k$, including the current
trajectory; the actor $\pi^k$ depends only on $\cD^{k-1}$. On the joint confidence
and stable-bonus event, the lower width sandwich and truncation give
\[
0\leq
\overline Q_{j,h}^k-f_{j,k,h}^\star
\leq2b_{j,h}^k.
\]
After adding the immediate entropy term exactly, scalarization proves
optimism. Bellman telescoping also gives
\[
0\leq V_{j,1}^k(s_1)-V_j^{\pi^k}
\leq2\sum_{h=1}^H\mathbb E_{d_h^{\pi^k}}[b_{j,h}^k].
\]
Accordingly, use the consistency certificate
\begin{equation}
\beta_k
=2\sum_{h=1}^H\left[
\mathbb E_{d_h^{\pi^k}}b_{r,h}^k
+\max_i\mathbb E_{d_h^{\pi^k}}b_{u_i,h}^k
+\sum_{i=1}^I\lambda_{k,i}\mathbb E_{d_h^{\pi^k}}b_{u_i,h}^k
+\tau\mathbb E_{d_h^{\pi^k}}b_{\psi,h}^k
\right].
\label{eq:general-replay-beta-definition}
\end{equation}

Let $\mathscr F_{k-1}$ contain the history before trajectory $k$ is
collected.  The stable-width sandwich gives
\[
b_{j,h}^k(x)
\leq
w\bigl(
\mathcal C_{Z_h^k}
(\widehat f_{j,h}^k,34\alpha_{j,h}),x
\bigr).
\]
Removing the current sample from the empirical norm and applying
Lemma~\ref{lem:general-post-data-width-envelope} yields the predictable
envelope
\begin{equation}
b_{j,h}^k(x)
\leq
\widetilde b_{j,h}^k(x)
\triangleq
\sup_{\substack{f,f'\in\mathcal F_{j,h}\\
\|f-f'\|_{Z_h^{k-1}}^2\leq136\alpha_{j,h}}}
|f(x)-f'(x)|.
\label{eq:general-replay-predictable-envelope}
\end{equation}
Both $\widetilde b_{j,h}^k$ and $\pi^k$ are
$\mathscr F_{k-1}$-measurable.  Consequently,
\[
\mathbb E_{d_h^{\pi^k}}[b_{j,h}^k]
\leq
\mathbb E_{d_h^{\pi^k}}[\widetilde b_{j,h}^k]
=
\mathbb E\!\left[
\widetilde b_{j,h}^k(s_h^k,a_h^k)\mid\mathscr F_{k-1}
\right].
\]
The first inequality holds for each realized fitted bonus.  Only the
predictable envelope is evaluated on the actual current trajectory in the
last equality, which avoids conditioning incorrectly on a bonus fitted
using that trajectory.

The envelope radius is $O(H^2\Gamma_{\rm off}+Kq_\epsilon^2)$ for reward and utility
components. Lemma~\ref{lem:general-sequential-width} at cutoff $r_\epsilon$, summed over stages,
therefore gives
\[
\sum_{k=1}^{K-1}\sum_{h=1}^H
\mathbb E_{d_h^{\pi^k}}[b_{j,h}^k]
\leq\widetilde O\left(
H^2\left[
\sqrt{d_0\Gamma_{\rm off}K}
+d_0
\right]
\right)+CHKr_\epsilon
\]
for reward and utility components, with one additional factor
$\log(e|\cA|)$ for entropy. Here the additional radius contributes
$CHKq_\epsilon\sqrt{d_0}=CHKr_\epsilon$.
The additive complexity term retains the initial
unresolved widths; it will be dominated only after choosing $K$.
Substituting these bounds into
\eqref{eq:general-replay-beta-definition}, and using
$\Gamma_{\rm off}
=\widetilde O(d_{\rm E}^2\mathfrak h_{\mathcal X}+d_{\rm a})$ from
\eqref{eq:general-continuation-complexity-orders} proves the first line of
\eqref{eq:general-replay-certificate}. Lemma~\ref{lem:general-common-boundedness}, together with
$0<\tau\leq1$, establishes boundedness in the stated form. A union bound over the confidence,
stable-bonus, and
sequential-width events completes the proof.
\end{proof}

\subsection{Proofs of the full last-iterate guarantees}
\label{app:general-last-iterate-proofs}

We make the logarithmic tuning explicit to avoid a sample-budget fixed point.
After choosing $\tau$ and the scales in
\eqref{eq:general-accuracy-resolutions}, set
\[
\iota_\epsilon
=\log\!\left(
e+\frac{|\cA|H(I+2)(1+d_{\rm E}+\mathfrak h_{\mathcal X}+d_{\rm a})(1+L_{\rm act})}
{\delta\xi\tau\epsilon\nu_\epsilon}
\right).
\]
All quantities in this logarithm are independent of $N$ and $K$.
For the choices below, $\log(NK)=O(\iota_\epsilon)$ and
$\log(1/\delta_0)=O(\iota_\epsilon)$, with constants depending only on the universal integer $p$.
The explicit covering and confidence bounds above involve only fixed powers of these
logarithms. Choose a universal integer $p$ larger than their total degree plus six.
Then powers of $\iota_\epsilon$ suffice for all suppressed logarithmic factors,
without evaluating a complexity measure at a budget-dependent scale.

\begin{proof}[Proof of the on-policy part of Theorem~\ref{thm:general-last-iterate}]
We first reduce the potential to its actor-approximation floor, then
choose the temperature and count trajectories.  Uniform initialization,
$\lambda_1=0$, and \eqref{eq:master-regularised-dual-radius} give
\[
\Phi_1
\leq
H\log|\cA|
+\frac{H^2(1+\log|\cA|)^2}{2\xi^2}.
\]
The geometric weights in Theorem~\ref{thm:master-recursion} sum to at
most $1/(\eta\tau)$.  Substituting
Proposition~\ref{prop:general-on-policy-ope} and
$\lambda_{\max}=2H\log(e|\cA|)/\xi$ therefore yields
\begin{equation}
\begin{aligned}
\Phi_K
\leq{}&
e^{-\eta\tau(K-1)}\Phi_1
+\frac{H\epsilon_{\rm act}}{\tau}
\\
&+\widetilde O\left(
\frac{\eta I^2H^5}{\tau\xi^2}
+\frac{I^2H^4}{\tau\xi^2}
\sqrt{\frac{d_{\rm E}^3\mathfrak h_{\mathcal X}}{N}}
\right)
\\
&+C\frac{I^2H^3L^2r_\epsilon}{\tau\xi^2}.
\end{aligned}
\label{eq:general-on-policy-unrolled-proof}
\end{equation}
The two terms in $\widetilde O$ come from the local second moment
(including $G_\tau^2$) and the statistical critic error. An explicit realization of
\eqref{eq:general-on-policy-parameters} is
\[
\begin{aligned}
\eta&=\frac{c\tau\xi^2\epsilon^2}{I^2H^8\iota_\epsilon^p},
&
K&=1+\left\lceil\frac{C L_\epsilon}{\eta\tau}\right\rceil,
\\
N&=\left\lceil
\frac{C d_{\rm E}^3\mathfrak h_{\mathcal X}I^4H^{14}\iota_\epsilon^p}
{\tau^2\xi^4\epsilon^4}
\right\rceil.
\end{aligned}
\]
Here $c$ is sufficiently small and $C$ sufficiently large. These choices satisfy the
stepsize conditions and make both statistical terms $O(\epsilon^2/H^3)$.
Since $L_\epsilon\geq\log(1+H^3\Phi_1/\epsilon^2)$, they also control the initial term.
The retained approximation error satisfies
\[
C\frac{I^2H^3L^2r_\epsilon}{\tau\xi^2}
\leq Cc_0\frac{\epsilon^2}{H^3L_\epsilon}
\leq Cc_0\frac{\epsilon^2}{H^3}
\]
by the choice of $r_\epsilon$. Consequently,
\begin{equation}
\Phi_K
\leq
\frac{C\epsilon^2}{H^3}
+\frac{H\epsilon_{\rm act}}{\tau}.
\label{eq:general-on-policy-potential-proof}
\end{equation}
Since the right-hand sides of the conversion bounds in
Theorem~\ref{thm:master-recursion} are nonnegative, applying those bounds and
$\sqrt{x+y}\leq\sqrt{x}+\sqrt{y}$ gives, in the normalized Slater
regime specified above,
\[
\max\left\{
[V_r^{\pi^\star}-V_r^{\pi^K}]_+,\,
\max_{i\in[I]}[c_i-V_{u_i}^{\pi^K}]_+
\right\}
\leq
C\left[
\epsilon+\frac{H\log(e|\cA|)}{\xi}\tau
+H^2\epsilon_{\rm act}^{1/2}\tau^{-1/2}
\right].
\]
The regularisation term is $O(\epsilon)$ at temperature
$\xi\epsilon/[H\log(e|\cA|)]$.  Balancing it with the actor term
instead gives temperature
$H^{2/3}\xi^{2/3}\epsilon_{\rm act}^{1/3}/\log^{2/3}(e|\cA|)$.
Taking the larger of these two temperatures, capped at one, gives
\eqref{eq:general-on-policy-parameters}.  When the cap is inactive,
the resulting bound is
\eqref{eq:general-final-error}.  In particular, the cap is inactive
throughout the nontrivial regime
$\epsilon_{\rm act}\leq\xi/H^2$ specified above.  Outside this
regime the stated approximation floor is already at least $H$, the
largest possible reward or constraint gap.

Each oracle call uses $N$ trajectories shared by all stages.  Multiplying
the batch size by the number of calls gives
\[
NK
=\widetilde O\left(
\frac{d_{\rm E}^3\mathfrak h_{\mathcal X}I^6H^{22}}
{\xi^6\epsilon^6\tau^4}
\right).
\]
To keep the temperature cap explicit, its reciprocal satisfies
\[
\frac1{\tau^4}
=
\widetilde O\left(
\max\left\{1,\min\left\{
\frac{H^4}{\xi^4\epsilon^4},
\frac1{H^{8/3}\xi^{8/3}
\epsilon_{\rm act}^{4/3}}
\right\}\right\}
\right),
\]
where the second entry of the minimum is interpreted as $+\infty$ when
$\epsilon_{\rm act}=0$. Substitution into the preceding display yields
\eqref{eq:general-on-policy-sample-complexity}. When the cap is inactive,
the outer maximum can be removed. Each trajectory contains $H$ transitions,
so the transition count is $HNK$.
\end{proof}

\begin{proof}[Proof of the off-policy part of Theorem~\ref{thm:general-last-iterate}]
Use $\lambda_{\max}=2H\log(e|\cA|)/\xi$. Uniform policy initialization,
$\lambda_1=0$, and \eqref{eq:master-regularised-dual-radius} give
\[
\Phi_1
\leq
H\log|\cA|+\frac{H^2(1+\log|\cA|)^2}{2\xi^2}.
\]
For the critic term, bound the geometric weights by one and apply the
cumulative certificate of Proposition~\ref{prop:general-replay-ope}.
For the second-moment and actor terms, retain the geometric sum
$1/(\eta\tau)$.  The resulting master bound is
\begin{equation}
\begin{aligned}
\Phi_K
\leq{}&
e^{-\eta\tau(K-1)}\Phi_1
+\frac{H\epsilon_{\rm act}}{\tau}
 +\widetilde O\left(\frac{\eta I^2H^5}{\tau\xi^2}\right)
\\
&+\widetilde O\left(
\frac{\eta I^2H^4}{\xi^2}
\sqrt{d_{\rm E}(d_{\rm E}^2\mathfrak h_{\mathcal X}+d_{\rm a})K}
\right)
\\
&+\widetilde O\left(
\frac{\eta I^2H^4}{\xi^2}
d_{\rm E}(d_{\rm E}^2\mathfrak h_{\mathcal X}+d_{\rm a})
\right)
\\
&+C\frac{\eta K I^2H^3L^2}{\xi^2}r_\epsilon.
\end{aligned}
\label{eq:general-replay-unrolled-proof}
\end{equation}
An explicit realization of \eqref{eq:general-replay-parameters} is
\[
\eta
=\frac{c\tau\xi^4\epsilon^4}
{d_{\rm E}(d_{\rm E}^2\mathfrak h_{\mathcal X}+d_{\rm a})I^4H^{14}\iota_\epsilon^p},
\qquad
K=1+\left\lceil\frac{C L_\epsilon}{\eta\tau}\right\rceil.
\]
These choices ensure
$K\geq d_{\rm E}(d_{\rm E}^2\mathfrak h_{\mathcal X}+d_{\rm a})$,
so the additive width term is dominated by the square-root term.
Substituting $K=\widetilde O(1/(\eta\tau))$, the latter contributes
\[
\widetilde O\left(
\frac{I^2H^4}{\xi^2}
\sqrt{\frac{\eta d_{\rm E}
(d_{\rm E}^2\mathfrak h_{\mathcal X}+d_{\rm a})}{\tau}}
\right)
=O(\epsilon^2/H^3),
\]
with the logarithmic factors included in the stepsize choice.
The second-moment term is also $O(\epsilon^2/H^3)$ under this
stepsize. The last term in \eqref{eq:general-replay-unrolled-proof} satisfies
\[
C\frac{\eta K I^2H^3L^2}{\xi^2}r_\epsilon
\leq C'\frac{L_\epsilon I^2H^3L^2}{\tau\xi^2}r_\epsilon
\leq C'c_0\frac{\epsilon^2}{H^3}.
\]
Thus retaining the compression and discretization error does not change the
required potential accuracy, and
\[
\Phi_K
\leq
\frac{C\epsilon^2}{H^3}
+\frac{H\epsilon_{\rm act}}{\tau}.
\]
Applying \eqref{eq:master-regularisation-conversion} and
$\sqrt{x+y}\leq\sqrt x+\sqrt y$ gives
\[
\max\left\{
[V_r^{\pi^\star}-V_r^{\pi^K}]_+,
\max_{i\in[I]}[c_i-V_{u_i}^{\pi^K}]_+
\right\}
\leq
C\left[
\epsilon+\frac{H\log(e|\cA|)}{\xi}\tau
+H^2\sqrt{\frac{\epsilon_{\rm act}}{\tau}}
\right].
\]
Balancing the last two terms and imposing the target-accuracy scale gives
the temperature in \eqref{eq:general-on-policy-parameters}, hence
\eqref{eq:general-final-error}.

The off-policy branch collects one trajectory per iteration and, before substituting $\tau$,
\[
K
=
\widetilde O\left(
\frac{d_{\rm E}(d_{\rm E}^2\mathfrak h_{\mathcal X}+d_{\rm a})I^4H^{14}}
{\xi^4\epsilon^4\tau^2}
\right).
\]
Keeping the temperature cap gives
\[
\frac1{\tau^2}
=
\widetilde O\left(
\max\left\{1,\min\left\{
\frac{H^2}{\xi^2\epsilon^2},
\frac1{H^{4/3}\xi^{4/3}
\epsilon_{\rm act}^{2/3}}
\right\}\right\}
\right)
\]
where the second entry of the minimum is interpreted as $+\infty$ when
$\epsilon_{\rm act}=0$. Substitution proves
\eqref{eq:general-replay-sample-complexity}; when the cap is inactive, the
outer maximum can be removed. The transition count is $HK$.
\end{proof}

\begin{remark}[Relation to the linear oracle]
Both oracle constructions fit the components backward, add optimistic
bonuses, include $\psi_{k,h}$ exactly, and scalarize after fitting.
They also share the same sampling modes: a fresh batch or the cumulative off-policy
dataset.  The distinction lies in the bonus geometry.  The linear oracle
uses a covariance matrix and an elliptical-potential argument
\citep{jin2020provably}; the general oracle uses sensitivity sampling
and rounding to obtain a controlled family of widths, followed by
eluder counting \citep{wang2020generalvalue}.

This general construction sacrifices some dimension dependence.
For bounded linear critics,
$d_{\rm E},\mathfrak h_{\mathcal X}=\widetilde O(d_{\rm c})$, so
the on-policy complexity factor $d_{\rm E}^3\mathfrak h_{\mathcal X}$
becomes $\widetilde O(d_{\rm c}^4)$, compared with
$\widetilde O(d_{\rm c}^3)$ for the tailored linear analysis.
The analogous off-policy factors are
$\widetilde O(d_{\rm c}^4+d_{\rm c}d_{\rm a})$ and
$\widetilde O(d_{\rm c}^3+d_{\rm c}d_{\rm a})$, respectively.
The same loss under linear specialization is discussed by
\citet[Theorem~1 and Remarks~2-3]{wang2020generalvalue}.
Only the bonus representation is compressed; every regression continues
to use all supplied data.
\end{remark}

\section{Additional Experimental Details}
\label{app:experiment-details}
The synthetic experiment evaluates the online linear CMDP
algorithm under both fresh on-policy and cumulative off-policy optimistic
evaluation. 
\paragraph{Linear CMDP.}
We use a horizon-$4$ CMDP with one utility constraint, six actions, and
$d_{\rm c}=d_{\rm a}=6$. There are ten nonterminal states: one initial
state and three observable clones of each of three outcome groups
$\{\mathrm{bad},\mathrm{neutral},\mathrm{good}\}$; an additional terminal
state is used only for bookkeeping. The known critic and actor features
coincide, $\phi=\varphi$, but the transition, reward, and utility parameters
are hidden from the learner.

For completeness, let $\mathbf e_j$ be the $j$th coordinate vector,
let $g(s)\in\{0,1,2\}$ index the group of a noninitial state, and let
$\oplus$ denote cyclic addition over the six feature coordinates. With
$(\rho,\gamma)=(6.5,1.5)\times10^{-1}$, the dense simplex features are
\[
\phi(s_1,a)=\frac{1-\rho}{6}\mathbf 1+\rho\mathbf e_a,
\qquad
\phi(s,a)=\frac{1-\rho-\gamma}{6}\mathbf 1
 +\rho\mathbf e_{a\oplus2g(s)}+\gamma\mathbf e_{2g(s)+1}.
\]
The six transition anchors distribute the next state over the three
groups according to
\[
M=10^{-1}
\begin{pmatrix}
9.0&0.8&0.2\\
0.2&0.8&9.0\\
1.2&8.0&0.8\\
0.8&2.5&6.7\\
6.7&2.5&0.8\\
1.8&2.0&6.2
\end{pmatrix}.
\]
We use $M$ and its two cyclic column shifts at the first three stages and
sample uniformly among the three clones of the selected group; the fourth
stage transitions to the terminal state. Reward and utility means are
linear in $\phi$ with parameters
\[
\upsilon_r
 =10^{-1}(9.5,0.5,8.0,2.5,6.5,1.0)^\top,
\qquad
\upsilon_u
 =10^{-1}(0.5,9.5,2.5,8.5,4.5,9.0)^\top.
\]
Conditional rewards and utilities are independent Bernoulli observations
with these means. Thus, learning is genuinely online and stochastic even
though the finite hidden model permits exact post-hoc evaluation.

\paragraph{Benchmark and plotted quantities.}
The lower-bound constraint is
$c=2.14$. Solving the occupancy-measure linear programme gives
$V_r^{\pi^\star}\approx2.2$ and
$V_u^{\pi^\star}\approx2.1$, with the utility constraint active.
For every current policy, we use the hidden model only to report
\[
\Delta_c^k=c-V_u^{\pi^k},
\qquad
\Delta_r^k=V_r^{\pi^\star}-V_r^{\pi^k}.
\]
Hence $\Delta_c^k\leq0$ denotes feasibility. The reward gap is signed:
an infeasible policy may have $\Delta_r^k<0$ because it can exceed the
reward of the constrained optimum.

\paragraph{Compared methods.}
Both methods use the online primal-dual loop of Algorithm~\ref{alg:master-rpgpd}
and the componentwise optimistic ridge evaluation of Algorithm~\ref{alg:explicit-linear-ope}.
The saved runs use a fixed-dimensional log-linear actor fitted by the weighted
action-difference objective in~\eqref{eq:explicit-actor-loss} on a balanced
$15$-triplet coreset. The fit uses the minimum-norm solution on the rank-five
feature-difference span and no actor ridge. The critic ridge parameter is $1.0$,
the confidence level is $\delta=5.0\times10^{-2}$, and $\lambda_{\rm max}\approx22.5$.
Each run performs $K=1.5\times10^4$ policy updates, with stepsize
$\eta=2.0\times10^{-2}$ on-policy and $\eta=5.0\times10^{-2}$ off-policy.
The regularised method uses $\tau=2.0\times10^{-2}$, whereas the baseline
uses $\tau=0$. Both methods train and are evaluated at $c=2.14$.

\paragraph{On- and off-policy data.}
In the on-policy condition, the oracle discards past data and receives a
fresh batch of $N=3.2\times10^4$ trajectories at every update. In the
off-policy condition, it appends one new trajectory per update to cumulative
ridge sufficient statistics. This gives $4.8\times10^8$
trajectories per on-policy seed and $1.5\times10^4$ trajectories per
off-policy seed. The large fresh batches isolate last-iterate behaviour
rather than establish an on-policy sample-efficiency advantage.

The figures show one seed in each sampling mode.

\begin{figure}[htbp]
    \centering
    \includegraphics[width=\linewidth]{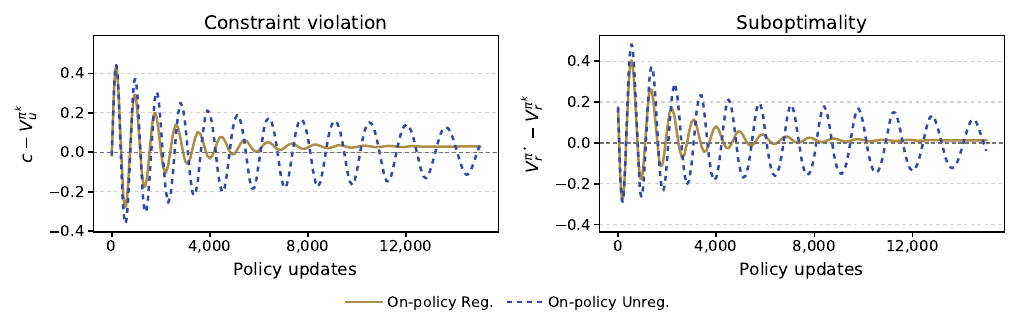}
    \caption{On-policy last-iterate performance for one seed.
    Solid and dashed curves denote the regularised and unregularised
    methods, both trained at $c$.}
    \label{fig:small-linear-cmdp-on-policy}
\end{figure}

\begin{figure}[htbp]
    \centering
    \includegraphics[width=\linewidth]{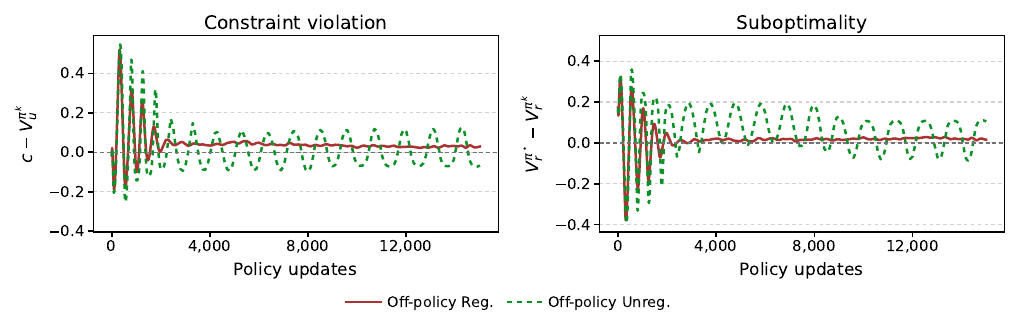}
    \caption{Off-policy last-iterate performance for one seed,
    using the same panels, visual encoding, and axis limits as
    Figure~\ref{fig:small-linear-cmdp-on-policy}.}
    \label{fig:small-linear-cmdp-off-policy}
\end{figure}

\begin{table}[htbp]
    \centering
    \caption{Late-training fluctuations of the signed constraint gap. For each run,
we sort the recorded gaps during updates $13{,}000$--$15{,}000$,
discard the lowest and highest $10\%$, and measure the width of the
remaining values. Each entry gives the median width across $30$
seeds, with the [25th, 75th] percentiles in brackets.}
    \label{tab:small-linear-cmdp-variability}
    \begin{tabular}{lcc}
        \toprule
        Sampling mode & Regularised & Unregularised \\
        \midrule
        On-policy & $0.0027\,[0.0023,\,0.0037]$ & $0.2234\,[0.2199,\,0.2247]$ \\
        Off-policy & $0.0068\,[0.0052,\,0.0084]$ & $0.1790\,[0.0794,\,0.3122]$ \\
        \bottomrule
    \end{tabular}
\end{table}

Figures~\ref{fig:small-linear-cmdp-on-policy}
and~\ref{fig:small-linear-cmdp-off-policy} show the two sampling modes
separately. Over $1.5\times10^4$ updates, the regularised curves
damp their initial oscillations more rapidly and remain flatter than the
unregularised curves in both modes.

We further complement the training-curve plot with the following
statistic, computed over 30 seed runs, to highlight the difference
in fluctuations between the vanilla primal-dual method and its
regularised counterpart.
For seed $j$, let $\Delta_{c,j}^k$ and $\Delta_{r,j}^k$ denote the
signed constraint and reward gaps at update $k$. Over the recorded
checkpoints
$\mathcal K_{\rm tail}=\{13{,}010,13{,}020,\ldots,15{,}000\}$,
the fluctuation score in
Table~\ref{tab:small-linear-cmdp-variability} is
\[
W_j=\operatorname{Quant}_{0.9}
\bigl(\Delta_{c,j}^k:k\in\mathcal K_{\rm tail}\bigr)
-\operatorname{Quant}_{0.1}
\bigl(\Delta_{c,j}^k:k\in\mathcal K_{\rm tail}\bigr),
\]
where $\operatorname{Quant}_p$ is the empirical $p$-quantile.
Intuitively, a larger $W_j$ indicates greater oscillation in the
late-training curve.


\end{document}